\documentclass[preprint,12pt]{elsarticle}
\usepackage{multirow} 
\usepackage{diagbox}
\usepackage{graphicx} 
\usepackage[reqno]{amsmath}%
\usepackage{amssymb}%
\usepackage{bm}
\usepackage[table]{xcolor}%
\usepackage{subcaption}
\allowdisplaybreaks

\usepackage{amsthm}
\usepackage{extarrows}
\usepackage{arydshln}

\usepackage{xcolor}%
\usepackage{mleftright}%
\usepackage{xspace}%
\usepackage{graphicx}
\usepackage{subcaption} 
\usepackage{hyperref}%

\usepackage{algorithm}
\usepackage{algpseudocode}
\usepackage{makecell}
\usepackage{tabularx}
\usepackage{float}
\usepackage{placeins}

\theoremstyle{plain}%
\newtheorem{theorem}{Theorem}[section]

\newtheorem{lemma}[theorem]{Lemma}

\newtheorem{corollary}[theorem]{Corollary}

\newtheorem{proposition}[theorem]{Proposition}
\newtheorem{prop}[theorem]{Proposition}

\theoremstyle{plain}%
\newtheorem{definition}[theorem]{Definition}

\newtheorem{example}[theorem]{Example}

\numberwithin{figure}{section}%
\numberwithin{table}{section}%
\numberwithin{equation}{section}%

\newcommand{\indep}{\mathop{\perp\!\!\!\!\perp}}
\newcommand{\notindep}{\mathop{\not \perp\!\!\!\!\perp}}



\usepackage[inline]{enumitem}

\newlist{compactenumA}{enumerate}{5}%
\setlist[compactenumA]{topsep=0pt,itemsep=-1ex,partopsep=1ex,parsep=1ex,%
   label=(\Alph*)}%

\newlist{compactenuma}{enumerate}{5}%
\setlist[compactenuma]{topsep=0pt,itemsep=-1ex,partopsep=1ex,parsep=1ex,%
   label=(\alph*)}%

\newlist{compactenumI}{enumerate}{5}%
\setlist[compactenumI]{topsep=0pt,itemsep=-1ex,partopsep=1ex,parsep=1ex,%
   label=(\Roman*)}%

\newlist{compactenumi}{enumerate}{5}%
\setlist[compactenumi]{topsep=0pt,itemsep=-1ex,partopsep=1ex,parsep=1ex,%
   label=(\roman*)}%

\newlist{compactitem}{itemize}{5}%
\setlist[compactitem]{topsep=0pt,itemsep=-1ex,partopsep=1ex,parsep=1ex,%
   label=\ensuremath{\bullet}}%

\usepackage{tikz}
\usetikzlibrary{shapes,decorations}
\usetikzlibrary{shapes,arrows}
\usetikzlibrary{shapes.geometric, arrows}
\usetikzlibrary{arrows.meta}

\usepackage{booktabs}
\usepackage{pifont}
\usepackage{color}

\begin{document}

\begin{frontmatter}



\title{Causal Discovery for Linear DAGs with Dependent Latent Variables via Higher-order Cumulants}


\author[1]{Ming Cai} 
\author[1]{Penggang Gao}
\author[1,2]{Hisayuki Hara}

\affiliation[1]{organization={Graduate School of Informatics, Kyoto University},
            addressline={Yoshida Konoe-cho}, 
            city={Kyoto},
            postcode={606-8501}, 
            country={Japan}}
\affiliation[2]{organization={Institute for Liberal Arts and Sciences, Kyoto University},
            addressline={Yoshida Nihonmatsu-cho}, 
            city={Kyoto},
            postcode={606-8501}, 
            country={Japan}}

\begin{abstract}
This paper addresses the problem of estimating causal directed acyclic graphs in linear non-Gaussian acyclic models with latent confounders (LvLiNGAM). Existing methods assume mutually independent latent confounders or cannot properly handle models with causal relationships among observed variables.

We propose a novel algorithm that identifies causal DAGs in LvLiNGAM, allowing causal structures among latent variables, among observed variables, and between the two. The proposed method leverages higher-order cumulants of observed data to identify the causal structure. Extensive simulations and experiments with real-world data demonstrate the validity and practical utility of the proposed algorithm. 
\end{abstract}


\end{frontmatter}



\section{Introduction}
\label{sec: introduction}
Estimating causal directed acyclic graphs (DAGs) in the presence of latent confounders has been a major challenge in causal analysis. Conventional causal discovery methods, such as the Peter–Clark (PC) algorithm \cite{Spirtes}, Greedy Equivalence Search (GES) \cite{Chickering2003}, and the Linear Non-Gaussian Acyclic Model (LiNGAM) \cite{Shimizu2006, Shimizu2011}, focus solely on the causal model without latent confounders. 

Fast Causal Inference (FCI) \cite{Spirtes} extends the PC algorithm to handle latent variables, recovering a partial ancestral graph (PAG) under the faithfulness assumption. However, FCI is computationally intensive and, moreover, often fails to determine the causal directions. Really Fast Causal Inference (RFCI) \cite{colombo2012learning} trades off some independence tests for speed, at the loss of estimating accuracy. Greedy Fast Causal Inference (GFCI) \cite{Ogarrio2016ahybrid} hybridizes GES and FCI but inherits the limitation of FCI.

The assumption of linearity and non-Gaussian disturbances in the causal model enables the identification of causal structures beyond the PAG. The linear non-Gaussian acyclic model with latent confounders (LvLiNGAM) is an extension of LiNGAM that incorporates latent confounders. Hoyer et al. \cite{hoyer2008estimation} demonstrated that LvLiNGAM can be transformed into a canonical model in which all latent variables are mutually independent and causally precede the observed variables. They proposed estimating the canonical models using overcomplete ICA \cite{lewicki1998learning}, assuming that the number of latent variables is known. Overcomplete ICA can identify the causal DAG only up to permutations and scaling of the variables. Thus, substantial computational effort is required to identify the true causal DAG from the many candidate models. Another limitation of overcomplete ICA is its tendency to converge to local optima. Salehkaleybar et al. \cite{salehkaleybar2020learning} improved the algorithm by reducing the candidate models. 

Other methods for estimating LvLiNGAM, based on linear regression analysis and independence testing, have also been developed.\citep{Entner2011, tashiro2014parcelingam, Maeda2020, maeda2022rcd}. Furthermore, Multiple Latent Confounders LiNGAM (ML-CLiNGAM) \cite{Chen2022Causal} and FRITL \cite{chen2021fritl} initially identify the causal skeleton using a constraint-based method, and then estimate the causal directions of the undirected edges in the skeleton using linear regression and independence tests. While these methods can identify structures among observed variables that are not confounded by latent variables, they cannot necessarily determine the causal direction between two variables confounded by latent variables.

More recently, methods using higher-order cumulants have led to new developments in the identification of canonical LvLiNGAMs. 
Cai et al. \cite{cai2023causal} assume that each latent variable has at least three observed children, and that there exists a subset of these children that are not connected by any other observed or latent variables. Then, cumulants are employed to identify one-latent-component structures and latent influences are recursively removed to recover the underlying causal relationships. Chen et al. \cite{chen2024identification} show that if two observed variables share one latent confounder, the causal direction between them can be identified by leveraging higher-order cumulants. Schkoda et al. \cite{schkoda2024causal} introduced ReLVLiNGAM, a recursive approach that leverages higher-order cumulants to estimate canonical LvLiNGAM with multiple latent parents. One strength of ReLVLiNGAM is that it does not require prior knowledge of the number of latent variables. 

The methods reviewed so far are estimation methods for the canonical LvLiNGAM. A few methods, however, have been proposed to estimate the causal DAG of LvLiNGAM when latent variables exhibit causal relationships. A variable is said to be pure if it is conditionally independent of other observed variables given its latent parents; otherwise, it is called impure.
Silva et al. \cite{Silva2006learning} showed that the latent DAG is identifiable under the assumption that each latent variable has at least three pure children, by employing tetrad conditions on the covariance of the observed variables. Cai et al. \cite{cai2019triad} proposed a two-phase algorithm, LSTC (learning the structure of latent variables based on Triad Constraints), to identify the causal DAG where each latent variable has at least two children, all of which are pure, and each observed variable has a single latent parent. Xie et al. \cite{xie2020generalized} generalized LSTC and defined the linear non-Gaussian latent variable model (LiNGLaM), where observed variables may have multiple latent parents but no causal edges among them, and proved its identifiability. In \cite{cai2019triad} and \cite{xie2020generalized}, causal clusters are defined as follows:
\begin{definition}[Causal cluster \cite{cai2019triad, xie2020generalized}]
\label{def: cluster}
A set of observed variables that share the same latent parents is called a causal cluster.
\end{definition}
Their methods consist of two main steps: identifying causal clusters and then recovering the causal order of latent variables.
LSTC and the algorithm for LiNGLaM estimate clusters of observed variables by leveraging the Triad constraints or the generalized independence noise (GIN) conditions. It is also possible to define clusters in the same manner as Definition \ref{def: cluster} for models where causal edges exist among observed variables. However, when impure observed variables exist, their method might fail to identify the clusters, resulting in an incorrect estimation of both the number of latent variables and latent DAGs.
Several recent studies have shown that LvLiNGAM remains identifiable even when some observed variables are impure \cite{Xie2023Causal, Xie2024GIN, jin2024structural}. However, these methods still rely on the existence of at least some pure observed variables in each cluster.

\subsection{Contributions}
\label{sec: contributions}
In this paper, we relax the pure observed children assumption of Cai et al. \cite{cai2019triad} and investigate the identifiability of the causal DAG for an extended model that allows causal structures both among latent variables and among observed variables. 
Using higher-order cumulants of the observed data, we show the identifiability of the causal DAG of a class of LvLiNGAM and propose a practical algorithm for estimating the class. The proposed method first estimates clusters using the approaches of \cite{cai2019triad, xie2020generalized}. When causal edges exist among observed variables, the clusters estimated by using Triad constraints or GIN conditions may be over-segmented compared to the true clusters.  
The proposed method leverages higher-order cumulants of observed variables to refine these clusters, estimates causal edges within clusters, determines the causal order among latent variables, and finally estimates the exact causal structure among latent variables. 

In summary, our main contributions are as follows: 
\begin{enumerate}
\setlength{\itemsep}{0pt}
\setlength{\parsep}{0pt}
\setlength{\parskip}{0pt}
\item Demonstrate identifiability of causal DAGs in a class of LvLiNGAM, allowing causal relationships among latent and observed variables.
\item Extend the causal cluster estimation methods of \cite{cai2019triad} and \cite{xie2020generalized} to handle cases where directed edges exist among observed variables within clusters.
\item Propose a top-down algorithm using higher-order cumulants to infer the causal order of latent variables.
\item Develop a bottom-up recursive procedure to reconstruct the latent causal DAG from latent causal orders.
\end{enumerate}

The rest of this paper is organized as follows. Section \ref{sec: problem definition} defines the class of LvLiNGAM considered in this study. In Section \ref{sec: problem definition}, we also summarize some basic facts on higher-order cumulants. Section \ref{sec: proposed method} describes the proposed method in detail. Section \ref{sec: simulations} presents numerical simulations to demonstrate the effectiveness of the proposed method.
Section \ref{sec: real world} evaluates the usefulness of the proposed method by applying it to the Political Democracy dataset \cite{bollen1989structural}. Finally, Section \ref{sec: conclusion} concludes the paper. All proofs of theorems, corollaries, and lemmas in the main text are provided in the Appendices. 

\section{Preliminaries}
\label{sec: problem definition}
\subsection{LvLiNGAM}
\label{sec: LvLiNGAM}
Let $\bm{X} = (X_{1}, \dots, X_{p})^{\top}$ and $\bm{L} = (L_{1}, \dots, L_{q})^{\top}$ be vectors of observed and latent variables, respectively. In this paper, we identify these vectors with the corresponding set of variables. 
Define $\bm{V} = \bm{X} \cup \bm{L} = \{V_1,\ldots,V_{p+q}\}$. Let $\mathcal{G} = (\bm{V}, E)$ be a causal DAG. $V_i \to V_j$ denotes a directed edge from $V_i$ to $V_j$. $\mathrm{Anc}(V_i)$, $\mathrm{Pa}(V_i)$, and $\mathrm{Ch}(V_i)$ are the sets of ancestors, parents, and children of $V_i$, respectively. We use $V_i \prec V_j$ to indicate that $V_i$ precedes $V_j$ in a causal order. 

The LvLiNGAM considered in this paper is formulated as
\begin{align}
    \label{model: lvLiNGAM General}
        \left[\begin{array}{c}
            \bm{L} \\
            \bm{X}
        \end{array}\right] 
        = 
        \left[\begin{array}{cc}
            \bm{A}       & \bm{0}  \\
            \bm{\Lambda} & \bm{B}
        \end{array}\right]
        \left[\begin{array}{c}
            \bm{L} \\
            \bm{X}
        \end{array}\right]
        +
        \left[\begin{array}{c}
            \bm{\epsilon} \\
            \bm{e}
        \end{array}\right],   
\end{align}
where $\bm{A}=\{a_{ji}\}$, $\bm{B}=\{b_{ji}\}$, and $\bm{\Lambda}=\{\lambda_{ji}\}$ are matrices of causal coefficients, while $\bm{\epsilon}$ and $\bm{e}$ denote vectors of independent non-Gaussian disturbances associated with $\bm{L}$ and $\bm{X}$, respectively.
Let $a_{ji}$, $\lambda_{ji}$, and $b_{ji}$ be the causal coefficients from $L_i$ to $L_j$, from $L_i$ to $X_j$, and from $X_i$ to $X_j$, respectively. Due to the arbitrariness of the scale of latent variables, we may, without loss of generality, set one of the coefficients $\lambda_{ji}$ to $1$ for some $X_j \in \mathrm{Ch}(L_i)$. Hereafter, such a normalization will often be used.

$\bm{A}$ and $\bm{B}$ can be transformed into lower triangular matrices by row and column permutations. We assume that the elements of $\bm{\epsilon}$ and $\bm{e}$ are mutually independent and follow non-Gaussian continuous distributions. Let $\mathcal{M}_{\mathcal{G}}$ denote the LvLiNGAM defined by $\mathcal{G}$. As shown in (\ref{model: lvLiNGAM General}), we assume in this paper that all observed variables are not ancestors of any latent variables. 

Consider the following reduced form of (\ref{model: lvLiNGAM General}),  
\begin{align*}
        \left[\begin{array}{c}
            \bm{L} \\
            \bm{X}
        \end{array}\right] 
        &=
        \left[\begin{array}{cc}
            (\bm{I}_q-\bm{A})^{-1} & \bm{0}\\
            (\bm{I}_p-\bm{B})^{-1}\bm{\Lambda}(\bm{I}_q-\bm{A})^{-1} & (\bm{I}_p-\bm{B})^{-1}
        \end{array}\right]
        \left[\begin{array}{c}
            \bm{\epsilon} \\
            \bm{e}
        \end{array}\right]. 
\end{align*}
Let $\alpha^{ll}_{ji}$, $\alpha^{ol}_{ji}$, and $\alpha^{oo}_{ji}$ represent the total effects from $L_i$ to $L_j$, $L_i$ to $X_j$, and $X_i$ to $X_j$, respectively. Thus, $(\bm{I}_q-\bm{A})^{-1} = \{\alpha^{ll}_{ji}\}$, $(\bm{I}_p-\bm{B})^{-1}\bm{\Lambda}(\bm{I}_q-\bm{A})^{-1} = \{\alpha^{ol}_{ji}\}$, and $(\bm{I}_p-\bm{B})^{-1} = \{\alpha^{oo}_{ji}\}$. The total effect from $V_i$ to $V_j$ is denoted by $\alpha_{ji}$, with the superscript omitted. 

$\bm{M}:=\left[(\bm{I}_p-\bm{B})^{-1}\bm{\Lambda}(\bm{I}_q-\bm{A})^{-1}, (\bm{I}_p-\bm{B})^{-1}\right]$ is called a mixing matrix of the model (\ref{model: lvLiNGAM General}). Denote $\bm{u} =(\bm{\epsilon}^\top, \bm{e}^\top)^\top$. Then, $\bm{X}$ is written as
\begin{align}
\label{model: mixing matrix}
\bm{X} = \bm{M}\bm{u},
\end{align}
which conforms to the formulation of the overcomplete ICA problem \cite{lewicki1998learning, eriksson2004identifiability, hoyer2008estimation}. $\bm{M}$ is said to be irreducible if every pair of columns is linearly independent. $\mathcal{G}$ is said to be minimal if and only if $\bm{M}$ is irreducible. If $\mathcal{G}$ is not minimal, some latent variables can be absorbed into other latent variables, resulting in a minimal graph \cite{salehkaleybar2020learning}. $\mathcal{M}_\mathcal{G}$ is called the canonical model when $\bm{A} = \bm{0}$ and $\bm{M}$ is irreducible. Hoyer et al. \cite{hoyer2008estimation} showed that any LvLiNGAM can be transformed into an observationally equivalent canonical model. For example, the LvLiNGAM defined by the DAG in Figure~\ref{fig: canonical models and other} (a) is the canonical model of the LvLiNGAM defined by the DAG in Figure~\ref{fig: canonical models and other} (b). Hoyer et al. \cite{hoyer2008estimation} also demonstrated that, when the number of latent variables is known, the canonical model can be identified up to observational equivalent models using overcomplete ICA. 

\begin{figure}[b]
    \centering
    \begin{subfigure}[t]{0.45\textwidth}
        \centering
        \begin{tikzpicture}[scale=0.7]
            \node[draw, rectangle] (L_1) at (-2, 2) {$L_{1}$};
            \node[draw, rectangle] (L_2) at (0, 2) {$L_{2}$};
            \node[draw, rectangle] (L_3) at (2, 2) {$L_{3}$};
            \node[draw, circle] (X_1) at (-2, 0) {$X_{1}$};
            \node[draw, circle] (X_2) at (0, 0) {$X_{2}$};
            \node[draw, circle] (X_3) at (2, 0) {$X_{3}$};
            \node[draw, circle] (X_4) at (-2, 4) {$X_{4}$};
            \node[draw, circle] (X_5) at (0, 4) {$X_{5}$};
            \node[draw, circle] (X_6) at (2, 4) {$X_{6}$};
            \draw[thick, ->] (L_1) -- (X_1);
            \draw[thick, ->] (L_1) -- (X_2);
            \draw[thick, ->] (L_1) -- (X_3);
            \draw[thick, ->] (L_1) -- (X_4);
            \draw[thick, ->] (L_1) -- (X_5);
            \draw[thick, ->] (L_1) -- (X_6);
            \draw[thick, ->] (L_2) -- (X_2);
            \draw[thick, ->] (L_2) -- (X_3);
            \draw[thick, ->] (L_2) -- (X_5);
            \draw[thick, ->] (L_2) -- (X_6);
            \draw[thick, ->] (L_3) -- (X_3);
            \draw[thick, ->] (L_3) -- (X_6);
        \end{tikzpicture}
        \caption{An example of canonical LvLiNGAM}
    \end{subfigure}
    \hfill
    \begin{subfigure}[t]{0.45\textwidth}
        \centering
        \begin{tikzpicture}[scale=0.7]
            \node[draw, rectangle] (L_1) at (-2, 2) {$L_{1}$};
            \node[draw, rectangle] (L_2) at (0, 2) {$L_{2}$};
            \node[draw, rectangle] (L_3) at (2, 2) {$L_{3}$};
            \node[draw, circle] (X_1) at (-2, 0) {$X_{1}$};
            \node[draw, circle] (X_4) at (-2, 4) {$X_{4}$};
            \node[draw, circle] (X_2) at (0, 0) {$X_{2}$};
            \node[draw, circle] (X_5) at (0, 4) {$X_{5}$};
            \node[draw, circle] (X_3) at (2, 0) {$X_{3}$};
            \node[draw, circle] (X_6) at (2, 4) {$X_{6}$};
            \draw[thick, ->] (L_1) -- (L_2);
            \draw[thick, ->] (L_2) -- (L_3);
            \draw[thick, ->] (L_1) -- (X_1);
            \draw[thick, ->] (L_1) -- (X_4);
            
            \draw[thick, ->] (L_2) -- (X_2);
            \draw[thick, ->] (L_2) -- (X_5);
            \draw[thick, ->] (L_3) -- (X_3);
            \draw[thick, ->] (L_3) -- (X_6);
        \end{tikzpicture}
        \caption{An LvLiNGAM that can be identified by \cite{cai2019triad, xie2020generalized}}
    \end{subfigure}
    \caption{Examples of LvLiNGAMs}
    \label{fig: canonical models and other}
\end{figure}
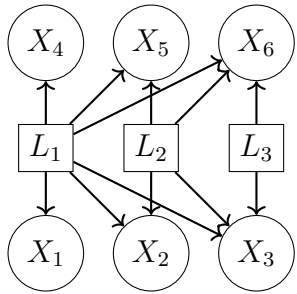
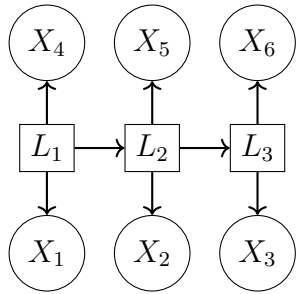

Salehkaleybar et al. \cite{salehkaleybar2020learning} showed that, even when $\bm{A} \ne \bm{0}$, the irreducibility of $\bm{M}$ is a necessary and sufficient condition for the identifiability of the number of latent variables. However, they did not provide an algorithm for estimating the number of latent variables. Schkoda et al. \cite{schkoda2024causal} proposed ReLVLiNGAM to estimate the canonical model with generic coefficients even when the number of latent variables is unknown. However, the canonical model derived from an LvLiNGAM with $\bm{A} \ne \bm{0}$ lies in a measure-zero subset of the parameter space, which prevents ReLVLiNGAM from accurately identifying the number of latent confounders between two observed variables in such cases. For example, ReLVLiNGAM may not identify the canonical model in Figure~\ref{fig: canonical models and other} (a) from data generated by the LvLiNGAM in Figure~\ref{fig: canonical models and other} (b). 

Cai et al. \cite{cai2019triad} and Xie et al. \cite{xie2020generalized} demonstrated that within LvLiNGAMs where all the observed children of latent variables are pure, there exists a class, such as the models shown in Figure~\ref{fig: canonical models and other} (b), in which the causal order among latent variables is identifiable. They proposed algorithms for estimating the causal order. However, the complete causal structure cannot be identified solely from the causal order, and their algorithm cannot be generalized to cases where causal edges exist among observed variables or where latent variables do not have sufficient pure children. 

In this paper, we introduce the following class of models, which generalizes the class of models in Cai et al.~\cite{cai2019triad} by allowing causal edges among the observed variables, and consider the problem of identifying the causal order among observed variables within each cluster as well as the causal structure among the latent variables. 

{\renewcommand{\theenumi}{A\arabic{enumi}} 
\renewcommand{\labelenumi}{\textbf{\theenumi}.}
\makeatletter\renewcommand{\p@enumi}{}\makeatother 
\begin{enumerate}
\setlength{\itemsep}{0pt}
\setlength{\parsep}{0pt}
\setlength{\parskip}{0pt}
  \item\label{A1} Each observed variable has only one latent parent. 
  \item\label{A2} Each latent variable has at least two children, at least one of which is observed. 
  \item\label{A3} There are no direct causal paths between causal clusters. 
  \item\label{A4} The model satisfies the faithfulness assumption.
  \item\label{A5} The higher-order cumulant of each component of the disturbance $\bm{u}$ is nonzero.
\end{enumerate}}

In Section \ref{sec: proposed method}, we demonstrate that the causal structure of latent variables and the causal order of observed variables for the LvLiNGAM that satisfies Assumption A1-A5 are identifiable, and we provide an algorithm for estimating the causal DAG for this class. The proposed method enables the identification not only of the causal order among latent variables but also of their complete causal structure. 

Under Assumption A1, every observed variable is assumed to have one latent parent.
However, even if there exist observed variables without latent parents, the estimation problem can sometimes be reduced to a model satisfying Assumption A1 by applying ParceLiNGAM~\cite{tashiro2014parcelingam} or repetitive causal discovery (RCD) \cite{Maeda2020, maeda2022rcd} as a preprocessing step of the proposed method. Details are provided in Appendix~\ref{sec: reduce models}. 
\subsection{Cumulants}
\label{sec: cumulant}
The proposed method leverages higher-order cumulants of observed data to identify the causal structure among latent variables. In this subsection, we summarize some facts on higher-order cumulants. 
First, we introduce the definition of a higher-order cumulant. 
\begin{definition}[Cumulants \cite{Brillinger}]
    Let $i_1,\ldots,i_k \in \{1,\ldots,p\}$. 
    The $k$-th order cumulant of random vector $(X_{i_1},\ldots,X_{i_k})$ is 
    \begin{align*}
    c_{i_1,\ldots,i_k}^{(k)} &=
    \mathrm{cum}^{(k)}(X_{i_1},\ldots,X_{i_k})\\
    &= \sum_{(I_1,\ldots,I_h)} (-1)^{h-1}(h-1)!
    E\left[
    \prod_{j \in I_1}X_j    
    \right]
    \cdots
    E\left[
    \prod_{j \in I_h}X_j    
    \right],
    \end{align*}
    where the sum is taken over all partitions $(I_1,\ldots,I_h)$ of $(i_1,\ldots,i_k)$. 
\end{definition}
If $i_1 = \cdots = i_k=i$, we write $\mathrm{cum}^{(k)}(X_i)$ to denote $\mathrm{cum}^{(k)}(X_{i},\ldots,X_{i})$. 
The $k$-th order cumulants of the observed variables of LvLiNGAM satisfy
\begin{align*}
    c^{(k)}_{i_{1}, i_{2}, \dots, i_{k}} &= 
    \mathrm{cum}^{(k)}(X_{i_1},\ldots,X_{i_k})\\
    &=\sum_{j = 1}^{q}\alpha^{ol}_{i_{1}j}\cdots\alpha^{ol}_{i_{k}j}\mathrm{cum}^{(k)}(\epsilon_{j})+
    \sum_{j = 1}^{p}\alpha^{oo}_{i_{1}j}\cdots\alpha^{oo}_{i_{k}j}\mathrm{cum}^{(k)}(e_{j})
\end{align*}

We consider an LvLiNGAM in which all variables except $X_i$ and $X_j$ are regarded as latent variables. We refer to the canonical model that is observationally equivalent to this model as the canonical model over $X_i$ and $X_j$. Let $\mathrm{Conf}(X_i, X_j) = \{L^\prime_{1}, L^\prime_{2}, \cdots, L^\prime_{\ell}\}$ be the set of latent confounders in the canonical model over $X_{i}$ and $X_{j}$, where all $L_{h}^{\prime} \in \mathrm{Conf}(X_i, X_j)$ are mutually independent. Without loss of generality, we assume that $X_{j} \notin \mathrm{Anc}(X_{i})$. 
Then, $X_i$ and $X_j$ are expressed as 
\begin{align}
\label{eq: two variables}
X_{i} &= \sum^{\ell}_{h = 1} {\alpha^{\prime}_{ih}}L^{\prime}_{h} + v_{i}, \quad X_{j} = \sum^{\ell}_{h = 1} {\alpha^{\prime}_{jh}}L^{\prime}_{h} + \alpha^{oo}_{ji}v_{i} + v_{j},  
\end{align}
where $v_{i}$ and $v_{j}$ are disturbances, and ${\alpha^{\prime}_{ih}}$ and ${\alpha^{\prime}_{jh}}$ are total effects from $L_{h}^{\prime}$ to $X_{i}$ and $X_{j}$, respectively, in the canonical model over them. We note that the model (\ref{eq: two variables}) is a canonical model with generic parameters, and that $\ell$ is equal to the number of confounders in the original model $\mathcal{M}_{\mathcal{G}}$. 

Schkoda et al.~\cite{schkoda2024causal} proposed an algorithm for estimating the canonical model with generic parameters by leveraging higher-order cumulants. Several of their theorems concerning higher-order cumulants are also applicable to the canonical model over $X_i$ and $X_j$. They define a $\left(\sum_{i=0}^{k_{2}-k_{1}+1} i\right) \times k_{1}$ matrix $A^{(k_{1}, k_{2})}_{(X_{i} \to X_{j})}$ as follows: 

\begin{align}
    A^{(k_{1}, k_{2})}_{(X_{i} \to X_{j})} =
    \left[
    \begin{array}{cccc}
       c^{(k_{1})}_{i, i, \dots, i}  & c^{(k_{1})}_{i, i, \dots, j} & \dots & c^{(k_{1})}_{i, i, \dots, j}\\
       c^{(k_{1}+1)}_{i, i, i, \dots, i}  & c^{(k_{1}+1)}_{i, i, i, \dots, j} & \dots & c^{(k_{1}+1)}_{i, i, j, \dots, j}\\
       c^{(k_{1}+1)}_{j, i, i, \dots, i}  & c^{(k_{1}+1)}_{j, i, i, \dots, j} & \dots & c^{(k_{1}+1)}_{j, i, j, \dots, j}\\
       \vdots & \vdots & \ddots & \vdots \\
       c^{(k_{2})}_{i, \dots, i, i, i, \dots, i, i}  & c^{(k_{2})}_{i, \dots, i, i, i, \dots, i, j} & \dots & c^{(k_{2})}_{i, \dots, i, i, j, \dots, j, j}\\
       \vdots & \vdots & \ddots & \vdots \\
       c^{(k_{2})}_{j, \dots, j, i, i, \dots, i, i}  & c^{(k_{2})}_{j, \dots, j, i, i, \dots, i, j} & \dots & c^{(k_{2})}_{j, \dots, j, i, j, \dots, j, j}
    \end{array}
    \right],
\end{align}
where $k_{1} < k_{2}$. $A^{(k_{1}, k_{2})}_{(X_{j} \to X_{i})}$ is defined similarly by swapping the indices $i$ and $j$ in $A^{(k_{1}, k_{2})}_{(X_{i} \to X_{j})}$. Proposition~\ref{thm: latent number} enables the identification of $\ell$ in (\ref{eq: two variables}) and the causal order between $X_i$ and $X_j$.  
\begin{prop}[Theorem 3 in \cite{schkoda2024causal}]
\label{thm: latent number}
    For two observed variables $X_{i}$ and $X_{j}$ where $X_{j} \notin Anc(X_{i})$. Let $m := \min(\sum^{k_{2}-k_{1}+1}_{i=1}i, k_{1})$.
    Then, 
    \begin{enumerate}
    \setlength{\itemsep}{0pt}
    \setlength{\parsep}{0pt}
    \setlength{\parskip}{0pt}
        \item $A^{(k_{1}, k_{2})}_{(X_i \to X_j)}$ generically has rank $min(\ell+1, m)$.
        \item If $\alpha^{oo}_{ji} \ne 0$, $A^{(k_{1}, k_{2})}_{(X_{i} \to X_{j})}$ generically has rank $\min(\ell+2, m)$.
        \item If $\alpha^{oo}_{ji} = 0$, $A^{(k_{1}, k_{2})}_{(X_{i} \to X_{j})}$ generically has rank $\min(\ell+1, m)$.
    \end{enumerate}
\end{prop} 

Define $A^{(\ell)}_{(X_{i} \to X_{j})}$ as $A^{(k_1, k_2)}_{(X_{i} \to X_{j})}$ for the case where $k_1 = \ell + 2$ and $k_2$ is the smallest possible choice, and let $\tilde{A}^{(\ell)}_{(X_{i} \to X_{j})}$ be the matrix obtained by adding the row vector $(1, \alpha, \dots, \alpha^{\ell+1})$ as the first row of $A^{(\ell)}_{(X_{i} \to X_{j})}$. 
\begin{prop}[Theorem 4 in \cite{schkoda2024causal}]
\label{thm: estimate b}
    Consider the determinant of an $(\ell+2)\times (\ell+2)$ minor of $\tilde{A}^{(\ell)}_{(X_{i} \to X_{j})}$ that contains the first row and treat it as a polynomial in $\alpha$. 
    Then, the roots of this polynomial are $\alpha^{oo}_{ji}, \alpha^{ol}_{j1}, \cdots, \alpha^{ol}_{j\ell}$.
\end{prop}
Proposition \ref{thm: estimate b} enables the identification of $\alpha^{oo}_{ji}, \alpha^{ol}_{j1}, \cdots, \alpha^{ol}_{j\ell}$ up to permutation. The following proposition plays a crucial role in this paper in identifying both the number of latent variables and the true clusters.
\begin{prop}[Lemma 5 in \cite{schkoda2024causal}]
\label{lem: estimate e cumulant}
    For two observed variables $X_{i}$ and $X_{j}$, $\alpha^{oo}_{ji}, \alpha^{ol}_{j1}, \cdots, \alpha^{ol}_{j\ell}$ are the roots of the polynomial in Proposition \ref{thm: estimate b}. 
    Then 
    \begin{align}
    \label{eq: system}
    \left[
        \begin{array}{cccc}
            1 & 1 & \dots & 1 \\ 
            \alpha^{oo}_{ji} & \alpha^{ol}_{j1} & \dots & \alpha^{ol}_{j\ell} \\
            \vdots & \vdots & \ddots & \vdots \\
            (\alpha^{oo}_{ji})^{k-1} & (\alpha^{ol}_{j1})^{k-1} & \dots & (\alpha_{j\ell}^{ol})^{k-1}
        \end{array}
    \right]
    \left[
        \begin{array}{c}
             \mathrm{cum}^{(k)}({v}_{i}) \\ 
            \mathrm{cum}^{(k)}(L^\prime_{1}) \\
            \vdots \\
            \mathrm{cum}^{(k)}(L^\prime_{\ell})\\
        \end{array}
    \right]
    =\left[
        \begin{array}{c}
            c^{(k)}_{i, i, \dots, i} \\ 
            c^{(k)}_{i, i, \dots, j}\\
            \vdots \\
            c^{(k)}_{i, j, \dots, j}
        \end{array}
    \right]
\end{align}
(\ref{eq: system}) is generically uniquely solvable if $k \geq \ell+1$.
\end{prop}

In the following, let $c^{(k)}_{(X_{i} \to X_{j})}(L^{\prime}_h)$, where $h=1,\ldots,\ell$, denote the solution of $\mathrm{cum}^{(k)}(L^\prime_{h})$ in (\ref{eq: system}). 
\section{Proposed Method}
\label{sec: proposed method}
In this section, we propose a three-stage algorithm for identifying LvLiNGAM that satisfy Assumptions A1–A5. In the first stage, leveraging Cai et al.~\cite{cai2019triad}’s Triad constraints and Proposition \ref{thm: latent number}, the method estimates over-segmented causal clusters and assigns a latent parent to each cluster. In this stage, the ancestral relationships among observed variables are also estimated. In the second stage, Proposition \ref{thm: estimate b} is employed to identify latent sources recursively and, as a result, the causal order among the latent variables is estimated. When multiple latent variables are found to have identical cumulants, their corresponding clusters are merged, enabling the identification of the true clusters. In general, even if the causal order among latent variables can be estimated, the causal structure among them cannot be determined. The final stage identifies the exact causal structure among latent variables in a bottom-up manner.
\subsection{Stage I: Estimating Over-segmented Clusters}
\label{sec: partial cluster}
First, we introduce the Triad constraint proposed by Cai et al. \cite{cai2019triad}, which also serves as a key component of our method in this stage. 
\begin{definition}[Triad constraint \cite{cai2019triad}]
    Let $X_{i}$, $X_{j}$, and $X_{k}$ be observed variables in the LvLiNGAM and assume that $\mathrm{Cov}(X_j,X_k) \ne 0$. Define Triad statistic $e_{(X_{i}, X_{j}\mid X_{k})}$ by 
    \begin{align}
        e_{(X_{i}, X_{j}\mid X_{k})} := X_{i} - \frac{\mathrm{Cov}(X_{i}, X_{k})}{\mathrm{Cov}(X_{j}, X_{k})}X_{j}.
    \end{align}
    If $e_{(X_{i}, X_{j}\mid X_{k})} \indep X_{k}$, we say that $\{X_{i}, X_{j}\}$ and $X_{k}$ satisfy the Triad constraint. 
\end{definition}
The following propositions are also provided by Cai et al.~\cite{cai2019triad}. 
\begin{prop}[\cite{cai2019triad}]
\label{thm: minimal cluster}
    Assume that all observed variables are pure, and $X_{i}$ and $X_{j}$ are dependent. 
    If $\{X_{i}, X_{j}\}$ and all $X_{k} \in \bm{X} \setminus \{X_{i}, X_{j}\}$ satisfy the Triad constraint, then $X_{i} \text{ and } X_{j}$ form a cluster.
\end{prop}
\begin{prop}[\cite{cai2019triad}]
\label{thm: merging cluster}
    Let $\hat{C}_1$ and $\hat{C}_2$ be two clusters estimated by using Triad constraints.  
    If $\hat{C}_{1}$ and $\hat{C}_{2}$ satisfy $\hat{C}_{1} \cap \hat{C}_{2} \neq \emptyset$, $\hat{C}_{1} \cup \hat{C}_{2}$ also forms a cluster. 
\end{prop}

When all observed variables are pure, as in the model shown in Fig. \ref{fig: canonical models and other} (b), the correct clusters can be identified in two steps: first, apply Proposition \ref{thm: minimal cluster} to find pairs of variables in the same cluster; then, merge them using Proposition \ref{thm: merging cluster}. However, when impure observed variables are present, the clusters obtained using this method become over-segmented relative to the true clusters.

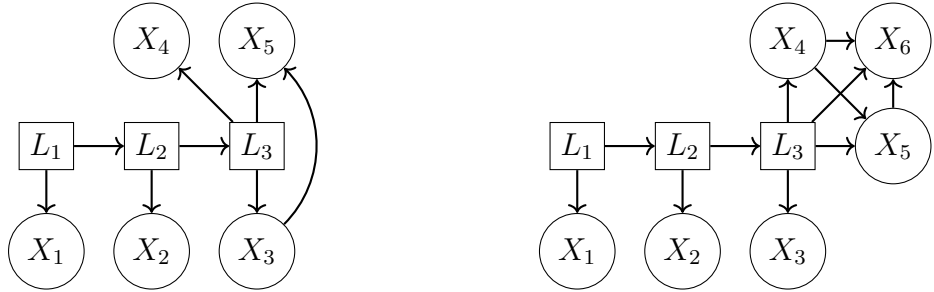
\begin{figure}[t]
    \centering
    \begin{subfigure}[t]{0.45\textwidth}
        \centering
        \begin{tikzpicture}[scale=0.7]
            \node[draw, rectangle] (L_1) at (-2, 2) {$L_{1}$};
            \node[draw, rectangle] (L_2) at (0, 2) {$L_{2}$};
            \node[draw, rectangle] (L_3) at (2, 2) {$L_{3}$};
            \node[draw, circle] (X_1) at (-2, 0) {$X_{1}$};
            \node[draw, circle] (X_2) at (0, 0) {$X_{2}$};
            
            \node[draw, circle] (X_3) at (2, 0) {$X_{3}$};
            
            \node[draw, circle] (X_4) at (0, 4) {$X_{4}$};
            \node[draw, circle] (X_5) at (2, 4) {$X_{5}$};
            \draw[thick, ->] (L_1) -- (L_2);
            \draw[thick, ->] (L_2) -- (L_3);
            \draw[thick, ->] (L_1) -- (X_1);
            \draw[thick, ->] (L_2) -- (X_2);
            \draw[thick, ->] (L_3) -- (X_3);
            \draw[thick, ->] (L_3) -- (X_4);
            \draw[thick, ->] (L_3) -- (X_5);
            \draw[thick, ->] (X_3) to [out=45, in=-45] (X_5);
        \end{tikzpicture}
        \caption{An example of LvLiNGAM with impure children (1)}
    \end{subfigure}
    \hfill
    \begin{subfigure}[t]{0.45\textwidth}
        \centering
        \begin{tikzpicture}[scale=0.7]
            \node[draw, rectangle] (L_1) at (-2, 2) {$L_{1}$};
            \node[draw, rectangle] (L_2) at (0, 2) {$L_{2}$};
            \node[draw, rectangle] (L_3) at (2, 2) {$L_{3}$};
            \node[draw, circle] (X_1) at (-2, 0) {$X_{1}$};
            \node[draw, circle] (X_3) at (2, 0) {$X_{3}$};
            \node[draw, circle] (X_2) at (0, 0) {$X_{2}$};
            \node[draw, circle] (X_5) at (4, 2) {$X_{5}$};
            \node[draw, circle] (X_4) at (2, 4) {$X_{4}$};
            \node[draw, circle] (X_6) at (4, 4) {$X_{6}$};
            \draw[thick, ->] (L_1) -- (L_2);
            \draw[thick, ->] (L_2) -- (L_3);
            \draw[thick, ->] (L_1) -- (X_1);
            \draw[thick, ->] (L_3) -- (X_4);
            \draw[thick, ->] (L_2) -- (X_2);
            \draw[thick, ->] (L_3) -- (X_3);
            \draw[thick, ->] (L_3) -- (X_5);
            \draw[thick, ->] (L_3) -- (X_6);
            \draw[thick, ->] (X_4) -- (X_6);
            \draw[thick, ->] (X_5) -- (X_6);
            \draw[thick, ->] (X_4) to (X_5);
        \end{tikzpicture}
        \caption{An example of LvLiNGAM with impure children (2)}
    \end{subfigure}
    \caption{Two examples of LvLiNGAM with impure children}
    \label{fig: impure children}
\end{figure}

The correct clustering for the model in Figure~\ref{fig: impure children}~(a) is $\{X_1\}$, $\{X_2\}$, $\{X_3, X_4, X_5\}$, and the correct clustering for the model in Figure~\ref{fig: impure children}~(b) is $\{X_1\}$, $\{X_2\}$, $\{X_3, X_4, X_5, X_6\}$. However, the above method incorrectly partitions the variables into $\{X_1\}$, $\{X_2\}$, $\{X_3,X_5\}$, $\{X_4\}$ for (a), and $\{X_1\}$, $\{X_2\}$, $\{X_3\}$, $\{X_4\}$, $\{X_5\}, \{X_6\}$ for (b), respectively. 
As in Figure~\ref{fig: impure children}~(b), when three or more variables in the same cluster form a complete graph, no pair of these observed variables satisfies the Triad constraint. 

However, even for models in which there exist causal edges among observed variables within the same cluster, it can be shown that a pair of variables satisfying the Triad constraint is a sufficient condition for them to belong to the same cluster. 
\begin{theorem}
\label{thm: clusters in proposed method}
    Assume the model satisfies Assumptions A1-A4. 
    If two dependent observed variables $X_i$ and $X_j$ satisfy the Triad constraint for all $X_k \in \bm{X} \setminus \{X_i, X_j\}$, they belong to the same cluster. 
\end{theorem}

Under Assumption A3, the presence of ancestral relationships between two observed variables implies that they belong to the same cluster. Proposition~\ref{thm: latent number} allows us to determine ancestral relationships between two observed variables. 
Using Proposition~\ref{thm: latent number}, it is possible to identify $X_3 \in \mathrm{Anc}(X_5)$ in the model of Figure~\ref{fig: impure children}(a) and $X_4 \in \mathrm{Anc}(X_5)$, $X_4 \in \mathrm{Anc}(X_6)$, and $X_5 \in \mathrm{Anc}(X_6)$ in the model of Figure~\ref{fig: impure children}(b). 

Moreover, it follows that Proposition~\ref{thm: merging cluster} also holds for the models considered in this paper. By applying it, the model in Figure~\ref{fig: impure children}(a) is clustered into $\{X_1\}$, $\{X_2\}$, $\{X_3,X_5\}$, $\{X_4\}$, while the model in Figure~\ref{fig: impure children}(b) is clustered into $\{X_1\}$, $\{X_2\}$, $\{X_3\}$, and $\{X_4, X_5, X_6\}$. 

Even when Theorem \ref{thm: clusters in proposed method} and Proposition \ref{thm: minimal cluster} are applied, the resulting clusters are generally over-segmented. To obtain the correct clusters, it is necessary to merge some of them. The correct clustering is obtained in the subsequent stage.

The algorithm for Stage I is presented in Algorithm \ref{alg: proposed cluster}. 

\begin{algorithm}[t] 
    \caption{ Estimating over-segmented clusters} 
	\label{alg: proposed cluster} 
    \footnotesize
	\begin{algorithmic}[1]
		\Require $\bm{X}=(X_1,\ldots,X_p)^\top$
		\Ensure Estimated clusters $\hat{\mathcal{C}}$ and $\mathcal{A}_{O} = \{\mathrm{Anc}(X_{i}) \mid X_{i} \in \bm{X}\}$
        \State{Initialize $\hat{\mathcal{C}} \gets \{\{X_{1}\}, \dots, \{X_{p}\}\}$, $\mathrm{Anc}(X_i) \gets \emptyset$ for $i=1,\ldots,p$}
        \ForAll{pairs $(X_{i}, X_{j})$}
            \If{$X_{i}, X_{j}$ satisfy Theorem~\ref{thm: minimal cluster} \textbf{or} have an ancestral relationship by Proposition~\ref{thm: latent number}}
                \State Merge $\{X_{i}\}$ and $\{X_{j}\}$
                \State{Update $\hat{\mathcal{C}}$ and $\mathcal{A}_{O}$}
            \EndIf
        \EndFor
        \State{Merge clusters in $\hat{\mathcal{C}}$ and update $\hat{\mathcal{C}}$ by applying Proposition~\ref{thm: merging cluster}}
        \State \Return $\hat{\mathcal{C}}$, $\mathcal{A}_{O}$
	\end{algorithmic} 
\end{algorithm}
\subsection{Stage II: Identifying the Causal Order among Latent Variables}
\label{sec: causal order}
In this section, we provide an algorithm for estimating the correct clusters and the causal order among latent variables. 
Suppose that, as a result of applying Algorithm \ref{alg: proposed cluster}, $K$ clusters $\hat{\mathcal{C}} = \{\hat{C}_1, \dots, \hat{C}_{K}\}$ are estimated. Associate a latent variable $L_i$ with each cluster $\hat{C}_i$ for $i=1,\ldots,K$, and define $\hat{\bm{L}} = \{L_1, \ldots, L_K\}$. As stated in the previous section, $K \ge q$. When $K > q$, some clusters must be merged to recover the true clustering.

$\bm{X}$ can be partitioned into maximal subsets of mutually dependent variables. Each observed variable in these subsets has a corresponding latent parent. If the causal order of the latent parents within each subset is determined, then the causal order of the entire latent variable set $\hat{\bm{L}}$ is uniquely determined. Henceforth, we assume, without loss of generality, that $\bm{X}$ itself forms one such maximal subset.
\subsubsection{Determining the Source Latent Variable}
\label{subsec: latent source 1}
Since we assume that $\bm{X}$ consists of mutually dependent variables, $\mathcal{G}$ contains only one source node among the latent variables. Theorem \ref{thm: proposed cluster root} provides the necessary and sufficient condition for a latent variable to be a source node. 
\begin{theorem}
\label{thm: proposed cluster root}
Let $X_i$ denote the observed variable with the highest causal order among $\hat{C}_i$. Then, $L_{i}$ is generically a latent source in $\mathcal{G}$ if and only if $\mathrm{Conf}(X_{i}, X_j)$ are identical across all $X_j \in \bm{X} \setminus \{X_{i}\}$ such that $X_{i} \notindep X_{j}$ in the canonical model over $X_{i}$ and $X_j$, with their common value being $\{L_i\}$. 
\end{theorem}

Note that in Stage I, the ancestral relationships among the observed variables are determined. Hence, the causal order within each cluster can also be determined. Let $X_j$ be the observed variable with the highest causal order among $\hat{C}_j$ for $j=1,\ldots, K$ and define $\bm{X}_{\mathrm{oc}}=\{X_1,\ldots,X_K\}$. When $\vert \hat{C}_i \vert \ge 2$, let $X_{i'}$ be any element in $\hat{C}_i \setminus \{X_i\}$. Define $\mathcal{X}_i$ by 
\begin{equation}
    \label{eq: calX}
\mathcal{X}_i = \left\{
\begin{array}{ll}
    \{X_{i'}\}, & \text{if } \vert\hat{C}_i \vert \ge 2\\
    \emptyset, & \text{if } \vert\hat{C}_i \vert = 1. 
\end{array}
\right.
\end{equation}
Let $L^{(i, j)}$ denote a latent confounder of $X_i$ and $X_j$ in the canonical model over them. 

In the implementation, we verify whether the conditions of Theorem \ref{thm: proposed cluster root} are satisfied by using Corollary \ref{col: proposed cluster root}.
\begin{corollary}
\label{col: proposed cluster root}
Assume $k \ge 3$.
$L_{i}$ is generically a latent source in $\mathcal{G}$ if and only if 
one of the following two cases holds: 
\begin{enumerate}
    \setlength{\itemsep}{0pt}
    \setlength{\parsep}{0pt}
    \setlength{\parskip}{0pt}
    \item $\mathcal{X}_i=\emptyset$ and $\vert \bm{X}_{\mathrm{oc}} \setminus \{X_{i}\} \vert = 1$
    \item $\vert (\bm{X}_{\mathrm{oc}} \cup \mathcal{X}_i) \setminus \{X_{i}\} \vert \ge 2$ and the following all hold:
    \begin{enumerate}
        \setlength{\itemsep}{0pt}
        \setlength{\parsep}{0pt}
        \setlength{\parskip}{0pt}
        \item In the canonical model over $X_{i}$ and $X_{j}$, $\vert \mathrm{Conf}(X_{i}, X_{j}) \vert = 1$ for $X_j \in (\bm{X}_{\mathrm{oc}} \cup \mathcal{X}_i) \setminus \{X_i\}$ such that $X_i \notindep X_j$. 
        \item $c^{(k)}_{(X_{i} \to X_{j})}(L^{(i, j)})$ are identical for $X_j \in (\bm{X}_{\mathrm{oc}} \cup \mathcal{X}_i) \setminus \{X_i\}$. 
    \end{enumerate}
\end{enumerate}
\end{corollary}

When $\mathcal{X}_i = \emptyset$ and $\lvert \bm{X}_{\mathrm{oc}} \setminus \{X_{i}\} \rvert = 1$, it is trivial by Assumption A2 that $L_i$ is a latent source. Otherwise,  for $L_i$ to be a latent source, it is necessary that $|\mathrm{Conf}(X_i, X_j)| = 1$ for all $X_j \in (\bm{X}_{\mathrm{oc}} \cup \mathcal{X}_i) \setminus \{X_i\}$. This can be verified by using Condition 1 of Proposition \ref{thm: latent number}. In addition, if $c^{(k)}_{(X_i \to X_j)}(L^{(i, j)})$ for $X_j \in (\bm{X}_{\mathrm{oc}} \cup \mathcal{X}_i) \setminus \{X_i\}$ are identical, $L_i$ can be regarded as a latent source. 

When $X_i \in \mathrm{Anc}(X_{i'})$, the equation (\ref{eq: system}) yields two distinct solutions, 
$c^{(k)}_{(X_i \to X_{i'})}(L^{(i,i')})=c^{(k)}_{(X_i \to X_{i'})}(L_i)$ and $c^{(k)}_{(X_i \to X_{i'})}(e_i)$, that are identifiable only up to a permutation of the two. 
If either of these two solutions equals $c^{(k)}_{(X_i \to X_j)}(L^{(i,j)})$ for all $X_j \in \bm{X}_{\mathrm{oc}} \setminus \{X_i\}$, 
then $L_i$ can be identified as the latent source. 

\begin{figure}[t]
    \centering
    \begin{subfigure}[t]{0.47\textwidth}
        \centering
        \begin{tikzpicture}[scale=0.7]
            \node[draw, rectangle] (L_1) at (0, 2) {$L_{1}$};
            \node[draw, circle] (X_1) at (-2, 0) {$X_{1}$};
            \node[draw, circle] (X_2) at (0, 0) {$X_{2}$};
            \node[draw, circle] (X_3) at (2, 0) {$X_{3}$};
            \draw[thick, ->] (L_1) -- (X_1);
            \draw[thick, ->] (L_1) -- (X_2);
            \draw[thick, ->] (L_1) -- (X_3);
            \draw[thick, ->] (X_2) -- (X_3);
        \end{tikzpicture}
        \caption{}
    \end{subfigure}
    \hfill
    \begin{subfigure}[t]{0.47\textwidth}
        \centering
        \begin{tikzpicture}[scale=0.7]
            \node[draw, rectangle] (L_1) at (-1, 2) {$L_{1}$};
            \node[draw, rectangle] (L_2) at (2, 2) {$L_{2}$};
            \node[draw, circle] (X_1) at (-1, 0) {$X_{1}$};
            \node[draw, circle] (X_2) at (1, 0) {$X_{2}$};
            \node[draw, circle] (X_3) at (3, 0) {$X_{3}$};
            \draw[thick, ->] (L_1) -- (X_1);
            \draw[thick, ->] (L_2) -- (X_2);
            \draw[thick, ->] (L_3) -- (X_3);
            \draw[thick, ->] (X_2) -- (X_3);
            \draw[thick, ->] (L_1) -- (L_2);
        \end{tikzpicture}
        \caption{}
    \end{subfigure}
    \caption{An example of merging clusters in Stage II}
    \label{fig: 1L4O}
\end{figure}
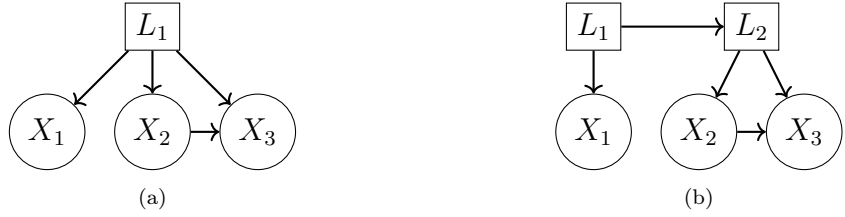

\begin{example}
    Consider the models in Figure \ref{fig: 1L4O}.
    For both models (a) and (b), the clusters estimated in Stage I are $\hat{C}_{1} = \{X_{1}\}$ and $\hat{C}_{2} = \{X_{2}, X_{3}\}$, and let $L_1$ and $L_2$ be the latent parents assigned to $\hat{C}_1$ and $\hat{C}_2$, respectively. Then, $\bm{X}_{\mathrm{oc}} = \{X_{1}, X_{2}\}$. In the model (a), we can assume $\lambda_{11}=1$ without loss of generality. Then, the model (a) is expressed as
    \begin{align*}
        X_{1} &= \epsilon_{1} + e_1, \quad X_{2} = \lambda_{21}\epsilon_{1} + e_2, \quad X_{3} = (\lambda_{21}b_{32} + \lambda_{31})\epsilon_{1} + b_{32}e_{2} + e_3.
    \end{align*}
    By Proposition \ref{lem: estimate e cumulant} and assuming $k \ge 2$, we can obtain
    \begin{align*}
        c^{(k)}_{(X_{1} \to X_{2})}(L^{(1,2)}) &= \mathrm{cum}^{(k)}(\epsilon_{1}), \\
        c^{(k)}_{(X_{2} \to X_{1})}(L^{(2,1)}) &= c^{(k)}_{(X_{2} \to X_{3})}(L^{(2,3)}) = \mathrm{cum}^{(k)}(\lambda_{21}\epsilon_{1}). 
    \end{align*}
    Since $\vert \bm{X}_{\mathrm{oc}}  \setminus \{X_{1}\} \vert = \vert \{X_{2}\} \vert = 1$ and $ c^{(k)}_{(X_{2} \to X_{1})}(L^{(2,1)}) = c^{(k)}_{(X_{2} \to X_{3})}(L^{(2,3)})$, both $L_{1}$ and $L_{2}$ are determined as latent sources. The dependence between $X_{1}$ and $X_{2}$ leads to $L_{1}$ and $L_{2}$ being regarded as a single latent source, resulting in the merging of $\hat{C}_{1}$ and $\hat{C}_{2}$.

    In the model (b), we can assume $\lambda_{11}=\lambda_{22}=1$ without loss of generality. Then, the model (b) is described as 
    \begin{align*}
        X_{1} &= \epsilon_{1} + e_1, \quad X_{2} = (a_{21}\epsilon_{1}+\epsilon_{2})+ e_2, \\ 
        X_{3} &= (b_{32} + \lambda_{31})(a_{21}\epsilon_{1}+\epsilon_{2}) + b_{32}e_{2} + e_3. 
    \end{align*}
    Then, 
    \begin{align*}
        c^{(k)}_{(X_{1} \to X_{2})}(L^{(1,2)}) &= \mathrm{cum}^{(k)}(\epsilon_{1}), \\
        c^{(k)}_{(X_{2} \to X_{1})}(L^{(2,1)}) &= \mathrm{cum}^{(k)}(a_{21}\epsilon_{1}) \ne c^{(k)}_{(X_{2} \to X_{3})}(L^{(2,3)}) = \mathrm{cum}^{(k)}(a_{21}\epsilon_{1}+\epsilon_{2}). 
    \end{align*}
    Therefore, $L_{1}$ is a latent source, while $L_{2}$ is not. 
\end{example}

As in model (a), multiple latent variables may also be identified as latent sources. In such cases, their observed children are merged into a single cluster. Once $L_i$ is established as a latent source, it implies that $L_i$ is an ancestor of the other elements in $\hat{\bm{L}}$. The procedure of Section \ref{subsec: latent source 1} is summarized in Algorithm \ref{alg: proposed causal order first}. 

\begin{algorithm}[t]
	\caption{Finding latent sources} 
	\label{alg: proposed causal order first}
    \footnotesize
	\begin{algorithmic}[1]
        \Require  Mutually dependent $\bm{X}_{\mathrm{oc}}$, $\hat{\mathcal{C}}$, and $\mathcal{A}_{L}$
        \Ensure $\bm{X}_{\mathrm{oc}}$, $\hat{\mathcal{C}}$, and a set of ancestral relationships between latent variables $\mathcal{A}_{L}$
            \State{Each cluster is assigned one latent parent and let $\hat{\bm{L}}$ be the set of latent parents}
            \State{Apply Corollary \ref{col: proposed cluster root} to find the latent sources $\bm{L}_{s}$}
            \State{Assume $L_s \in {\bm{L}}_s$ and $\hat{C}_s \in \hat{\mathcal{C}}$}
            \If{$\vert \bm{L}_{s} \vert \ge 2$}
                \State{Merge the corresponding clusters into $\hat{C}_{s}$ and update $\hat{\mathcal{C}}$ and $\bm{X}_{\mathrm{oc}}$}
                \State{Identify all latent parents in $\bm{L}_s$ with $L_s$}
            \EndIf
            \ForAll{$L_i \in \hat{\bm{L}} \setminus \bm{L}_s$}
                \State{$\mathrm{Anc}(L_i) \gets \{L_s\}$}
            \EndFor
            \State{$\bm{X}_{\mathrm{oc}} \gets \bm{X}_{\mathrm{oc}} \setminus \{X_s\}$}
            \State $\mathcal{A}_L \gets \mathcal{A}_L \cup \{\mathrm{Anc}(L_i) \mid L_i \in \hat{\bm{L}} \setminus \bm{L}_s\} \cup \{\mathrm{Anc}(L_s)=\emptyset\}$
        \\ \Return $\bm{X}_{\mathrm{oc}}$, $\hat{\mathcal{C}}$, and $\mathcal{A}_{L}$
	\end{algorithmic} 
\end{algorithm}
\subsubsection{Determining the Causal Order of Latent Variables}
\label{subsec: latent source 2}
Next, we address the identification of subsequent latent sources after finding $L_1$ in the preceding procedure. If the influence of the latent source can be removed from the observed descendant, the subsequent latent source may be identified through a procedure analogous to the one previously applied. The statistic $\tilde{e}_{(X_i, X_h)}$, defined below, serves as a key quantity for removing such influence. 

\begin{definition}
\label{def: new statistic}
    Let $X_{i}$ and $X_{h}$ be two observed variables. 
    Define $\tilde{e}_{(X_i, X_h)}$ as
    \begin{align*}
        \tilde{e}_{\left(X_{i}, X_{h}\right)} = X_{i} - 
        \rho_{\left(X_{i}, X_{h}\right)}
        X_{h}, 
    \end{align*}
    where
    \begin{align*}
        \rho_{\left(X_{i}, X_{h}\right)}=
        \left\{
        \begin{array}{lc}
           \displaystyle{\frac{\mathrm{cum}(X_{i}, X_{i}, X_{h}, X_{h})}{\mathrm{cum}(X_{i}, X_{h}, X_{h}, X_{h})}}  & X_{i} \notindep X_{h}, \\
           \ \\
           0  & X_{i} \indep X_{h}.
        \end{array}
        \right.
    \end{align*}
\end{definition}
Under Assumption 5, when $X_i \notindep X_h$, $\rho_{\left(X_{i}, X_{h}\right)}$ is shown to be generically finite and non-zero. See Lemma~\ref{lem: non-Gaussian no-zero} in the Appendix for details. Let $L_h$ be the latent source, and let $X_h$ be its observed child with the highest causal order. When there is no directed path between $X_i$ and $X_h$, $\tilde{e}_{(X_i, X_h)}$ can be regarded as $X_i$ after removing the influence of $L_h$. 
\begin{example}
\label{ex: indep}
Consider the model in Figure \ref{fig: etilde} (a). We can assume $\lambda_{22}=\lambda_{33}=1$ without loss of generality. Then, $X_1$, $X_2$ and $X_3$ are described as 
\begin{align*}
    X_1 = \epsilon_1+e_1, \quad X_2 = a_{21} \epsilon_1 + \epsilon_2 + e_2,\quad
    X_3 = a_{31}\epsilon_1 + \epsilon_3 + e_3. 
\end{align*}
We can easily show that $\rho_{(X_2,X_1)}=a_{21}$. Hence, we have
\[
\tilde{e}_{(X_2,X_1)} = - a_{21} e_1 + \epsilon_2 + e_2.
\]
It can be seen that $\tilde{e}_{(X_2,X_1)}$ does not depend on $L_1 = \epsilon_1$, and that $\tilde{e}_{(X_2,X_1)}$ and $X_3$ are mutually independent.
\end{example}
\begin{example}
\label{ex: source}
Consider the model in Figure \ref{fig: etilde} (b). We can assume that $\lambda_{11}=\lambda_{22}=\lambda_{33}=1$ without loss of generality. 
Then, the model is described as 
\begin{align*}
    X_1 &= \epsilon_1+e_1, \quad X_2 = a_{21} \epsilon_1 + \epsilon_2 + e_2,\\
    X_3 &= a_{32}a_{21}\epsilon_1 + a_{32} \epsilon_2 + \epsilon_3 + e_3,\quad
    X_4 = \lambda_{42}(a_{21} \epsilon_1 + \epsilon_2) + e_4,\\
    X_5 &= (\lambda_{53}+b_{53})(a_{32}a_{21}\epsilon_1 + a_{32} \epsilon_2 + \epsilon_3) + b_{53}e_{3} + e_5.
\end{align*}
We can easily show that $\rho_{(X_2,X_1)}=a_{21}$ and $\rho_{(X_3,X_1)}=a_{32}a_{21}$. 
Hence, we have 
\begin{align*}
    \tilde{e}_{(X_2,X_1)} &= - a_{21} e_1 + \epsilon_2 + e_2, \quad
    \tilde{e}_{(X_3,X_1)} = - a_{32}a_{21} e_1 + a_{32}\epsilon_2 + \epsilon_3 + e_3.     
\end{align*}
It can be seen that $\tilde{e}_{(X_2,X_1)}$ and $\tilde{e}_{(X_3,X_1)}$ are obtained by replacing $L_1=\epsilon_1$ with $-e_1$. The model for $(\tilde{e}_{(X_2,X_1)},X_3)$ and $(\tilde{e}_{(X_2,X_1)},X_5)$ are described by canonical models with $\mathrm{Conf}(\tilde{e}_{(X_2,X_1)},X_3)=\mathrm{Conf}(\tilde{e}_{(X_2,X_1)},X_5)=\{\epsilon_2\}$, respectively. The model for $(\tilde{e}_{(X_3,X_1)},X_2)$ and $(\tilde{e}_{(X_3,X_1)},X_5)$ are described by canonical models with $\mathrm{Conf}(\tilde{e}_{(X_3,X_1)},X_2)=\{\epsilon_2\}$ and $\mathrm{Conf}(\tilde{e}_{(X_3,X_1)},X_5)=\{a_{32} \epsilon_2 + \epsilon_3, e_{3}\}$, respectively. $X_5$ contains $\{\epsilon_1, \epsilon_2, \epsilon_3, e_3, e_5\}$, and $\tilde{e}_{(X_3,X_1)}$ contains $\{\epsilon_2, \epsilon_3, e_1, e_3\}$. Since these sets are not in an inclusion relationship, it follows from Lemma~5 of Salehkaleybar et al.~\cite{salehkaleybar2020learning} that there is no ancestral relationship between $\tilde{e}_{(X_3, X_1)}$ and $X_5$. 
\end{example}

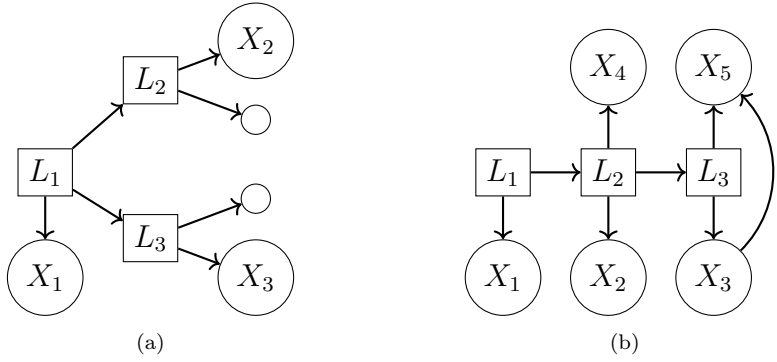
\begin{figure}[t]
    \centering
        \begin{subfigure}[t]{0.45\textwidth}
        \centering
        \begin{tikzpicture}[scale=0.7]
            \node[draw, rectangle] (L_1) at (-2, 2) {$L_{1}$};
            \node[draw, rectangle] (L_2) at (0, 3.75) {$L_{2}$};
            \node[draw, rectangle] (L_3) at (0, 0.75) {$L_{3}$};
            \node[draw, circle] (X_1) at (-2, 0) {$X_{1}$};
            \node[draw, circle] (X_4) at (2, 1.5) {};
            \node[draw, circle] (X_3) at (2, 0) {$X_{3}$};
            \node[draw, circle] (X_5) at (2, 3) {};
            \node[draw, circle] (X_2) at (2, 4.5) {$X_{2}$};
            \draw[thick, ->] (L_1) -- (L_2);
            \draw[thick, ->] (L_1) -- (L_3);
            \draw[thick, ->] (L_1) -- (X_1);
            \draw[thick, ->] (L_2) -- (X_2);
            \draw[thick, ->] (L_2) -- (X_5);
            \draw[thick, ->] (L_3) -- (X_3);
            \draw[thick, ->] (L_3) -- (X_4);
        \end{tikzpicture}
        \caption{}
    \end{subfigure}
    \begin{subfigure}[t]{0.45\textwidth}
        \centering
        \begin{tikzpicture}[scale=0.7]
            \node[draw, rectangle] (L_1) at (-2, 2) {$L_{1}$};
            \node[draw, rectangle] (L_2) at (0, 2) {$L_{2}$};
            \node[draw, rectangle] (L_3) at (2, 2) {$L_{3}$};
            \node[draw, circle] (X_1) at (-2, 0) {$X_{1}$};
            \node[draw, circle] (X_2) at (0, 0) {$X_{2}$};
            \node[draw, circle] (X_4) at (2, 0) {$X_{3}$};
            \node[draw, circle] (X_3) at (0, 4) {$X_{4}$};
            \node[draw, circle] (X_5) at (2, 4) {$X_{5}$};
            \draw[thick, ->] (L_1) -- (L_2);
            \draw[thick, ->] (L_2) -- (L_3);
            \draw[thick, ->] (L_1) -- (X_1);
            \draw[thick, ->] (L_2) -- (X_2);
            \draw[thick, ->] (L_2) -- (X_3);
            \draw[thick, ->] (L_3) -- (X_4);
            \draw[thick, ->] (L_3) -- (X_5);

            \draw[thick, ->] (X_4) to [in=-45, out=45](X_5);
        \end{tikzpicture}
        \caption{}
    \end{subfigure}
        \caption{Examples of LvLiNGAMs}
    \label{fig: etilde}
\end{figure}
It is noteworthy that $\tilde{e}_{(X_3, X_1)}$ and $X_5$ share two latent confounders, and that no ancestral relationship exists between them even though $X_3 \in \mathrm{Anc}(X_5)$ in the original graph. 

Let $L_{1}$ be the current latent source identified by the preceding procedure. Let $\mathcal{G}^{-}(\{L_1\})$ be the subgraph of $\mathcal{G}$ induced by $\bm{V} \setminus (\{L_1\} \cup \hat{C}_1)$.  
By generalizing the discussions in Examples \ref{ex: indep} and \ref{ex: source}, we obtain the following theorems. 
\begin{theorem}
\label{thm: independent e tilde}
For $X_{i}, X_{j} \in \bm{X}_{\mathrm{oc}} \setminus \{X_1\}$ and their respective latent parent $L_{i}$ and $L_{j}$, $L_{i} \indep L_{j} \mid L_1$ if and only if $\tilde{e}_{(X_{i}, X_{1})} \indep X_{j}$. 
\end{theorem}
\begin{theorem}
\label{thm: proposed cluster root next}
Let $L_i$ denote the latent parent of $X_i \in \bm{X}_{\mathrm{oc}} \setminus \{X_1\}$. If $|\hat{C}_i| \ge 2$, let $X_{i'}$ be an element of $\hat{C}_i \setminus \{X_i\}$. $\mathcal{X}_i$ is defined in the same manner as (\ref{eq: calX}).

Then, $L_{i}$ is generically a source in $\mathcal{G}^{-}(\{L_1\})$ if and only if the following two conditions hold: 
\begin{enumerate}
    \setlength{\itemsep}{0pt}
    \setlength{\parsep}{0pt}
    \setlength{\parskip}{0pt}
    \item $\mathrm{Conf}(\tilde{e}_{(X_{i}, X_{1})}, X_{j})$ are identical for all $X_{j} \in \bm{X}_{\mathrm{oc}} \setminus \{X_{1}, X_{i}\}$ such that $\tilde{e}_{(X_{i}, X_{1})} \notindep X_{j}$, 
    with their common value being $\{\epsilon_i\}$. 
    \item If $\mathcal{X}_i \ne \emptyset$, $\mathrm{Conf}(\tilde{e}_{(X_{i}, X_{1})}, X_{i'}) \cap \mathrm{Conf}(\tilde{e}_{(X_{i}, X_{1})}, X_{j})$ are identical for all $X_{j} \in \bm{X}_{\mathrm{oc}} \setminus \{X_{1}, X_{i}\}$, 
    with their common value being $\{\epsilon_i\}$.  
\end{enumerate}
\end{theorem}

By applying Theorem \ref{thm: independent e tilde}, we can obtain the family of maximal dependent subsets of $\bm{X}_{\mathrm{oc}} \setminus \{X_1\}$ in the conditional distribution given $L_1$. Theorem~\ref{thm: proposed cluster root next} allows us to verify whether $L_i$ is a latent source in $\mathcal{G}^{-}(\{L_{1}\})$. 

By recursively iterating such a procedure, the ancestral relationships among the latent variables can be identified. To achieve this, it is necessary to generalize $\tilde{e}_{(X_{i}, X_{1})}$ defined as in Definition \ref{def: tilde e general}. 
Let $\mathcal{G}^{-}(\{L_{1}, \dots, L_{s-1}\})$ denote the subgraph of $\mathcal{G}$ induced by $\bm{V}$ except for $\{L_{1}, \dots, L_{s-1}\}$ and their observed children
and $L_{1}, \dots, L_{s-1}$ be latent sources in 
\[
\mathcal{G},\mathcal{G}^{-}(\{L_{1}\}), \mathcal{G}^{-}(\{L_{1}, L_{2}\}), \dots, \mathcal{G}^{-}(\{L_{1}, \dots, L_{s-2}\}),
\] 
respectively. 
Then, $\{L_{1}, \dots, L_{s-1}\}$ has a causal order $L_{1} \prec \dots \prec L_{s-1}$.

\begin{definition}
    \label{def: tilde e general}
    For $i \ge s$, $\tilde{e}_{(X_i,\tilde{\bm{e}}_s)}$ is defined as follows.
    \begin{align*}
        \tilde{e}_{(X_{i}, \tilde{\bm{e}}_s)} = 
        \left\{ 
        \begin{array}{ll}
            X_{i}  & s = 1, \\
            X_{i} - \sum_{h=1}^{s-1}\rho_{(X_{i}, \tilde{e}_{(X_{h}, \tilde{\bm{e}}_{h})})}\tilde{e}_{(X_{h}, \tilde{\bm{e}}_{h})}  & s > 1, 
        \end{array}
        \right.  
    \end{align*}
    where $\tilde{\bm{e}}_s = (\tilde{e}_{(X_1,\tilde{\bm{e}}_1)},\ldots,\tilde{e}_{(X_{s-1},\tilde{\bm{e}}_{s-1})})$. 
\end{definition}
$\tilde{e}_{(X_i,\tilde{\bm{e}}_s)}$ can be regarded as a statistic with the information of $L_1,\ldots, L_{s-1}$ eliminated from $X_i$. 
The following lemma shows that $\tilde{e}_{(X_i,\tilde{\bm{e}}_s)}$ is obtained by replacing the information of $\epsilon_1,\ldots,\epsilon_{s-1}$ with that of $e_1,\ldots,e_{s-1}$. 
\begin{lemma}
\label{lem: new statistic equation}
    Let $X_{1}, \dots, X_{s-1}$, and $X_{i}$ be the observed children with the highest causal order of $L_{1}, \dots, L_{s-1}$, and $L_{i}$, respectively.
    $\tilde{e}_{(X_{i}, \tilde{\bm{e}}_{s})}$ can be expressed as
    \begin{align*}
        \tilde{e}_{(X_{i}, \tilde{\bm{e}}_{s})} =
         \left\{
         \begin{array}{ll}
            \epsilon_{i} + U_{[i]} & i=s \\ 
            \epsilon_{i} + \sum_{h=s}^{i-1}\alpha^{ll}_{ih}\epsilon_{h}+ e_i & i > s \text{ and } s = 1\\
            \epsilon_{i} + \sum_{h=s}^{i-1}\alpha^{ll}_{ih}\epsilon_{h} + U_{[s-1]} + e_i & i > s \text{ and } s > 1,
        \end{array} \right.
    \end{align*}
    where $U_{[i]}$ and $U_{[s-1]}$ are linear combinations of $\{e_{1}, \dots, e_{i}\}$ and $\{e_{1}, \dots, e_{s-1}\}$, respectively.
\end{lemma}

By using $\tilde{e}_{(X_i,\tilde{\bm{e}}_s)}$ in Definition \ref{def: tilde e general}, we obtain Theorems \ref{thm: independent e tilde generalized} and \ref{thm: proposed cluster root generalized}, which generalize Theorems~\ref{thm: independent e tilde} and \ref{thm: proposed cluster root next}, respectively. 
\begin{theorem}
\label{thm: independent e tilde generalized}
For $X_{i}, X_{j} \in \bm{X}_{\mathrm{oc}} \setminus \{X_{1}, \dots, X_{s-1}\}$ 
and their respective latent parent $L_{i}$ and $L_{j}$, $L_{i} \indep L_{j} \mid \{L_{1}, \dots, L_{s-1}\}$ if and only if $\tilde{e}_{(X_i,\tilde{\bm{e}}_s)} \indep X_{j}$. 
\end{theorem}
\begin{theorem}
\label{thm: proposed cluster root generalized}
Let $L_{i}$ be the latent parent of $X_{i} \in \bm{X}_{\mathrm{oc}} \setminus \{X_{1}, \dots, X_{s-1}\}$. 
If $|\hat{C}_i| \ge 2$, let $X_{i'}$ be an element of $\hat{C}_i \setminus \{X_i\}$. $\mathcal{X}_i$ is defined in the same manner as (\ref{eq: calX}). 

Then, $L_{i}$ is generically a latent source in $\mathcal{G}^{-}(\{L_{1}, \dots, L_{s-1}\})$ if and only if the following two conditions hold: 
\begin{enumerate}
    \setlength{\itemsep}{0pt}
    \setlength{\parsep}{0pt}
    \setlength{\parskip}{0pt}
    \item $\mathrm{Conf}(\tilde{e}_{(X_{i}, \tilde{\bm{e}}_s)}, X_{j})$ are identical for all $X_{j} \in \bm{X}_{\mathrm{oc}} \setminus \{X_{1}, \ldots, X_{s-1}, X_{i}\}$ such that $\tilde{e}_{(X_{i}, \tilde{\bm{e}}_s)} \notindep X_{j}$, with their common value being $\{\epsilon_i\}$. 
    \item When $\mathcal{X}_i \ne \emptyset$, $\mathrm{Conf}(\tilde{e}_{(X_{i}, \tilde{\bm{e}}_s)}, X_{i'}) \cap \mathrm{Conf}(\tilde{e}_{(X_{i}, \tilde{\bm{e}}_s)}, X_{j})$ are identical for all $X_{j} \in \bm{X}_{\mathrm{oc}} \setminus \{X_{1}, \ldots, X_{s-1}, X_{i}\}$ such that $\tilde{e}_{(X_{i}, \tilde{\bm{e}}_s)} \notindep X_{j}$, with their common value being $\{\epsilon_i\}$.  
\end{enumerate}
\end{theorem}

As in Theorem \ref{thm: independent e tilde}, by applying Theorem \ref{thm: independent e tilde generalized}, we can identify the family of maximal dependent subsets of $\bm{X}_{\mathrm{oc}} \setminus \{X_1,\ldots,X_{s-1}\}$ in the conditional distribution given $\{L_1,\ldots,L_{s-1}\}$. For each maximal dependent subset, we can apply Theorem \ref{thm: proposed cluster root generalized} to identify the next latent source. In the implementation, we verify whether the conditions of Theorem \ref{thm: proposed cluster root generalized} are satisfied using Corollary \ref{col: proposed cluster root generalized}, which generalizes Corollary \ref{col: proposed cluster root}. 
\begin{corollary}
\label{col: proposed cluster root generalized}
Assume $k \ge 3$. $L_{i}$ is generically a latent source in \\ $\mathcal{G}^{-}(\{L_{1}, \dots, L_{s-1}\})$ if and only if one of the following two cases holds: 
\begin{enumerate}
    \setlength{\itemsep}{0pt}
    \setlength{\parsep}{0pt}
    \setlength{\parskip}{0pt}
    \item $\mathcal{X}_i = \emptyset$ and $\vert \bm{X}_{\mathrm{oc}} \setminus \{X_{1}, \dots, X_{s-1}, X_{i}\} \vert = 1$.
    \item $\vert (\bm{X}_{\mathrm{oc}} \cup \mathcal{X}_{i}) \setminus \{X_{1}, \dots, X_{s-1}, X_{i}\} \vert \ge 2$, and the following all hold: 
        \begin{enumerate}
        \item In the canonical model over $\tilde{e}_{(X_{i}, \tilde{\bm{e}}_{s})}$ and $X_{j}$, $\vert \mathrm{Conf}(\tilde{e}_{(X_{i}, \tilde{\bm{e}}_{s})}, X_{j}) \vert = 1$ for all $X_{j} \in \bm{X}_{\mathrm{oc}} \setminus \{X_1,\ldots,X_{s-1}, X_{i}\}$ such that $\tilde{e}_{(X_{i}, \tilde{\bm{e}}_{s})} \notindep X_{j}$. 
        \item 
        $c^{(k)}_{(\tilde{e}_{(X_{i}, \tilde{\bm{e}}_{s})} \to X_{j})}(L^{(i, j)})$ are identical for all $X_{j} \in \bm{X}_{\mathrm{oc}} \setminus \{X_1,\ldots, X_{s-1}, X_{i}\}$ such that $\tilde{e}_{(X_{i}, \tilde{\bm{e}}_{s})} \notindep X_{j}$, where $L^{(i, j)}$ is the unique latent confounder in the canonical model over $\tilde{e}_{(X_{i}, \tilde{\bm{e}}_{s})}$ and $X_{j}$. 
        \item $\tilde{e}_{(X_{i}, \tilde{\bm{e}}_{s})}$ and $X_{i'}$ has a latent confounder $L^{(i,i')}$ in the canonical model over them that satisfies 
        $c^{(k)}_{(\tilde{e}_{(X_{i}, \tilde{\bm{e}}_{s})} \to X_{i'})}(L^{(i, i')})=c^{(k)}_{(\tilde{e}_{(X_{i}, \tilde{\bm{e}}_{s})} \to X_{j})}(L^{(i, j)})$ for all $X_{j} \in \bm{X}_{\mathrm{oc}} \setminus \{X_1,\ldots,X_{s-1},X_{i}\}$ such that $\tilde{e}_{(X_{i}, \tilde{\bm{e}}_{s})} \notindep X_{j}$, when $\mathcal{X}_{i} \ne \emptyset$.
    \end{enumerate}
\end{enumerate}
\end{corollary}

To determine whether $L_i$ is a latent source of $\mathcal{G}^{-}(\{L_1,\ldots,L_{s-1}\})$, we first examine, using Condition 1 of Proposition \ref{thm: latent number}, whether $\lvert \mathrm{Conf}(\tilde{e}_{(X_i,\tilde{\bm{e}}_s)},X_j)\rvert = 1$, as in Section \ref{subsec: latent source 1}. If $c^{(k)}_{(\tilde{e}_{(X_i,\tilde{\bm{e}}_s)} \to X_j)}(L^{(i, j)})$ are identical for $X_j \in (\bm{X}_{\mathrm{oc}} \cup \mathcal{X}_i) \setminus \{X_1,\ldots, X_{s-1}, X_i\}$, $L_i$ is identified as a latent source. As in the previous case, when $\mathcal{X}_{i} \ne \emptyset$ and $X_{i} \in \mathrm{Anc}(X_{i^{\prime}})$, the equation (\ref{eq: system}) yields two distinct solutions for the higher-order cumulants of latent confounders. Here, we determine that $L_i$ is a latent source in $\mathcal{G}^{-}(\{L_1,\ldots,L_{s-1}\})$ if either of two solutions of (\ref{eq: system}) equals to $c^{(k)}_{(\tilde{e}_{(X_{i}, \tilde{\bm{e}}_{s})} \to X_{j})}(L^{(i, j)})$ for $X_j \in \bm{X}_{\mathrm{oc}} \setminus \{X_1, \ldots, X_{s-1}, X_i\}$.

\begin{algorithm}[t]
	\caption{Finding subsequent latent sources} 
	\label{alg: proposed causal order second}
    \footnotesize
	\begin{algorithmic}[1]
		\Require $\bm{X}_{\mathrm{oc}}$, $\hat{\mathcal{C}}$, and $\mathcal{A}_{L}$
		\Ensure $\hat{\mathcal{C}}$ and $\mathcal{A}_{L}$
            \State{Apply Corollary \ref{col: proposed cluster root generalized} to find the set of latent sources $\bm{L}_{0}$ in $\mathcal{G}^-(\bm{L}_s)$}
            \If{$\bm{L}_{0} = \emptyset$}
                \State{\textbf{return} $\hat{\mathcal{C}}$ and $\mathcal{A}_{L}$}
            \EndIf
            \If{$\vert \bm{L}_{0} \vert \ge 2$}
            \ForAll{pairs $X_{i}, X_{j} \in \left(\bigcup_{k: L_{k} \in \bm{L}_{0}} \hat{C}_k \right) \cap \bm{X}_{\mathrm{oc}}$}
                \If{$\tilde{e}_{(X_{i}, \tilde{\bm{e}}_{s})} \notindep X_{j}$}
                \State{Merge $\hat{C}_{j}$ into $\hat{C}_{i}$}
                \State{$\hat{\mathcal{C}} \gets \hat{\mathcal{C}} \setminus \{\hat{C}_j\}$, $\hat{\bm{L}} \gets \hat{\bm{L}} \setminus \{L_j\}$, $\bm{X}_{\mathrm{oc}} \gets \bm{X}_{\mathrm{oc}} \setminus \{X_j\}$, 
                $\mathcal{A}_L \gets \mathcal{A}_L \setminus \{\mathrm{Anc}(L_j)\}$}
                \EndIf
            \EndFor
            \EndIf
            \ForAll{$X_{i} \in \left(\bigcup_{k: L_{k} \in \bm{L}_{0}} \hat{C}_k \right) \cap \bm{X}_{\mathrm{oc}}$}
                \State{$\bm{X}_{\mathrm{oc}}^{(i)} \gets \emptyset$}
                \ForAll{$X_j \in \bm{X}_{\mathrm{oc}} \setminus \{X_i\}$}
                    \If{$X_j \notindep \tilde{e}_{(X_{i}, \tilde{\bm{e}}_{s})}$}
                        \State{$\mathrm{Anc}(L_j) \gets \mathrm{Anc}(L_j) \cup \{L_i\}$, $\bm{X}_{\mathrm{oc}}^{(i)} \gets \bm{X}_{\mathrm{oc}}^{(i)} \cup \{X_j\}$}
                    \EndIf
                \EndFor
                \State{$\hat{\mathcal{C}}, \mathcal{A}_{L} \gets \text{ Algorithm \ref{alg: proposed causal order second} }(\bm{X}^{(i)}_{\mathrm{oc}}, \hat{\mathcal{C}}, \mathcal{A}_{L})$}
            \EndFor
            \\ \Return $\hat{\mathcal{C}}$ and $\mathcal{A}_{L}$
	\end{algorithmic} 
\end{algorithm}

\begin{algorithm}[t]
	\caption{Finding the ancestral relationships between latent variables} 
	\label{alg: proposed causal order overall}
    \footnotesize
	\begin{algorithmic}[1]
		\Require $\bm{X}$, $\mathcal{A}_O$, and ${\hat{\mathcal{C}}}$
		\Ensure $\hat{\mathcal{C}}$ and $\mathcal{A}_{L}$
        \State{$\mathcal{A}_{L}\to \emptyset$}
        \ForAll{mutually dependent $\bm{X}_{\mathrm{oc}}$}
            \State{$\bm{X}_{\mathrm{oc}},\hat{\mathcal{C}}, \mathcal{A}_{L} \gets \text{ Algorithm \ref{alg: proposed causal order first} }(\bm{X}_{\mathrm{oc}}, \hat{\mathcal{C}}, \mathcal{A}_{L})$}
            \State{$\hat{\mathcal{C}}, \mathcal{A}_{L} \gets \text{ Algorithm \ref{alg: proposed causal order second} }(\bm{X}_{\mathrm{oc}}, \hat{\mathcal{C}}, \mathcal{A}_{L})$}
        \EndFor
        \\ \Return $\hat{\mathcal{C}}$ and $\mathcal{A}_{L}$
	\end{algorithmic} 
\end{algorithm}


If multiple latent sources are identified for any element in a mutually dependent maximal subset of $\bm{X}_{\mathrm{oc}} \setminus \{X_1,\ldots, X_{s-1}\}$, the corresponding clusters must be merged.
 As latent sources are successively identified, the correct set of latent variables $\bm{L}$, the ancestral relationships among $\bm{L}$, and the correct clusters are also successively identified. 

The procedure of Section \ref{subsec: latent source 2} is presented in Algorithm~\ref{alg: proposed causal order second}. Algorithm \ref{alg: proposed causal order overall} combines Algorithms~\ref{alg: proposed causal order first} and~\ref{alg: proposed causal order second} to provide the complete procedure for Stage II. 
\begin{example}
For the model in Figure~\ref{fig: impure children} (a), the estimated clusters obtained in Stage I are $\{X_1\}$, $\{X_2\}$, $\{X_3, X_5\}$, and $\{X_4\}$, with their corresponding latent parents denoted as $L_1$, $L_2$, $L_3$, and $L_4$, respectively. Set $\bm{X}_{\mathrm{oc}}=\{X_1,X_2,X_3,X_4\}$. 

Only $X_1$ satisfies Corollary~\ref{col: proposed cluster root}, and thus $L_1$ is identified as the initial latent source. Then, we remove $X_{1}$ from $\bm{X}_{\mathrm{oc}}$ and update it to $\bm{X}_{\mathrm{oc}} = \{X_2, X_3, X_4\}$. Next, since it can be shown that only $L_2$ satisfies Corollary~\ref{col: proposed cluster root generalized}, i.e.,
\begin{align*}
c^{(3)}_{(\tilde{e}_{(X_{2}, X_{1})}\to X_{3})}(L^{(2, 3)}) &= c^{(3)}_{(\tilde{e}_{(X_{2}, X_{1})}\to X_{4})}(L^{(2, 4)}),
\end{align*}
it follows that $L_{2}$ is the latent source of $\mathcal{G}^{-}(\{L_{1}\})$. 
Similarly, we remove $X_{2}$ from the current $\bm{X}_{\mathrm{oc}}$ and update it to $\bm{X}_{\mathrm{oc}} = \{X_{3}, X_{4}\}$.

Let $X_{3'} = X_{5}$. In $\mathcal{G}^{-}(\{L_{1}, L_{2}\})$, we compute $\tilde{e}_{(X_{3}, \tilde{\bm{e}}_{3})}$ and $\tilde{e}_{(X_{4}, \tilde{\bm{e}}_{3})}$, and find that
\begin{align*}
    c^{(3)}_{(\tilde{e}_{(X_{3}, \tilde{\bm{e}}_{3})} \to X_{4})}(L^{(3, 4)}) = c^{(3)}_{(\tilde{e}_{(X_{3}, \tilde{\bm{e}}_{3})} \to X_{5})}(L^{(3, 5)}), \quad
    \vert \bm{X}_{\mathrm{oc}} \cup \emptyset \setminus \{X_{4}\}\vert = 1, 
\end{align*}
indicating both $L_{3}$ and $L_{4}$ are latent sources by Corollary \ref{col: proposed cluster root generalized}. 
Furhtermore, we conclude that $\{X_3, X_5\}$ and $\{X_4\}$ should be merged into one cluster confounded by $L_3$. 
\end{example}
\subsection{Stage III: Identifying Causal Structure among Latent Variables}
\label{sec: redundant edges}
By the end of Stage II, the clusters of observed variables have been identified, as well as the ancestral relationships among latent variables and among observed variables. The ancestral relationships among $\bm{L}$ alone do not uniquely determine the complete causal structure of $\bm{L}$. Here, we propose a bottom-up algorithm to estimate the causal structure of the latent variables. Note that if the ancestral relationships among $\bm{L}$ are known, a causal order of $\bm{L}$ can also be obtained. Theorem \ref{thm: proposed cluster root independence 2} provides an estimator of the causal coefficients between latent variables. 
\begin{theorem}
\label{thm: proposed cluster root independence 2}
Assume that $\mathrm{Anc}(L_{i}) = \{L_1, \dots, L_{i-1}\}$ with the causal order $L_1 \prec \dots \prec L_{i-1}$. Let $X_{1}, \dots, X_{i}$ be the observed children of $L_1, \dots, L_{i}$ with the highest causal order, respectively. Define $\tilde{r}_{i,k-1}$ as
\begin{align*}
    \tilde{r}_{i,k-1} = \left\{
    \begin{array}{ll}
        X_i, &  k=1\\
        X_{i} - \sum_{h= i-(k-1)}^{i-1}{a}_{ih}X_{h}, & k \ge 2
    \end{array}
    \right.
\end{align*}
When we set $\lambda_{11} = \dots = \lambda_{ii} = 1$, $a_{i, i-k} = \rho_{(\tilde{r}_{i, k-1}, \tilde{e}_{(X_{i-k}, \tilde{\bm{e}}_{i-k})})}$ generically holds. In addition, under Assumption A4, it holds generically that $a_{i,i-k}=0$ if and only if $\tilde{r}_{i, k-1} \indep  \tilde{e}_{(X_{i-k}, \tilde{\bm{e}}_{i-k})}$. 
\end{theorem}


If the only information available is the ancestral relationships among $\{L_1, \dots, L_i\}$, we cannot determine whether there is an edge $L_{i-k} \to L_i$ in $\mathcal{G}$. 
However, according to Theorem~\ref{thm: proposed cluster root independence 2}, if $\tilde{r}_{i, k-1} \indep \tilde{e}_{(X_{i-k}, \tilde{\bm{e}}_{i-k})}$, then $a_{i,i-k}=0$, and thus it follows that $L_{i-k} \to L_i$ does not exist.

Algorithm \ref{alg: proposed cut edge} describes how Theorem \ref{thm: proposed cluster root independence 2} is applied to estimate the causal structure among $\bm{L}$. 
\begin{example}
For the model in Figure \ref{fig: impure children} (a), the estimated causal order of latent variables is $L_{1} \prec L_{2} \prec L_{3}$ with $\bm{X}_{\mathrm{oc}} = \{X_1, X_2, X_3\}$. Assume initially that $L_{1}$, $L_{2}$, and $L_{3}$ form a complete graph. Then $X_1$, $X_2$, $X_3$, and $\tilde{e}_{(X_{2},\tilde{\bm{e}}_{2})}$ are 
\begin{align*}
    X_{1} &= \epsilon_{1} + e_{1}, \quad
    X_{2} = a_{21}\epsilon_{1} + \epsilon_{2} + e_{2}, \\
    X_{3} &= (a_{21}a_{32} + a_{31})\epsilon_{1} + a_{32}\epsilon_{2} + \epsilon_{3}+ e_{3}, \\
    \tilde{e}_{(X_{2}, \tilde{\bm{e}}_{2})} &= \tilde{e}_{(X_{2}, X_{1})} = \epsilon_{2} + e_2 - a_{21}e_{1}.
\end{align*}
We estimate $a_{32}$ and $a_{31}$ using Theorem \ref{thm: proposed cluster root independence 2} as follows: 
\begin{align*}
    {a}_{32} &= \rho_{(X_{3}, \tilde{e}_{(X_{2}, \tilde{\bm{e}}_{2})})}, \\
    \tilde{r}_{31} &= X_{3} - {a}_{32}X_{2} = a_{31}\epsilon_{1} + \epsilon_{3} -a_{32}e_2 + e_{3}, \\
    {a}_{31} &= \rho_{(\tilde{r}_{31}, X_{1})}. 
\end{align*}
Thus, if $\tilde{r}_{31} \indep X_1$, then $a_{31}=0$. In this case, we can conclude that $L_{1}\to L_{3}$ does not exist. 
\end{example}

\begin{algorithm}[H]
\caption{Finding causal structure among latent variables}
\label{alg: proposed cut edge}
\footnotesize
\begin{algorithmic}[1]
\Require $\bm{X}_{\mathrm{oc}}$, $\bm{L}$, $\mathcal{A}_{L}$
\Ensure An adjacency matrix $\bm{A}_\mathrm{adj}$ of $\bm{L}$
\Function{Adjacency}{$\bm{X}_{\mathrm{oc}}$, $L_i$, $\bm{L}_\mathrm{open}$, $\bm{A}_\mathrm{adj}$, $\tilde{r}_i$}
  \If{$\vert \bm{L}_\mathrm{open}\vert = 0$}
    \State{\textbf{return} $\bm{A}_\mathrm{adj}$}
  \EndIf
  \State{Initialize $\bm{L}_{\mathrm{next}} \gets \emptyset$}
  \ForAll{$L_{j} \in \bm{L}_\mathrm{open}$}
    \State{$\hat{a}_{ij} \gets 0$, $\bm{L}_\mathrm{next} \gets \bm{L}_\mathrm{next} \cup \mathrm{Pa}(L_{j})$}
    \If{$\exists \{L_{k}, L_{h}\} \subset \bm{L}_\mathrm{next}$ s.t. $L_k \in \mathrm{Anc}(L_h)$}
        \State{$\bm{L}_\mathrm{next} \gets \bm{L}_\mathrm{next} \setminus \{L_k\}$}
    \EndIf
    \If{$\tilde{r}_{i} \notindep  \tilde{e}_{(X_{j}, \tilde{\bm{e}}_{j})}$}
        \State{$\hat{a}_{ij} \gets \text{ an empirical counterpart of } a_{ij}$}
    \EndIf
    \State{$\tilde{r}_i \gets \tilde{r}_i - \hat{a}_{ij}X_j$} 
    \If{$\hat{a}_{ji} \ne 0$}
        \State{$\bm{A}_\mathrm{adj}[i,j] \gets 1$}
    \EndIf
  \EndFor
  \State{$\bm{L}_\mathrm{open}\gets \bm{L}_\mathrm{next}$}
  \State $\bm{A}_\mathrm{adj} \gets$ {\Call{Adjacency}{$\bm{X}_{\mathrm{oc}}$, $L_i$, $\bm{L}_\mathrm{open}$, $\bm{A}_\mathrm{adj}$, $\tilde{r}_i$}}
\EndFunction
\Statex
\Function{Main}{$\bm{X}_{\mathrm{oc}}$, $\bm{L}$, $\mathcal{A}_{L}$}
\State{Initialize $\bm{A}_\mathrm{adj} \gets \{0\}_{\vert \bm{L} \vert \times \vert \bm{L} \vert}$}
\ForAll{$L_i \in \bm{L}$}
    \State{$\tilde{r}_i \gets X_i$, $\bm{L}_\mathrm{open} \gets \mathrm{Pa}(L_{i})$}
    \If{$\exists \{L_{k}, L_{h}\} \subset \bm{L}_\mathrm{open}$ s.t. $L_k \in \mathrm{Anc}(L_h)$}
        \State{$\bm{L}_\mathrm{open} \gets \bm{L}_\mathrm{open} \setminus \{L_k\}$}
    \EndIf
    \State $\bm{A}_\mathrm{adj} \gets$ \Call{Adjacency}{$\bm{X}_{\mathrm{oc}}$, $L_i$, $\bm{L}_\mathrm{open}$, $\bm{A}_\mathrm{adj}$, $\tilde{r}_i$}
\EndFor
\State \Return $\bm{A}_\mathrm{adj}$
\EndFunction
\end{algorithmic}
\end{algorithm}

\subsection{Summary}
This section integrates Algorithms \ref{alg: proposed cluster}, \ref{alg: proposed causal order overall}, and \ref{alg: proposed cut edge} into Algorithm~\ref{alg: proposed overall}, which identifies the clusters of observed variables, the causal structure between latent variables, and the ancestral relationships between observed variables under the assumptions A1-A5. Since the causal clusters $\hat{\mathcal{C}}$ have been correctly identified, the directed edges from $\bm{L}$ to $\bm{X}$ are also identified. Although the ancestral relationships among observed variables can be identified, their exact causal structure remains undetermined. 
In conclusion, we obtain the following result: 
\begin{theorem}
\label{thm: final}
    Given observed data generated from an LvLiNGAM $\mathcal{M}_{\mathcal{G}}$ in (\ref{model: lvLiNGAM General}) that satisfies the assumptions A1-A5, the proposed method can identify the latent causal structure among $\bm{L}$, causal edges from $\bm{L}$ to $\bm{X}$, and ancestral relationships among $\bm{X}$. 
\end{theorem}
\begin{algorithm}[H]
	\caption{Identify the Causal Structure among Latent Variables} 
	\label{alg: proposed overall} 
    \footnotesize
	\begin{algorithmic}[1]
		\Require $\bm{X}=(X_1,\ldots,X_p)^\top$
		\Ensure $\mathcal{A}_{O}$, $\hat{\mathcal{C}}$, and $\bm{A}_\mathrm{adj}$
        \State{$\hat{\mathcal{C}}, \mathcal{A}_{O} \gets \text{ Algorithm \ref{alg: proposed cluster} }(\bm{X})$} \Comment{Estimate over-segmented clusters}
        \State{$\mathcal{A}_{L},\hat{\mathcal{C}} \gets \text{ Algorithm \ref{alg: proposed causal order overall} }(\bm{X}, \mathcal{A}_{O},\hat{\mathcal{C}})$} \Comment{Identify the causal order among latent variables}
        \State{$\bm{A}_{\mathrm{adj}} \gets \text{ Algorithm \ref{alg: proposed cut edge} }(\bm{X}_{\mathrm{oc}}, \bm{L},\mathcal{A}_{L})$} 
        \Comment{Find the causal structure among latent variables}
            \\ \Return $\mathcal{A}_{O}$, $\hat{\mathcal{C}}$, and $\bm{A}_\mathrm{adj}$
	\end{algorithmic} 
\end{algorithm}

\section{Simulations}
\label{sec: simulations}
In this section, we assess the effectiveness of the proposed method by comparing it with the algorithms proposed by Xie et al.~\cite{xie2020generalized} for estimating LiNGLaM and by Xie et al.~\cite{Xie2024GIN} for estimating LiNGLaH, as well as with ReLVLiNGAM~\cite{schkoda2024causal}, which serves as the estimation method for the canonical model with generic parameters. For convenience, we hereafter refer to both the model class introduced by Xie et al.~\cite{xie2020generalized} and its estimation algorithm as LiNGLaM, and likewise use LiNGLaH to denote both the model class and the estimation algorithm proposed by Xie et al.~\cite{Xie2024GIN}.

\subsection{Settings}
In the simulation, the true models are set to six LvLiNGAMs defined by the DAGs shown in Figures~\ref{fig: case studies} (a)-(f). We refer to these models as Models (a)-(f), respectively. All these models satisfy Assumptions A1-A3. 

All disturbances are assumed to follow a log-normal distribution, $u_{i} \sim \mathrm{Lognormal}(-1.1, 0.8)$, shifted to have zero mean by subtracting its expected value. The coefficient $\lambda_{ii}$ from $L_i$ to $X_i$ is fixed at 1. Other coefficients in $\bm{\Lambda}$ and $\bm{A}$ are drawn from $\mathrm{Uniform}(1.1, 1.5)$, while those in $\bm{B}$ are drawn from $\mathrm{Uniform}(0.5, 0.9)$. When all causal coefficients are positive, the faithfulness condition is satisfied. The higher-order cumulant of a log-normal distribution is non-zero. 

None of the models (a)-(f) is LiNGLaM or LiNGLaH. The models (a) and (b) are generic canonical models, whereas the canonical models derived from Figures 4.1 (c)-(f) do not satisfy the genericity assumption of Schkoda et al.~\cite{schkoda2024causal}. 

The sample sizes $N$ are set to $1000$, $2000$, $4000$, $8000$, and $16000$.
The number of iterations is set to $100$. We evaluate the performance of the proposed method and other methods using the following metrics.
\begin{itemize}
\setlength{\itemsep}{0pt}
\setlength{\parsep}{0pt}
\setlength{\parskip}{0pt}
    \item $N_{\mathrm{cl}}$, $N_{\mathrm{ls}}$, $N_{\mathrm{os}}$,  and $N_{\mathrm{cs}}$: The counts of iterations in which the resulting clusters, the latent structures, the ancestral relationships among $\bm{X}$, and the latent structure and the ancestral relationships among $\bm{X}$ are correctly estimated, respectively.
    \item $\mathrm{PRE}_{ll}$, $\mathrm{REC}_{ll}$, and $\mathrm{F1}_{ll}$: Averages of Precision, Recall, and F1-score of the estimated edges among latent variables, respectively, when clusters are correctly estimated.
    \item $\mathrm{PRE}_{oo}$, $\mathrm{REC}_{oo}$, and $\mathrm{F1}_{oo}$: Averages of Precision, Recall, and F1-score of the estimated causal ancestral relationships among observed variables, respectively, when clusters are correctly estimated.
\end{itemize}

\begin{figure}[t]
    \centering
    \begin{subfigure}[t]{0.31\textwidth}
        \centering
        \begin{tikzpicture}[scale=0.7]
            \node[draw, rectangle] (L_1) at (0, 2) {$L_{1}$};
            \node[draw, circle] (X_1) at (-2, 0) {$X_{1}$};
            \node[draw, circle] (X_2) at (0, 0) {$X_{2}$};
            \node[draw, circle] (X_3) at (2, 0) {$X_{3}$};
            \draw[thick, ->] (L_1) -- (X_1);
            \draw[thick, ->] (L_1) -- (X_2);
            \draw[thick, ->] (L_1) -- (X_3);
            \draw[thick, ->] (X_2) -- (X_3);
        \end{tikzpicture}
        \caption{}
    \end{subfigure}
    \hfill
    \begin{subfigure}[t]{0.31\textwidth}
        \centering
        \begin{tikzpicture}[scale=0.7]
            \node[draw, rectangle] (L_1) at (0, 2) {$L_{1}$};
            \node[draw, circle] (X_1) at (-2, 0) {$X_{1}$};
            \node[draw, circle] (X_2) at (0, 0) {$X_{2}$};
            \node[draw, circle] (X_3) at (2, 0) {$X_{3}$};
            \draw[thick, ->] (L_1) -- (X_1);
            \draw[thick, ->] (L_1) -- (X_2);
            \draw[thick, ->] (L_1) -- (X_3);
            \draw[thick, ->] (X_1) -- (X_2);
            \draw[thick, ->] (X_2) -- (X_3);
        \end{tikzpicture}
        \caption{}
    \end{subfigure}
    \hfill
    \begin{subfigure}[t]{0.31\textwidth}
        \centering
        \begin{tikzpicture}[scale=0.7]
            \node[draw, rectangle] (L_1) at (-1, 2) {$L_{1}$};
            \node[draw, rectangle] (L_2) at (1, 2) {$L_{2}$};
            \node[draw, circle] (X_1) at (-2, 0) {$X_{1}$};
            \node[draw, circle] (X_2) at (0, 0) {$X_{2}$};
            \node[draw, circle] (X_3) at (2, 0) {$X_{3}$};
            \draw[thick, ->] (L_1) -- (L_2);
            \draw[thick, ->] (L_1) -- (X_1);
            \draw[thick, ->] (L_2) -- (X_2);
            \draw[thick, ->] (L_2) -- (X_3);
            \draw[thick, ->] (X_2) -- (X_3);
        \end{tikzpicture}
        \caption{}
    \end{subfigure}


    \begin{subfigure}[t]{0.31\textwidth}
        \centering
        \begin{tikzpicture}[scale=0.7]
            \node[draw, rectangle] (L_1) at (-2, 2) {$L_{1}$};
            \node[draw, rectangle] (L_2) at (0, 2) {$L_{2}$};
            \node[draw, circle] (X_1) at (-2, 0) {$X_{1}$};
            \node[draw, circle] (X_2) at (0, 0) {$X_{2}$};
            \node[draw, circle] (X_3) at (2, 0) {$X_{3}$};
            \node[draw, circle] (X_4) at (2, 4) {$X_{4}$};
            \draw[thick, ->] (L_1) -- (L_2);
            \draw[thick, ->] (L_1) -- (X_1);
            \draw[thick, ->] (L_2) -- (X_2);
            \draw[thick, ->] (L_2) -- (X_3);
            \draw[thick, ->] (L_2) -- (X_4);
            \draw[thick, ->] (X_3) -- (X_4);
        \end{tikzpicture}
        \caption{}
    \end{subfigure}
    \hfill
    \begin{subfigure}[t]{0.31\textwidth}
        \centering
        \begin{tikzpicture}[scale=0.7]
            \node[draw, rectangle] (L_1) at (-2, 2) {$L_{1}$};
            \node[draw, rectangle] (L_2) at (0, 2) {$L_{2}$};
            \node[draw, rectangle] (L_3) at (2, 2) {$L_{3}$};
            \node[draw, circle] (X_1) at (-2, 0) {$X_{1}$};
            \node[draw, circle] (X_2) at (0, 0) {$X_{2}$};
            \node[draw, circle] (X_3) at (2, 0) {$X_{3}$};
            \node[draw, circle] (X_4) at (2, 4) {$X_{4}$};
            \draw[thick, ->] (L_1) -- (L_2);
            \draw[thick, ->] (L_2) -- (L_3);
            \draw[thick, ->] (L_1) -- (X_1);
            \draw[thick, ->] (L_2) -- (X_2);
            \draw[thick, ->] (L_3) -- (X_3);
            \draw[thick, ->] (L_3) -- (X_4);
            \draw[thick, ->] (X_3) to [in=-45, out=45](X_4);
        \end{tikzpicture}
        \caption{}
    \end{subfigure}
    \hfill
    \begin{subfigure}[t]{0.31\textwidth}
        \centering
        \begin{tikzpicture}[scale=0.7]
            \node[draw, rectangle] (L_1) at (-2, 2) {$L_{1}$};
            \node[draw, rectangle] (L_2) at (0, 2) {$L_{2}$};
            \node[draw, rectangle] (L_3) at (2, 2) {$L_{3}$};
            \node[draw, circle] (X_1) at (-2, 0) {$X_{1}$};
            \node[draw, circle] (X_2) at (0, 0) {$X_{2}$};
            \node[draw, circle] (X_3) at (2, 0) {$X_{3}$};
            \node[draw, circle] (X_4) at (2, 4) {$X_{4}$};
            \draw[thick, ->] (L_1) -- (L_2);
            \draw[thick, ->] (L_2) -- (L_3);
            \draw[thick, ->] (L_1) -- (X_1);
            \draw[thick, ->] (L_2) -- (X_2);
            \draw[thick, ->] (L_3) -- (X_3);
            \draw[thick, ->] (L_3) -- (X_4);
            \draw[thick, ->] (X_3) to [in=-45, out=45](X_4);
            \draw[thick, ->] (L_1) to [in=135, out=45](L_3);
        \end{tikzpicture}
        \caption{}
    \end{subfigure}
    \caption{Six models for simulations}
    \label{fig: case studies}
\end{figure}
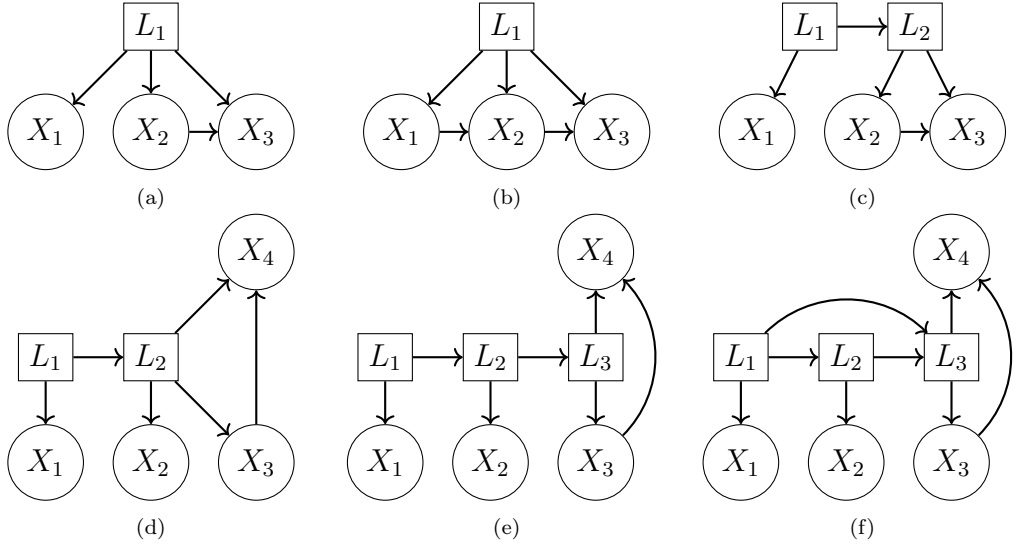

LiNGLaM and LiNGLaH assume that each cluster contains at least two observed variables. When a cluster includes only a single observed variable, these methods may fail to assign it to any cluster, resulting in it being left without an associated latent parent. Here, we treat such variables as individual clusters and assign each a latent parent. 

\subsection{Implementation}
Hilbert–Schmidt independence criterion (HSIC) \cite{Gretton2007} is employed for the independence tests in the proposed method. As HSIC becomes computationally expensive for large sample sizes, we randomly select 2,000 samples for HSIC when $N \ge 2000$. The significance level of HSIC is set to $\alpha_{\mathrm{ind}}=0.05$. 

When estimating the number of latent variables and the ancestral relationships among the observed variables, we apply Proposition \ref{thm: latent number}. Following Schkoda et al.~\cite{schkoda2024causal}, the rank of $A^{(k_{1}, k_{2})}_{(X_i \to X_j)}$ is determined from its singular values. Let $\sigma_{r}$ denote the $r$-th largest singular value of $A^{(k_{1}, k_{2})}_{(X_i \to X_j)}$ and let $\tau_{\mathrm{s}}$ be a predefined threshold. If $\sigma_{r} / \sigma_{1} \le \tau_{s}$, we set $\sigma_{r}$ to zero. To ensure termination in the estimation of the number of confounders between two observed variables, we impose an upper bound on the number of latent variables, following Schkoda et al.~\cite{schkoda2024causal}. In this experiment, we set the upper bound on the number of latent variables to two in both our proposed method and ReLVLiNGAM. 

When estimating latent sources, we use Corollaries \ref{col: proposed cluster root} and \ref{col: proposed cluster root generalized}. To check whether $\vert \mathrm{Conf}(X_i, X_j) \vert = 1$ in the canonical model over $X_i$ and $X_j$, one possible approach is to apply Proposition~\ref{thm: latent number}. Theorem~\ref{thm: one latent coef estimate} in the Appendix shows that $\vert \mathrm{Conf}(X_i, X_j) \vert = 1$ is equivalent to
\[ 
(c^{(6)}_{i, i, i, j, j, j})^{2} = c^{(6)}_{i, i, i, i, j, j}c^{(6)}_{i, i, j, j, j, j}.
\]
Based on this fact, one can alternatively check whether $\mathrm{Conf}(X_i, X_j) = 1$ by using the criterion
\begin{equation} 
\label{criterion: single confounder}
\frac{\vert (c^{(6)}_{i, i, i, j, j, j})^{2} - c^{(6)}_{i, i, i, i, j, j}c^{(6)}_{i, i, j, j, j, j} \vert}{\mathrm{max}\big((c^{(6)}_{i, i, i, j, j, j})^{2}, \vert c^{(6)}_{i, i, i, i, j, j}c^{(6)}_{i, i, j, j, j, j}\vert\big)} < \tau_{\mathrm{o}},
\end{equation}
where $\tau_{\mathrm{o}}$ is a predefined threshold. In this experiment, we compared these two approaches.

To check condition (b) of Corollary~\ref{col: proposed cluster root} and conditions (b) and (c) of Corollary~\ref{col: proposed cluster root generalized}, we use the empirical counterpart of $c^{(k)}_{(X_i \to X_{j})}(L^{(i, j)})$. In this experiment, we set $k=3$. We consider the situation of estimating the first latent source using Corollary \ref{col: proposed cluster root}. Let $\bm{c}^{(3)}_{X_{i}}$ be the set of $c^{(3)}_{(X_i \to X_{j})}(L^{(i, j)})$ for $X_{j} \in (\bm{X}_{\mathrm{oc}} \cup \mathcal{X}_{i}) \setminus \{X_{i}\}$. To show that $L_{i}$ is a latent source, it is necessary to demonstrate that all $c^{(3)}_{(X_i \to X_{j})}(L^{(i, j)}) \in \bm{c}^{(3)}_{X_{i}}$ are identical. 
Let $\bar{c}_i$ be 
\[
\bar{c}_i := \frac{1}{\vert \bm{c}^{(3)}_{X_i} \vert }\sum_{c \in \bm{c}^{(3)}_{X_{i}}}c
\]
and $s^2_i$ be the empirical counterpart of 
\[
\frac{1}{\vert \bm{c}^{(3)}_{X_i} \vert }\sum_{c \in \bm{c}^{(3)}_{X_{i}}}(c - \bar{c}_i)^{2}.
\]
Then, we regard $L_i$ as a latent source if $s^2_i$ is smaller than a given threshold $\tau_{m1}$. As mentioned previously, when $\mathcal{X}_i \ne \emptyset$, $c^{(3)}_{(X_i \to X_{i^{\prime}})}(L^{(i, i^{\prime})})$ cannot be determined, since (\ref{eq: system}) yields two distinct solutions. In this case, we compute $s_i^2$ for the two solutions, and if the smaller one is less than $\tau_{m1}$, we regard $L_i$ as a latent source. 

The estimation of the second and subsequent latent sources using Corollary~\ref{col: proposed cluster root generalized} proceeds analogously, provided that $\bm{c}_{X_i}^{(3)}$ is defined as the set of $c^{(3)}_{(\tilde{e}_{(X_i,\tilde{\bm{e}}_i)} \to X_{j})}(L^{(i, j)})$ for $X_{j} \in (\bm{X}_{\mathrm{oc}} \cup \mathcal{X}_{i}) \setminus \{X_{i}\}$. However, for the threshold applied to $s_i^2$, we use $\tau_{m2}$, which is larger than $\tau_{m1}$. This is because, as the iterations proceed, $|\bm{c}^{(3)}_{X_i}|$ decreases, and hence the variance of $s_i^2$ tends to increase. It would be desirable to increase the threshold gradually as the iterations proceed. However, in this experiment, we used the same $\tau_{m2}$ from the second iteration onward. 

In this experiment,  $(\tau_o, \tau_{m1}, \tau_{m2})=(0.001, 0.001, 0.01)$. For the model in Figure \ref{fig: case studies} (a)-(c), $\tau_{s}$ was set to $0.001$, and for the models (d)-(f) $\tau_{s}$ was set to $0.005$. 

All experiments were conducted on a workstation with a 3.0 GHz Core i9 processor and 256 GB memory.
\subsection{Results and Discussions} 
Table~\ref{tab: N} reports $N_{\mathrm{cl}}$, $N_{\mathrm{ls}}$, $N_{\mathrm{os}}$, and $N_{\mathrm{cs}}$, and Table~\ref{tab: performance edge} reports $\mathrm{PRE}_{ll}$, $\mathrm{REC}_{ll}$, $\mathrm{F1}_{ll}$, $\mathrm{PRE}_{oo}$, $\mathrm{REC}_{oo}$, and $\mathrm{F1}_{oo}$ for both the proposed and existing methods. 

Since Models (a)-(f) do not satisfy the assumptions of LiNGLaM and LiNGLaH, the results of them in Table~\ref{tab: performance edge} are omitted. The canonical models derived from Models (c)–(f) are measure-zero exceptions of the generic canonical models addressed by ReLVLiNGAM and thus cannot be identified, so the results of ReLVLiNGAM for Models (c)–(f) are not reported. Models (a) and (b) each involve only a single latent variable without latent–latent edges, so $\mathrm{PRE}_{ll}$, $\mathrm{REC}_{ll}$, and $\mathrm{F1}_{ll}$ are not reported.

Overall, the proposed method achieves superior accuracy in estimating clusters, causal relationships among latent variables, and ancestral relationships among observed variables, with the accuracy improving as the sample size increases. Only the proposed method correctly estimates both the structure of latent variables and the causal relationships among observed variables for all models. Moreover, it can be confirmed that the proposed method also correctly distinguishes the difference in latent structures between Models (e) and (f). While Models (a) and (b) are identifiable by ReLVLiNGAM, the proposed method achieves higher accuracy in both estimations for clusters. While the proposed method shows lower performance than ReLVLiNGAM in estimating ancestral relationships among observed variables for Model (b), its performance gradually approaches that of ReLVLiNGAM as the sample size increases.

In addition, when comparing the proposed method with and without Theorem~\ref{thm: one latent coef estimate}, the version incorporating Theorem~\ref{thm: one latent coef estimate} outperforms the one without it in most cases.

Although Models (a) and (b) do not satisfy the assumptions of LiNGLaM and LiNGLaH, and thus, in theory, these methods cannot identify the models, Table~\ref{tab: N} shows that they occasionally recover the single-cluster structure when the sample size is relatively small. It can also be seen from Table~\ref{tab: N} that the ancestral relationships among the observed variables are not estimated correctly at all.

As mentioned above, in the original LiNGLaM and LiNGLaH, clusters consisting of a single observed variable are not output and are instead treated as ungrouped variables. In this experiment, by regarding such ungrouped variables as clusters, higher clustering accuracy is achieved in Models (c), (e), and (f). Theoretically, it can also be shown that LiNGLaM is able to identify the clusters in Models (c), (e), and (f), while LiNGLaH can identify the clusters in Model (c). However, Table~\ref{tab: N} clearly shows that neither LiNGLaM nor LiNGLaH can correctly estimate the causal structure among latent variables or the ancestral relationships among observed variables. On the other hand, Table~\ref{tab: N} also shows that LiNGLaM and LiNGLaH fail to correctly estimate the clusters in Models (a), (b), and (d). 
This result suggests that the clustering algorithms of LiNGLaM and LiNGLaH are not applicable to all models in this paper.

\begin{table}[t]
\centering
\caption{The performance in terms of $N_\mathrm{cl}$, $N_\mathrm{ls}$, $N_\mathrm{os}$, and $N_\mathrm{cs}$}
\label{tab: N}
\resizebox{\textwidth}{!}{%
\begin{tabular}{c|c|ccccc|ccccc|ccccc|ccccc}
\hline
\multirow{2}{*}{Model} & \multirow{2}{*}{Method} & \multicolumn{5}{c|}{$N_\mathrm{cl}$} & \multicolumn{5}{c|}{$N_\mathrm{ls}$} & \multicolumn{5}{c|}{$N_\mathrm{os}$} & \multicolumn{5}{c}{$N_\mathrm{cs}$} \\
\cline{3-22}
 &  & 1K & 2K & 4K & 8K & 16K & 1K & 2K & 4K & 8K & 16K & 1K & 2K & 4K & 8K & 16K & 1K & 2K & 4K & 8K & 16K\\
\hline\hline
\multirow{5}{*}{(a)} 
& Proposed (A.7)        & \textbf{60} & \textbf{56} & \textbf{73} & \textbf{74} & \textbf{78} & \textbf{60} & \textbf{56} & \textbf{73} & \textbf{74} & \textbf{78} & \textbf{45} & \textbf{53} & \textbf{70} & \textbf{68} & \textbf{73} & \textbf{45} & \textbf{53} & \textbf{70} & \textbf{68} & \textbf{73}   \\
& Proposed              & \textbf{60} & \textbf{56} & \textbf{73} & \textbf{74} & \textbf{78} & \textbf{60} & \textbf{56} & \textbf{73} & \textbf{74} & \textbf{78} & 39 & 45 & 67 & 64 & 72 & 39 & 45 & 67 & 64 & 72   \\
 & LiNGLaM              & 10 & 1 & 0 & 0 & 0 & 10 & 1 & 0 & 0 & 0 & 0 & 0 & 0 & 0 & 0 & 0 & 0 & 0 & 0 & 0  \\
 & LiNGLaH              & 59 & 29 & 5 & 8 & 7 & 59 & 29 & 5 & 8 & 7 & 0 & 0 & 0 & 0 & 0 & 0 & 0 & 0 & 0 & 0  \\
 & ReLVLiNGAM           & 47 & 50 & 49 & 55 & 64 & 47 & 50 & 49 & 55 & 64 & 0 & 0 & 0 & 0 & 0 & 0 & 0 & 0 & 0 & 0  \\
\hline
\multirow{5}{*}{(b)} 
& Proposed (A.7)     & \textbf{62} & \textbf{75} & \textbf{86} & \textbf{92} & \textbf{93} & \textbf{62} & \textbf{75} & \textbf{86} & \textbf{92} & \textbf{93} & 11 & 23 & 34 & 53 & 60 & 11 & 23 & 34 & 53 & 60 \\
& Proposed      & 61 & \textbf{75} & \textbf{86} & \textbf{92} & \textbf{93} & 61 & \textbf{75} & \textbf{86} & \textbf{92} & \textbf{93} & 11 & 23 & 34 & 53 & 60 & 11 & 23 & 34 & 53 & 60 \\
 & LiNGLaM      & 0 & 0 & 0 & 0 & 0 & 0 & 0 & 0 & 0 & 0 & 0 & 0 & 0 & 0 & 0 & 0 & 0 & 0 & 0 & 0  \\
 & LiNGLaH      & 1 & 0 & 0 & 0 & 0 & 1 & 0 & 0 & 0 & 0 & 0 & 0 & 0 & 0 & 0 & 0 & 0 & 0 & 0 & 0  \\
& ReLVLiNGAM    & 54 & 60 & 78 & 79 & 74 & 54 & 60 & 78 & 79 & 74 & \textbf{32} & \textbf{41} & \textbf{55} & \textbf{65} & \textbf{68} & \textbf{32} & \textbf{41} & \textbf{55} & \textbf{65} & \textbf{68}  \\
\hline
\multirow{4}{*}{(c)} 
& Proposed (A.7) & 76 & 78 & 79 & 88 & 93 & \textbf{76} & \textbf{78} & \textbf{79} & \textbf{88} & 93 & \textbf{53} & \textbf{69} & \textbf{77} & \textbf{87} & \textbf{93} & \textbf{53} & \textbf{69} & \textbf{77} & \textbf{87} & \textbf{93}   \\
& Proposed & 76 & 78 & 79 & 88 & 94 & \textbf{76} & \textbf{78} & \textbf{79} & \textbf{88} & \textbf{94} & 47 & 55 & 63 & 79 & 78 & 47 & 55 & 63 & 79 & 78   \\
 & LiNGLaM & 87 & 90 & 90 & 93 & 90 & 0 & 0 & 0 & 0 & 0 & 0 & 0 & 0 & 0 & 0 & 0 & 0 & 0 & 0 & 0   \\
 & LiNGLaH & \textbf{98} & \textbf{99} & \textbf{97} & \textbf{99} & \textbf{99} & 0 & 0 & 0 & 0 & 0 & 0 & 0 & 0 & 0 & 0 & 0 & 0 & 0 & 0 & 0   \\
\hline
\multirow{4}{*}{(d)} 
& Proposed (A.7) & 44 & 24 & 38 & 32 & 63 & 44 & 24 & 38 & 32 & 63 & \textbf{10} & \textbf{22} & \textbf{30} & \textbf{24} & \textbf{58} & \textbf{10} & \textbf{22} & \textbf{30} & \textbf{24} & \textbf{58}   \\
& Proposed & \textbf{48} & \textbf{26} & \textbf{49} & \textbf{55} & \textbf{71} & \textbf{48} & \textbf{26} & \textbf{49} & \textbf{55} & \textbf{71} & 8 & 8 & 20 & 19 & 21 & 8 & 8 & 20 & 19 & 21   \\
 & LiNGLaM & 38 & 14 & 9 & 8 & 8 & 0 & 0 & 0 & 0 & 0 & 0 & 0 & 0 & 0 & 0 & 0 & 0 & 0 & 0 & 0   \\
 & LiNGLaH & 0 & 0 & 0 & 0 & 0 & 0 & 0 & 0 & 0 & 0 & 0 & 0 & 0 & 0 & 0 & 0 & 0 & 0 & 0 & 0   \\
\hline
\multirow{4}{*}{(e)} 
& Proposed (A.7) & 37 & 44 & 68 & 75 & \textbf{88} & \textbf{27} & \textbf{39} & \textbf{57} & 72 & \textbf{83} & \textbf{36} & \textbf{42} & \textbf{62} & \textbf{69} & \textbf{80} & \textbf{26} & \textbf{37} & \textbf{52} & \textbf{66} & \textbf{75} \\
& Proposed & 37 & 33 & 51 & 84 & 86 & 21 & 23 & 49 & \textbf{73} & \textbf{83} & 21 & 17 & 27 & 30 & 32 & 12 & 11 & 26 & 25 & 30 \\
 & LiNGLaM & \textbf{96} & \textbf{90} & \textbf{91} & \textbf{94} & 87 & 0 & 0 & 0 & 0 & 0 & 0 & 0 & 0 & 0 & 0 & 0 & 0 & 0 & 0 & 0 \\
 & LiNGLaH & 0 & 0 & 0 & 0 & 0 & 0 & 0 & 0 & 0 & 0 & 0 & 0 & 0 & 0 & 0 & 0 & 0 & 0 & 0 & 0 \\
\hline
\multirow{4}{*}{(f)} 
& Proposed (A.7) & 30 & 47 & 52 & 71 & 76 & \textbf{12} & 34 & 38 & \textbf{70} & \textbf{76} & \textbf{30} & \textbf{47} & \textbf{52} & \textbf{67} & \textbf{74} & \textbf{12} & \textbf{34} & \textbf{38} & \textbf{66} & \textbf{74} \\
& Proposed & 18 & 46 & 45 & 57 & 72 & 5 & \textbf{35} & \textbf{41} & 54 & 72 & 17 & 34 & 39 & 46 & 67 & 4 & 27 & 35 & 44 & 67 \\
 & LiNGLaM & \textbf{92} & \textbf{88} & \textbf{93} & \textbf{87} & \textbf{92} & 0 & 0 & 0 & 0 & 0 & 0 & 0 & 0 & 0 & 0 & 0 & 0 & 0 & 0 & 0 \\
 & LiNGLaH & 0 & 0 & 0 & 0 & 0 & 0 & 0 & 0 & 0 & 0 & 0 & 0 & 0 & 0 & 0 & 0 & 0 & 0 & 0 & 0 \\
\hline
\end{tabular}%
}
\end{table}

\begin{table}[t]
\centering
\caption{The performances in terms of $\mathrm{PRE}_{ll}$, $\mathrm{REC}_{ll}$, $\mathrm{F1}_{ll}$, $\mathrm{PRE}_{oo}$, $\mathrm{REC}_{oo}$, and $\mathrm{F1}_{oo}$}
\label{tab: performance edge}
\resizebox{\textwidth}{!}{%
\begin{tabular}{c|c|ccccc|ccccc|ccccc}
\hline
\multirow{2}{*}{Model} & \multirow{2}{*}{Method} & \multicolumn{5}{c|}{$\mathrm{PRE}_{ll}$} & \multicolumn{5}{c|}{$\mathrm{REC}_{ll}$} & \multicolumn{5}{c}{$\mathrm{F1}_{ll}$}\\
\cline{3-17}
 &  & 1K & 2K & 4K & 8K & 16K & 1K & 2K & 4K & 8K & 16K & 1K & 2K & 4K & 8K & 16K\\
\hline\hline
\multirow{2}{*}{(c)} 
& Proposed (A.7) & \textbf{1.000} & \textbf{1.000} & \textbf{1.000} & \textbf{1.000} & \textbf{1.000} & \textbf{1.000} & \textbf{1.000} & \textbf{1.000} & \textbf{1.000} & \textbf{1.000} & \textbf{1.000} & \textbf{1.000} & \textbf{1.000} & \textbf{1.000} & \textbf{1.000} \\
& Proposed & \textbf{1.000} & \textbf{1.000} & \textbf{1.000} & \textbf{1.000} & \textbf{1.000} & \textbf{1.000} & \textbf{1.000} & \textbf{1.000} & \textbf{1.000} & \textbf{1.000} & \textbf{1.000} & \textbf{1.000} & \textbf{1.000} & \textbf{1.000} & \textbf{1.000} \\
\hline
\multirow{2}{*}{(d)} 
& Proposed (A.7) & \textbf{1.000} & \textbf{1.000} & \textbf{1.000} & \textbf{1.000} & \textbf{1.000} & \textbf{1.000} & \textbf{1.000} & \textbf{1.000} & \textbf{1.000} & \textbf{1.000} & \textbf{1.000} & \textbf{1.000} & \textbf{1.000} & \textbf{1.000} & \textbf{1.000} \\
& Proposed & \textbf{1.000} & \textbf{1.000} & \textbf{1.000} & \textbf{1.000} & \textbf{1.000} & \textbf{1.000} & \textbf{1.000} & \textbf{1.000} & \textbf{1.000} & \textbf{1.000} & \textbf{1.000} & \textbf{1.000} & \textbf{1.000} & \textbf{1.000} & \textbf{1.000}\\
\hline
\multirow{2}{*}{(e)} 
& Proposed (A.7) & \textbf{0.869} & \textbf{0.962} & 0.946 & \textbf{0.987} & 0.981 & \textbf{0.905} & \textbf{1.000} & \textbf{1.000} & \textbf{1.000} & \textbf{1.000} & \textbf{0.884} & \textbf{0.977} & 0.968 & \textbf{0.992} & \textbf{0.989} \\
& Proposed & 0.824 & 0.884 & \textbf{0.987} & 0.956 & \textbf{0.988} & 0.865 & 0.955 & 1.000 & 1.000 & 1.000 & 0.837 & 0.912 & \textbf{0.992} & 0.974 & 0.993\\
\hline
\multirow{2}{*}{(f)} 
& Proposed (A.7) & \textbf{1.000} & \textbf{1.000} & \textbf{1.000} & \textbf{1.000} & \textbf{1.000} & \textbf{0.800} & 0.908 & 0.910 & \textbf{0.995} & \textbf{1.000} & \textbf{0.880} & 0.945 & 0.946 & \textbf{0.997} & \textbf{1.000} \\
& Proposed & \textbf{1.000} & \textbf{1.000} & \textbf{1.000} & \textbf{1.000} & \textbf{1.000} & 0.759 & \textbf{0.920} & \textbf{0.970} & 0.982 & \textbf{1.000} & 0.856 & \textbf{0.952} & \textbf{0.982} & 0.989 & \textbf{1.000}\\
\hline
\multirow{2}{*}{Model} & \multirow{2}{*}{Method} & \multicolumn{5}{c|}{$\mathrm{PRE}_{oo}$} & \multicolumn{5}{c|}{$\mathrm{REC}_{oo}$} & \multicolumn{5}{c}{$\mathrm{F1}_{oo}$}\\
\cline{3-17}
 &  & 1K & 2K & 4K & 8K & 16K & 1K & 2K & 4K & 8K & 16K & 1K & 2K & 4K & 8K & 16K\\
\hline\hline
\multirow{3}{*}{(a)} 
& Proposed (A.7)  &  \textbf{0.825} & \textbf{0.964} & \textbf{0.970} & \textbf{0.957} & \textbf{0.962} & \textbf{0.900} & \textbf{0.982} & \textbf{0.986} & \textbf{1.000} & \textbf{1.000} & \textbf{0.850} & \textbf{0.970} & \textbf{0.975} & \textbf{0.971} & \textbf{0.972}\\
& Proposed    & 0.733 & 0.821 & 0.929 & 0.903 & 0.949 & 0.817 & 0.839 & 0.945 & 0.946 & 0.987 & 0.761 & 0.827 & 0.934 & 0.917 & 0.959\\
 & ReLVLiNGAM &  0.262 & 0.273 & 0.320 & 0.321 & 0.323 & 0.787 & 0.820 & 0.959 & 0.964 & 0.969 & 0.394 & 0.410 & 0.480 & 0.482 & 0.484\\
\hline
\multirow{3}{*}{(b)} 
& Proposed (A.7) &    0.895 & \textbf{0.951} & \textbf{0.984} & 1.000 & \textbf{1.000} & 0.586 & 0.702 & 0.756 & 0.855 & 0.878 & 0.687 & 0.790 & 0.837 & 0.912 & 0.926\\
& Proposed & \textbf{0.902} & \textbf{0.951} & 0.984 & \textbf{1.000} & \textbf{1.000} & 0.590 & 0.702 & 0.756 & 0.855 & 0.878 & 0.692 & 0.790 & 0.837 & 0.912 & 0.926\\
 & ReLVLiNGAM & 0.827 & 0.872 & 0.880 & 0.941 & 0.973 & \textbf{0.827} & \textbf{0.872} & \textbf{0.880} & \textbf{0.941} & \textbf{0.973} & \textbf{0.827} & \textbf{0.872} & \textbf{0.880} & \textbf{0.941} & \textbf{0.973}\\
\hline
\multirow{2}{*}{(c)} 
& Proposed (A.7) & \textbf{0.697} & \textbf{0.885} & \textbf{0.975} & \textbf{0.989} & \textbf{1.000} & \textbf{0.697} & \textbf{0.885} & \textbf{0.975} & \textbf{0.989} & \textbf{1.000} & \textbf{0.697} & \textbf{0.885} & \textbf{0.975} & \textbf{0.989} & \textbf{1.000} \\
& Proposed & 0.618 & 0.705 & 0.797 & 0.898 & 0.830 & 0.618 & 0.705 & 0.797 & 0.898 & 0.830 & 0.618 & 0.705 & 0.797 & 0.898 & 0.830\\
\hline
\multirow{2}{*}{(d)} 
& Proposed (A.7) & \textbf{0.392} & \textbf{0.931} & \textbf{0.816} & \textbf{0.818} & \textbf{0.944} & \textbf{0.614} & \textbf{0.958} & \textbf{0.868} & \textbf{0.906} & \textbf{0.968} & \textbf{0.456} & \textbf{0.938} & \textbf{0.829} & \textbf{0.844} & \textbf{0.952} \\
& Proposed & 0.167 & 0.308 & 0.408 & 0.345 & 0.296 & 0.167 & 0.308 & 0.408 & 0.345 & 0.296 & 0.167 & 0.308 & 0.408 & 0.345 & 0.296\\
\hline
\multirow{2}{*}{(e)} 
& Proposed (A.7) & \textbf{0.973} & \textbf{0.955} & \textbf{0.912} & \textbf{0.920} & \textbf{0.909} & \textbf{0.973} & \textbf{0.955} & \textbf{0.912} & \textbf{0.920} & 0\textbf{.909} & \textbf{0.973} & \textbf{0.955} & \textbf{0.912} & \textbf{0.920} & \textbf{0.909} \\
& Proposed & 0.568 & 0.515 & 0.529 & 0.357 & 0.372 & 0.568 & 0.515 & 0.529 & 0.357 & 0.372 & 0.568 & 0.515 & 0.529 & 0.357 & 0.372\\
\hline
\multirow{2}{*}{(f)} 
& Proposed (A.7) & \textbf{1.000} & \textbf{1.000} & \textbf{1.000} & \textbf{0.944} & \textbf{0.974} & \textbf{1.000} & \textbf{1.000} & \textbf{1.000} & \textbf{0.944} & \textbf{0.974} & \textbf{1.000} & \textbf{1.000} & \textbf{1.000} & \textbf{0.944} & \textbf{0.974} \\
& Proposed & 0.944 & 0.739 & 0.867 & 0.807 & 0.931 & 0.944 & 0.739 & 0.867 & 0.807 & 0.931 & 0.944 & 0.739 & 0.867 & 0.807 & 0.931\\
\hline
\end{tabular}%
}
\end{table}

\begin{table}
\centering
\caption{The performances of the proposed method in $N_{cs}$ with small sample sizes}
\label{tab: small sample}
\resizebox{\textwidth}{!}{%
\begin{tabular}{c|ccccc|ccccc|ccccc|ccccc} \hline
$N$              & \multicolumn{5}{|c}{50}        &  \multicolumn{5}{|c}{100} &  \multicolumn{5}{|c}{200} &  \multicolumn{5}{|c}{400}\\ \hline \hline
\diagbox{$\tau_{o}$}{$\alpha_{\mathrm{ind}}$} & 0.01 & 0.05 & 0.1 & 0.2 & 0.3  & 0.01 & 0.05 & 0.1 & 0.2 & 0.3 & 0.01 & 0.05 & 0.1 & 0.2 & 0.3  & 0.01 & 0.05 & 0.1 & 0.2 & 0.3 \\  \hline
0.001          & 0    & 0    & 1   & 2   & \textbf{\textit{6}}  & 0 & \textit{3} & \textit{4} &  \textbf{\textit{10}} & \textit{9} & 0    & 6    & \textit{9}   & \textit{13}  & \textbf{\textit{18}} & \textit{5} & \textit{11} & \textbf{16} & 12 & 11\\ 
0.01           & 0    & \textit{1}    & 1   & \textbf{4}   & \textbf{4}  & 0 & 1 & 2 &  3 & \textbf{6} & \textit{1}    & \textit{7}    & 7   & \textbf{11}  & \textbf{11} & 2 & 7  & \textit{19} & 16 & \textit{15} \\ 
0.1            & 0    & \textit{1}    & \textit{2}   & \textbf{\textit{5}}   & 3  & 0 & 0 & \textit{4} &  \textbf{9} & 7 & \textit{1}    & 4    & 6   & 11  & \textbf{13} & 2 & \textit{15} & 17 & \textbf{\textit{18}} & 10\\\hline
\end{tabular}
}
\end{table}
\subsection{Additional Experiments with Small Sample Sizes}
\label{sec: small samples}
In the preceding experiments, the primary objective was to examine the identifiability of the proposed method, and hence the sample size was set to be sufficiently large. However, in practical applications, it is also crucial to evaluate the estimation accuracy when the sample size is limited. When the sample size is not large, the Type II error rate of HSIC increases, which in turn raises the risk of misclassifying clusters. Moreover, with small samples, the variability of the left-hand side of (\ref{criterion: single confounder}) becomes larger, thereby affecting the accuracy of Corollaries~\ref{col: proposed cluster root} and~\ref{col: proposed cluster root generalized}. To address this, we investigate whether the estimation accuracy of the model can be improved in small-sample settings by employing relatively larger values of the significance level $\alpha_{\mathrm{ind}}$ for HSIC and the threshold $\tau_o$ than those used in the previous experiments.

We conduct additional experiments under small-sample settings using Model (f) in Figure~\ref{fig: case studies}. 
The sample sizes $N$ are set to 50, 100, 200, and 400. In these experiments, only $N_{\mathrm{cs}}$ is used as the evaluation metric.
The parameters $(\tau_{s}, \tau_{m1}, \tau_{m2})$ are set to $(0.005, 0.001, 0.01)$, while the significance level of HSIC is chosen from $\alpha_{\mathrm{ind}} \in \{0.01, 0.05, 0.1, 0.2, 0.3\}$, and $\tau_{o} \in \{0.001, 0.01, 0.1\}$. 

Table~\ref{tab: small sample} reports the values of $N_{\mathrm{cs}}$ for each combination of $\alpha_{\mathrm{ind}}$ and $\tau_o$. 
The values in bold represent the best performances with fixed $N$ and $\tau_{o}$, and those in italic represent the best performances with fixed $N$ and $\alpha_{\mathrm{ind}}$.
Although the estimation accuracy is not satisfactory when the sample size is small, the results in Table~\ref{tab: small sample} suggest that relatively larger settings of $\alpha_{\mathrm{ind}}$ and $\tau_o$ tend to yield higher accuracy. The determination of appropriate threshold values for practical applications remains an important issue for future work. 

\section{Real-World Example}
\label{sec: real world}
We applied the proposed method to the Political Democracy dataset \cite{bollen1989structural}, a widely used benchmark in structural equation modeling (SEM). Originally introduced by Bollen \cite{bollen1989structural}, this dataset was designed to examine the relation between the level of industrialization and the level of political democracy across 75 countries in 1960 and 1965. It includes indicators for both industrialization and political democracy in each year, and is typically modeled using confirmatory factor analysis (CFA) as part of a structural equation model. In the standard SEM formulation, the model consists of three latent variables: {\bf ind60}, representing the level of industrialization in 1960; and {\bf dem60} and {\bf dem65}, representing the level of political democracy in 1960 and 1965, respectively. {\bf ind60} is measured by per capita GNP ($X_1$), per capita energy consumption ($X_2$), and the percentage of the labor force in nonagricultural sectors ($X_3$). {\bf dem60} and {\bf dem65} are each measured by four indicators: press freedom ($Y_1$, $Y_5$), freedom of political opposition ($Y_2$, $Y_6$), fairness of elections ($Y_3$, $Y_7$), and effectiveness of the elected legislatures ($Y_4$, $Y_8$). The SEM in Bollen \cite{bollen1989structural} specifies paths from {\bf ind60} to both {\bf dem60} and {\bf dem65}, and from {\bf dem60} to {\bf dem65}.

The marginal model for $X_{1},X_{2}$ and $Y_{3},\ldots,Y_{6}$ in the model in Bollen \cite{bollen1989structural} is as shown in Figure \ref{fig: real-world}~(a). This marginal model satisfies the assumptions A1-A3, as well as those of LiNGLaM \cite{xie2020generalized} and LiNGLaH \cite{Xie2024GIN}. We examined whether the proposed method, LiNGLaM, and LiNGLaH can recover the model in Figure \ref{fig: real-world}~(a) from observational data $X_{1},X_{2}$ and $Y_{3},\ldots,Y_{6}$. We set $(\tau_{\mathrm{s}}, \tau_{\mathrm{m1}}, \tau_{\mathrm{m2}}) = (0.005, 0.001, 0.01)$. Since the sample size is as small as $N=75$, we set $(\tau_{\mathrm{o}}, \alpha_{\mathrm{ind}}) = (0.1, 0.2)$, which are relatively large values, following the discussion in Section~\ref{sec: small samples}. The upper bound on the number of latent variables is set to $2$.

The resulting DAGs obtained by each method are shown in Figure~\ref{fig: real-world}~(b)–(d). Among them, the proposed method estimates the same DAG as in Bollen~\cite{bollen1989structural}. LiNGLaM fails to estimate the correct clusters and the causal structure among the latent variables. LiNGLaH incorrectly clusters all observed variables into two clusters. This result indicates that the proposed method not only outperforms existing methods such as LiNGLaM and LiNGLaH for models to which those methods are applicable, but is also effective even when the sample size is not large. 

\begin{figure}[t]
    \centering
    \begin{subfigure}[t]{1.0\textwidth}
        \centering
        \begin{tikzpicture}[scale=0.7]
            \node[draw, rectangle] (L_1) at (-2.5, 0) {\small \textbf{ind60}};
            \node[draw, rectangle] (L_2) at (0, 1) {\small \textbf{dem60}};
            \node[draw, rectangle] (L_3) at (0, -1) {\small \textbf{dem65}};
            \node[draw, circle] (X_1) at (-2.5, 2) {$X_{1}$};
            \node[draw, circle] (X_2) at (-2.5, -2) {$X_{2}$};
            \node[draw, circle] (Y_3) at (-1, 3) {$Y_{3}$};
            \node[draw, circle] (Y_4) at (1, 3) {$Y_{4}$};
            \node[draw, circle] (Y_5) at (-1, -3) {$Y_{5}$};
            \node[draw, circle] (Y_6) at (1, -3) {$Y_{6}$};
            \draw[thick, ->] (L_1) -- (L_2);
            \draw[thick, ->] (L_2) -- (L_3);
            \draw[thick, ->] (L_1) -- (L_3);
            \draw[thick, ->] (L_1) -- (X_1);
            \draw[thick, ->] (L_1) -- (X_2);
            \draw[thick, ->] (L_2) -- (Y_3);
            \draw[thick, ->] (L_2) -- (Y_4);
            \draw[thick, ->] (L_3) -- (Y_5);
            \draw[thick, ->] (L_3) -- (Y_6);
        \end{tikzpicture}
        \caption{Model in Bollen \cite{bollen1989structural}}
    \end{subfigure}
    \hfill
    \begin{subfigure}[t]{0.3\textwidth}
        \centering
        \begin{tikzpicture}[scale=0.7]
            \node[draw, rectangle] (L_1) at (-2, 0) {$L_{1}$};
            \node[draw, rectangle] (L_2) at (0, 1) {$L_{2}$};
            \node[draw, rectangle] (L_3) at (0, -1) {$L_{3}$};
            \node[draw, circle] (X_1) at (-2.5, 2) {$X_{1}$};
            \node[draw, circle] (X_2) at (-2.5, -2) {$X_{2}$};
            \node[draw, circle] (Y_3) at (-1, 3) {$Y_{3}$};
            \node[draw, circle] (Y_4) at (1, 3) {$Y_{4}$};
            \node[draw, circle] (Y_5) at (-1, -3) {$Y_{5}$};
            \node[draw, circle] (Y_6) at (1, -3) {$Y_{6}$};
            \draw[thick, ->] (L_1) -- (L_2);
            \draw[thick, ->] (L_2) -- (L_3);
            \draw[thick, ->] (L_1) -- (L_3);
            \draw[thick, ->] (L_1) -- (X_1);
            \draw[thick, ->] (L_1) -- (X_2);
            \draw[thick, ->] (L_2) -- (Y_3);
            \draw[thick, ->] (L_2) -- (Y_4);
            \draw[thick, ->] (L_3) -- (Y_5);
            \draw[thick, ->] (L_3) -- (Y_6);
        \end{tikzpicture}
        \caption{The proposed method}
    \end{subfigure}
    \hfill
    \begin{subfigure}[t]{0.3\textwidth}
        \centering
        \begin{tikzpicture}[scale=0.7]
            \node[draw, rectangle] (L_1) at (-2, 0) {$L_{1}$};
            \node[draw, rectangle] (L_2) at (0, 1) {$L_{2}$};
            \node[draw, rectangle] (L_3) at (0, -1) {$L_{3}$};
            \node[draw, circle] (Y_6) at (-2.5, 2) {$Y_{6}$};
            \node[draw, circle] (X_1) at (-1, 3) {$X_{1}$};
            \node[draw, circle] (X_2) at (1, 3) {$X_{2}$};
            \node[draw, circle] (Y_3) at (-2, -3) {$Y_{3}$};
            \node[draw, circle] (Y_4) at (-0, -3) {$Y_{4}$};
            \node[draw, circle] (Y_5) at (2, -3) {$Y_{5}$};

            \draw[thick, ->] (L_1) -- (Y_6);
            \draw[thick, ->] (L_2) -- (X_1);
            \draw[thick, ->] (L_2) -- (X_2);
            \draw[thick, ->] (L_3) -- (Y_3);
            \draw[thick, ->] (L_3) -- (Y_4);
            \draw[thick, ->] (L_3) -- (Y_5);
            \draw[thick, ->] (L_2) -- (L_3);

        \end{tikzpicture}
        \caption{LiNGLaM}
    \end{subfigure}
    \hfill
    \begin{subfigure}[t]{0.3\textwidth}
        \centering
        \begin{tikzpicture}[scale=0.7]
            \node[draw, rectangle] (L_1) at (0, 1) {$L_{1}$};
            \node[draw, circle] (X_1) at (-1, 3) {$X_{1}$};
            \node[draw, circle] (X_2) at (1, 3) {$X_{2}$};
            \node[draw, circle] (Y_3) at (-2, -2) {$Y_{3}$};
            \node[draw, circle] (Y_4) at (-1, -3) {$Y_{4}$};
            \node[draw, circle] (Y_5) at (1, -3) {$Y_{5}$};
            \node[draw, circle] (Y_6) at (2, -2) {$Y_{6}$};

            \node[draw, rectangle] (L_2) at (0, -1) {$L_{2}$};
            \draw[thick, ->] (L_1) -- (X_1);
            
            \draw[thick, ->] (L_1) -- (X_2);
            \draw[thick, ->] (L_2) -- (Y_3);
            \draw[thick, ->] (L_2) -- (Y_4);
            \draw[thick, ->] (L_2) -- (Y_5);
            \draw[thick, ->] (L_2) -- (Y_6);
        \end{tikzpicture}
        \caption{LiNGLaH}
    \end{subfigure}
    \caption{The application on the political democracy dataset.}
    \label{fig: real-world}
\end{figure}
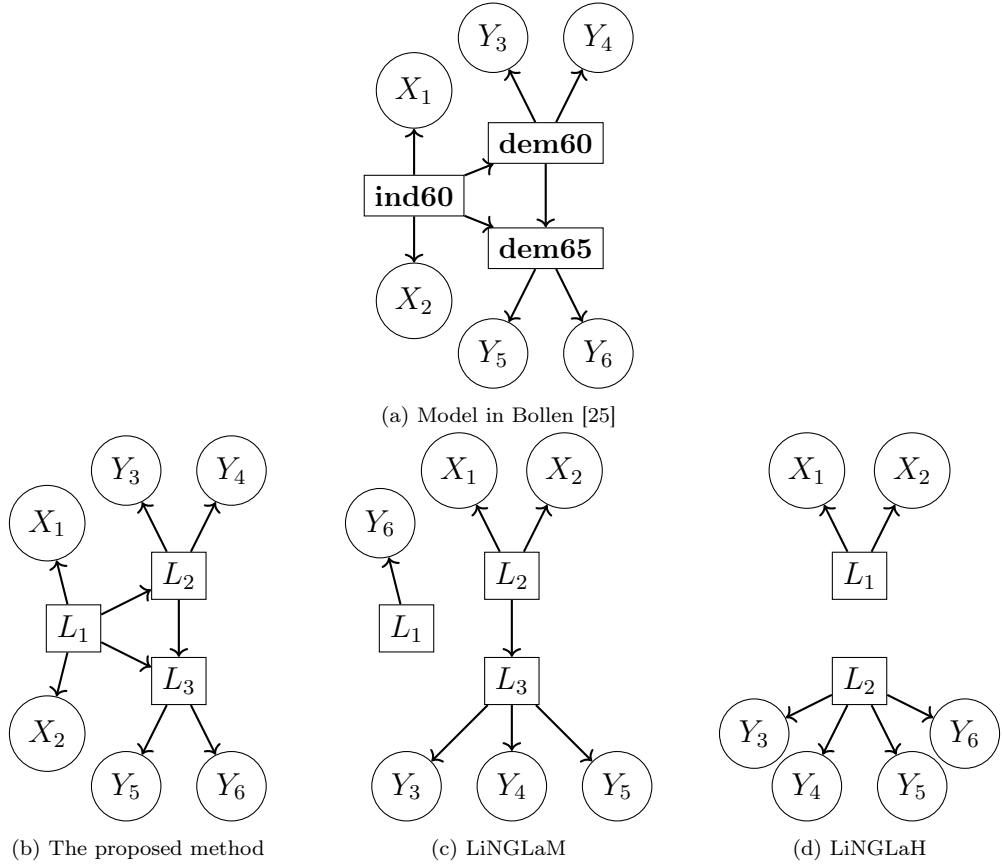

\section{Conclusion}
\label{sec: conclusion}
In this paper, we propose a novel algorithm for estimating LvLiNGAM models in which causal structures exist both among latent variables and among observed variables. Causal discovery for such a class of LvLiNGAM has not been completely addressed in any previous studies. 

Through numerical experiments, we also confirmed the consistency of the proposed method with the theoretical results on its identifiability. 
Furthermore, by applying the proposed method to the Political Democracy dataset~\cite{bollen1989structural}, a standard benchmark in structural equation modeling, we confirmed its practical usefulness. 

However, the class of models to which our proposed method can be applied remains limited. In particular, the assumptions that each observed variable has at most one latent parent and that there are no edges between clusters are restrictive. As mentioned in Section~\ref{sec: LvLiNGAM}, there exist classes of models that can be identified by the proposed method even when some variables have no latent parents. For further details, see Appendix~E. However, even so, the proposed method cannot be applied to many generic canonical models. Developing a more generalized framework that relaxes these constraints remains an important direction for future research.

\section*{Acknowledgement}
This work was supported by JST SPRING under Grant Number JPMJSP2110 and JSPS KAKENHI under Grant Numbers 21K11797 and 25K15017.
\bibliographystyle{elsarticle-num} 
\bibliography{main}
\newpage
\appendix
\renewcommand{\thesection}{\Alph{section}}
\section{Some Theorems and Lemmas for Proving Theorems in the Main Text}
In this section, we present several theorems and lemmas that are required for the proofs of the theorems in the main text.
In the following sections, we assume that the coefficient from each latent variable $L_i$ to its observed child $X_i$ with the highest causal order is normalized to $\lambda_{ii}=1$. 

\begin{theorem}[Darmois-Skitovitch theorem \cite{Darmois1953, Skitovich1953}]
    \label{thm: DS}
    Define two random variables $X_1$ and $X_2$ as linear combinations of independent random variables $u_1,\ldots,u_m$:
    \[
    X_1 = \sum_{i=1}^m \alpha_{1i} u_i, \quad 
    X_2 = \sum_{i=1}^m \alpha_{2i} u_i
    \]
    Then, if $X_1$ and $X_2$ are independent, all variables $u_i$ for which $\alpha_{1i}\alpha_{2i} \ne 0$ are Gaussian.
\end{theorem}
\begin{lemma}
\label{lem: non-Gaussian no-zero}
    Let $X_{1}$ and $X_{2}$ be mutually dependent observed variables in LvLiNGAM in (\ref{model: mixing matrix}) with mutually independent and non-Gaussian disturbances $\bm{u}$. Under Assumptions A4 and A5, $\mathrm{cum}^{(4)}(X_{1}, X_{1}, X_{2}, X_{2})$ and \\ $\mathrm{cum}^{(4)}(X_{1}, X_{1}, X_{1}, X_{2})$ are generically non-zero. 
\end{lemma}
\begin{proof}
    Let $X_{1}$ and $X_{2}$ be linear combinations of $u_1,\ldots,u_{p+q}$: 
    \begin{align*}
        X_1 = \sum_{i=1}^{p+q} \alpha_{1i} u_i, \quad 
        X_2 = \sum_{i=1}^{p+q} \alpha_{2i} u_i. 
    \end{align*}
    When $X_{1} \notindep X_{2}$, there must be $u_{j}$ with $\alpha_{1j}\alpha_{2j} \neq 0$ by Theorem~\ref{thm: DS}.  
    Therefore, generically
    \begin{align*}
        \mathrm{cum}^{(4)}(X_{1}, X_{1}, X_{2}, X_{2}) = \sum_{i=1}^{p+q}\alpha_{1i}^{2}\alpha_{2i}^{2}\mathrm{cum}^{(4)}(u_{i}, u_{i}, u_{i}, u_{i}) 
        \ne 0
    \end{align*}

    A similar proof shows that generically
    \begin{align*}
        \mathrm{cum}^{(4)}(X_{1}, X_{1}, X_{1}, X_{2}) \ne 0. 
    \end{align*}
\end{proof}
\begin{lemma}
    \label{lem: noBCA}
    For $V_i, V_j \in \bm{V}$, let $\alpha_{ji}$ be the total effect from $V_i$ to $V_j$. 
    Assume that $V_{i} \in \mathrm{Anc}(V_{j})$. 
    Then, it holds generically that $V_i$ and $V_j$ are not confounded if and only if $\alpha_{jk} = \alpha_{ji} \cdot \alpha_{ik}$ generically holds for all $V_k \in \mathrm{Anc}(V_{i})$. 
\end{lemma}
\begin{proof}
Please refer to Lemma A.2 in Cai et al. \cite{cai2024learninglinearacycliccausal} for the proof of sufficiency.

We prove the necessity by contrapositive. 
Suppose that $V_i$ and $V_j$ are confounded. We can arbitrarily choose a $V_k$ as their backdoor common ancestor and assume $\alpha_{jk} = \alpha_{ji} \cdot \alpha_{ik}$. From the faithfulness condition, it follows that $\alpha_{ji} \ne 0, \alpha_{ik} \ne 0$, and $\alpha_{jk} \ne 0$.  
Let $V_k \prec V_{k+1} \prec \cdots \prec V_{i-1} \prec V_i \prec V_{i+1} \prec \cdots \prec V_{j-1} \prec V_j$ be one possible causal order consistent with the model. 
Define 
\[
\bm{P}_{jk} :=
\begin{bmatrix}
  -b_{k+1,k}  &1  &\cdots  &0  &0  &0  &\cdots  &0\\
  \vdots  &\vdots  &\ddots  &\vdots  &\vdots  &\vdots  &\cdots  &0\\
  -b_{i-1,k}  &-b_{i-1,k+1}  &\cdots  &1  &0  &0  &\cdots &0\\
  -b_{i,k}  &-b_{i,k+1}  &\cdots  &-b_{i,i-1}  &1  &0  &\cdots  &0\\
  -b_{i+1,k}  &-b_{i+1,k+1}  &\cdots  &-b_{i+1,i-1}  &-b_{i+1,i}  &1 & \cdots  &0\\
  \vdots  &\vdots  &  &\vdots  &\vdots  &\vdots  &\ddots &\vdots\\
  -b_{j-1,k}  &-b_{j-1,k+1}  &\cdots  &-b_{j-1,i-1}  &-b_{j-1,i}  &-b_{j-1,i+1}  &\cdots  &1\\
  -b_{j,k}   &-b_{j,k+1}  &\cdots  &-b_{j,i-1}  &-b_{j,i}  &-b_{j,i+1}  &\cdots &-b_{j,j-1}\\
\end{bmatrix}. 
\]
$\bm{P}_{ji}$ and $\bm{P}_{ik}$ are defined in the same way. 
Then, 
\[
\alpha_{jk} = \left( -1\right)^{k+j} \cdot \vert \bm{P}_{jk} \vert, \quad 
\alpha_{ji} = \left( -1\right)^{i+j} \cdot \vert \bm{P}_{ji} \vert, \quad 
\alpha_{ik} = \left( -1\right)^{k+i} \cdot \vert \bm{P}_{ik} \vert
\]
Therefore, 
$\alpha_{jk}=\alpha_{ji} \cdot \alpha_{ik}$ implies
\begin{equation}
\label{detP}
\vert\bm{P}_{ji}\vert\cdot \vert\bm{P}_{ik}\vert = \vert\bm{P}_{jk} \vert. 
\end{equation}
The left-hand side of (\ref{detP}) equals the determinant of \( \bm{P}_{jk} \), in which the \( (i, i) \)-entry replaced by \( 0 \), 
which implies that $(i,i)$ minor of $\bm{P}_{jk}$ vanishes, that is, 
\[
\begin{vmatrix}
  -b_{k+1,k}  &1  &\cdots  &0  &0  &\cdots  &0\\
  \vdots  &\vdots  &\ddots  &\vdots  &\vdots  &\cdots  &0\\
    -b_{i-1,k}  &-b_{i-1,k+1}  &\cdots  &1  &0  &\cdots &0\\
  -b_{i+1,k}  &-b_{i+1,k+1}  &\cdots  &-b_{i+1,i-1}  &1  &\cdots  &0\\
  \vdots  &\vdots  &  &\vdots  &\vdots  &\ddots &\vdots\\
  -b_{j-1,k}  &-b_{j-1,k+1}  &\cdots  &-b_{j-1,i-1}  &-b_{j-1,i+1}  &\cdots  &1\\
    -b_{j,k}   &-b_{j,k+1}  &\cdots  &-b_{j,i-1}  &-b_{j,i+1}  &\cdots &-b_{j,j-1}\\
\end{vmatrix}=0.
\]
The space of $b_{rc}, r=k+1,\ldots,j, c=k,\ldots,j-1$ satisfying the above equation is a real algebraic set and constitutes a measure-zero subset of the parameter space. 
Hence, generically $\vert\bm{P}_{ji}\vert\cdot \vert\bm{P}_{ik}\vert \ne \vert\bm{P}_{jk} \vert$. 
\end{proof}
\begin{lemma}
\label{lem: X independent L independent}
    Let $L_{i}$ and $L_{j}$ be the latent parents of $X_i$ and $X_j$, respectively. 
    Under Assumptions \ref{A1} and \ref{A4}, 
    \begin{align*}
        X_{i} \indep X_{j} \Leftrightarrow L_{i} \indep L_{j}.
    \end{align*}
\end{lemma}
\begin{proof}
The proof is trivial.
\end{proof}
\begin{lemma}
\label{lem: X has multiple parent}   
    Let $X_i$ and $X_j$ be the observed variables with the highest causal order within the clusters formed by the observed children of their latent parents, $L_i$ and $L_j$, respectively. 
    Under Assumptions \ref{A1} and \ref{A3}, if $\mathrm{Conf}(L_{i}, L_{j}) = \emptyset$, 
    \begin{enumerate}
        \setlength{\itemsep}{0pt}
        \setlength{\parsep}{0pt}
        \setlength{\parskip}{0pt}
        \item $L_i \indep L_j \Rightarrow \mathrm{Conf}(X_{i}, X_{j})= \emptyset$. 
        \item $L_i \in \mathrm{Anc}(L_j) \Rightarrow \mathrm{Conf}(X_{i}, X_{j})= \{L_i\}$. 
    \end{enumerate}
\end{lemma}
\begin{proof}
    When $L_{i} \indep L_{j}$, $X_{i} \indep X_{j}$ according to Lemma \ref{lem: X independent L independent}. 
    Therefore, no latent confounder exists between $X_{i}$ and $X_{j}$. 

    Suppose $L_i \in \mathrm{Anc}(L_j)$. Under Assumptions \ref{A1} and \ref{A3}, $\mathrm{Conf}(X_{i}, X_{j}) \subset  \mathrm{Anc}(X_{i}) \cap \mathrm{Anc}(X_{j}) = \mathrm{Anc}(L_{i}) \cup \{L_{i}\}$. Every path originating from the variables in $\mathrm{Anc}(L_{i})$ to $X_{i}$ and $X_{j}$ passes through $L_{i}$. Therefore, the only possible latent confounder of $X_{i}$ and $X_{j}$ is $L_{i}$.
\end{proof}
\begin{lemma}
\label{lem: L does not have multiple parent}   
Let $X_i$ and $X_j$ be the observed variables with the highest causal order within the clusters formed by the observed children of their latent parents, $L_i$ and $L_j$, respectively. Under Assumption \ref{A1}, if $X_{i}$ and $X_{j}$ have only one latent confounder, $L_{i}$ and $L_{j}$ do not have multiple confounders.
\end{lemma}
\begin{proof}
    We will prove this lemma by contrapositive. 
    
    Assume that $\lvert \mathrm{Conf}(L_{i}, L_{j}) \rvert \ge 2$. There exist two distinct nodes $L_k, L_{k'} \in \mathrm{Conf}(L_i, L_j)$ such that there are directed paths from $L_k$ and $L_{k'}$ to $L_i$ and $L_j$, respectively, and the two paths share no node other than their starting points. Therefore, by Assumption \ref{A1}, two directed paths also exist from $L_{k}$ to $X_{i}$ and $X_{j}$, sharing no node other than $L_{k}$.  
    Hence, $\mathrm{Conf}(L_{i}, L_{j}) \subset \mathrm{Conf}(X_{i}, X_{j})$, which implies that $\lvert \mathrm{Conf}(X_{i}, X_{j}) \rvert \ge 2$. 
\end{proof}
\begin{theorem}
\label{thm: one latent coef estimate}
    $V_{i}, V_{j} \in \bm{V}$ are two confounded observed variables.
    Assume that all sixth cross-cumulants of $\bm{u}$ are non-zero.
    Then 
    \begin{align*}
        (c^{(6)}_{i, i, i, j, j, j})^{2} = c^{(6)}_{i, i, i, i, j, j}c^{(6)}_{i, i, j, j, j, j}
    \end{align*}
    if and only if the following two conditions hold simultaneously: 
    \begin{enumerate}
    \setlength{\itemsep}{0pt}
    \setlength{\parsep}{0pt}
    \setlength{\parskip}{0pt}
        \item There exists no direct path between $V_i$ and $V_j$.
        \item $V_i$ and $V_j$ share only one (latent or observed) confounder in the canonical model over $V_i$ and $V_j$. 
    \end{enumerate}
\end{theorem}
\begin{proof}
    Without loss of generality, assume that $V_{j} \notin \mathrm{Anc}(V_{i})$.
    Define $\mathcal{I}_A$, $\mathcal{I}_B$, and $\mathcal{I}_C$ by 
    \begin{align*}
        \mathcal{I}_{A} =& \{k \mid V_{k} \in \mathrm{Anc}(V_{i}) \cap \mathrm{Anc}(V_{j})\}, \\
        \mathcal{I}_{B} =& \{k \mid V_{k} \in \mathrm{Anc}(V_{i}) \setminus \mathrm{Anc}(V_{j})\}, \\
        \mathcal{I}_{C} =& \{k \mid V_{k} \in \mathrm{Anc}(V_{j}) \setminus (\mathrm{Anc}(V_{i})  \cup \{V_{i}\})\}.
    \end{align*}
    Then, $V_i$ and $V_j$ are expressed as
    \begin{align*}
        V_{i} =& 
        \sum_{{k} \in \mathcal{I}_{A}}^{} \alpha_{ik}u_{k} + \sum_{{k} \in \mathcal{I}_{B}}^{} \alpha_{ik}u_{k} + u_{i}, \\
        V_{j} 
        =& \sum_{{k} \in \mathcal{I}_{A}}^{}(\alpha_{jk}+ \alpha_{ji}\alpha_{ik})u_{k} + \sum_{{k} \in \mathcal{I}_{B}}^{} \alpha_{jk}u_{k} + \sum_{{k} \in \mathcal{I}_{C}}^{} \alpha_{jk}u_{k} + \alpha_{ji}u_{i} + u_{j}.
    \end{align*}
    Since $V_k \notin \mathrm{Conf}(V_i,V_j)$ for all $k \in \mathcal{I}_B$,  
    we have
    \[
    \sum_{{k} \in \mathcal{I}_{B}}^{} \alpha_{jk}u_{k} = \alpha_{ji}\sum_{{k} \in \mathcal{I}_{B}}^{} \alpha_{ik}u_{k},
    \]
    by Lemma \ref{lem: noBCA}. 
    Let $\tilde{u}_i$ and $\tilde{u}_j$ be 
    \[
    \tilde{u}_i = \sum_{{k} \in \mathcal{I}_{B}} \alpha_{ik} u_{k} + u_{i}, \quad 
    \tilde{u}_j = \sum_{{k} \in \mathcal{I}_{C}} \alpha_{jk} u_{k} + u_{j}.
    \]
    Then, 
    \begin{equation}
    \label{eq: appendix: one confounder}
    \begin{aligned}
        V_{i} = \sum_{{k} \in \mathcal{I}_{A}}^{} \alpha_{ik}u_{k} + \tilde{u}_i, \quad 
        V_{j} = \sum_{{k} \in \mathcal{I}_{A}}^{}(\alpha_{jk}+ \alpha_{ji}\alpha_{ik})u_{k} + \alpha_{ji} \tilde{u}_i + \tilde{u}_j. 
    \end{aligned}
    \end{equation}

    \noindent\textbf{Necessity:}
    We assume that $\alpha_{ji} = 0$ and $\vert \mathrm{Conf}(V_{i}, V_{j}) \vert=1$.
    Denote the disturbance of the unique confounder of $V_{i}$ and $V_{j}$ by $u_{c}$.
    Then $V_i$ and $V_j$ are expressed as 
    \begin{align*}
        V_{i} = \tilde{u}_{i} + \alpha_{ic}u_{c}, \quad
        V_{j} = \tilde{u}_{j} + \alpha_{jc}u_{c}.
    \end{align*}
    The sixth cross-cumulants of $V_{i}$ and $V_{j}$ is obtained by direct computation as follows: 
    \begin{align*}
        c^{(6)}_{i, i, i, j, j, j} = \alpha^{3}_{ic}\alpha^{3}_{jc}\mathrm{cum}^{(6)}(u_{c}), \\
        c^{(6)}_{i, i, j, j, j, j} = \alpha^{2}_{ic}\alpha^{4}_{jc}\mathrm{cum}^{(6)}(u_{c}), \\
        c^{(6)}_{i, i, i, i, j, j} = \alpha^{4}_{ic}\alpha^{2}_{jc}\mathrm{cum}^{(6)}(u_{c}).
    \end{align*}
    Therefore we have 
    \begin{align*}
        {\left(c^{(6)}_{i, i, i, j, j, j}\right)}^{2} =c^{(6)}_{i, i, j, j, j, j}c^{(6)}_{i, i, i, i, j, j}.
    \end{align*}
    
    \noindent\textbf{Sufficiency:} 
    According to Hoyer et al. \cite{hoyer2008estimation}, $u_k$, $k \in \mathcal{I}_A$ can be merged as one confounder, that is, $\vert \mathrm{Conf}(V_{i}, V_{j})\vert = 1$ in the canonical model over $V_i$ and $V_j$. 
    
    From (\ref{eq: appendix: one confounder}), we have
    \begin{align*}
        c^{(6)}_{i, i, i, j, j, j} =& \sum_{{k} \in \mathcal{I}_{A}}^{}\alpha_{ik}^{3}(\alpha_{jk} + \alpha_{ji}\alpha_{ik})^{3} \mathrm{cum}^{(6)}(u_{k}) + (\alpha_{ji})^{3} \mathrm{cum}^{(6)}(\tilde{u}_{i}), \\
        c^{(6)}_{i, i, i, i, j, j} =& \sum_{{k} \in \mathcal{I}_{A}}^{}\alpha_{ik}^{4}(\alpha_{jk} + \alpha_{ji}\alpha_{ik})^{2} \mathrm{cum}^{(6)}(u_{k}) + (\alpha_{ji})^{2} \mathrm{cum}^{(6)}(\tilde{u}_{i}), \\
        c^{(6)}_{i, i, j, j, j, j} =& \sum_{{k} \in \mathcal{I}_{A}}^{}\alpha_{ik}^{2}(\alpha_{jk} + \alpha_{ji}\alpha_{ik})^{4} \mathrm{cum}^{(6)}(u_{k}) + (\alpha_{ji})^{4} \mathrm{cum}^{(6)}(\tilde{u}_{i}).
    \end{align*}
    For notational simplicity, we denote the first terms in the right-hand side of the three equations by $A_{33}$, $A_{42}$, and $A_{24}$, respectively.  
    When 
    \begin{align*}
        (c^{(6)}_{i, i, i, j, j, j})^{2} =c^{(6)}_{i, i, j, j, j, j}c^{(6)}_{i, i, i, i, j, j}, 
    \end{align*}
    we have
    \begin{align*}
        &A^{2}_{33} + 2A_{33}(\alpha_{ji})^{3}\mathrm{cum}^{(6)}(\tilde{u}_{i}) + (\alpha_{ji})^{6}\mathrm{cum}^{(6)}(\tilde{u}_{i})^{2} \\
        &\quad = A_{42}A_{24} + (\alpha_{ji})^{2}A_{24}\mathrm{cum}^{(6)}(\tilde{u}_{i}) + A_{42}(\alpha_{ji})^{4}\mathrm{cum}^{(6)}(\tilde{u}_{i}) + (\alpha_{ji})^{6}\mathrm{cum}^{(6)}(\tilde{u}_{i})^{2}, 
    \end{align*}
    which is equivalent to 
    \begin{align*}
        (2(\alpha_{ji})^{3}A_{33} - (\alpha_{ji})^{2}A_{24} - (\alpha_{ji})^{4}A_{42})\mathrm{cum}^{(6)}(\tilde{u}_{i})+ (A^{2}_{33} - A_{42}A_{24}) = 0.
    \end{align*}
    This implies
    \begin{equation}
    \label{eq: two terms}
        \begin{split}
        2(\alpha_{ji})^{2}A_{33} - (\alpha_{ji})^{}A_{24} - (\alpha_{ji})^{3}A_{42} = 0, 
        \quad 
        A^{2}_{33} - A_{42}A_{24} = 0.
        \end{split}
    \end{equation}
    We note that 
    \begin{align*}
        &A^{2}_{33} = A_{42}A_{24} \Leftrightarrow \\
        &\left( \sum_{{k} \in \mathcal{I}_{A}}^{}\alpha_{ik}^{3}(\alpha_{jk} + \alpha_{ji}\alpha_{ik})^{3} \mathrm{cum}^{(6)}(u_{k}) \right)^{2} \\ 
        &\quad =
        \left( \sum_{{k} \in \mathcal{I}_{A}}^{}\alpha_{ik}^{4}(\alpha_{jk} + \alpha_{ji}\alpha_{ik})^{2} \mathrm{cum}^{(6)}(u_{k}) \right)\left(\sum_{{k} \in \mathcal{I}_{A}}^{}\alpha_{ik}^{2}(\alpha_{jk} + \alpha_{ji}\alpha_{ik})^{4} \mathrm{cum}^{(6)}(u_{k})\right).
    \end{align*}   
    By Lagrange's identity
    \[
    \forall k \in \mathcal{I}_{A}, \quad 
    \frac{(\alpha_{jk} + \alpha_{ji}\alpha_{ik})}{\alpha_{ik}} = c \; \Leftrightarrow \;
    \frac{\alpha_{jk}}{\alpha_{ik}} = c - \alpha_{ji}
    \] 
    for a constant $c$.

    For the first equation in (\ref{eq: two terms}), we have 
    \begin{align*}
        &2(\alpha_{ji})^{2}A_{33} - \alpha_{ji}A_{24} - (\alpha_{ji})^{3}A_{42}\\
        &\quad = \alpha_{ji}\left[\left(\alpha_{ji} - \frac{A_{33}}{A_{42}}\right)^{2}
         - \frac{A^{2}_{33}}{A^{2}_{42}} + \frac{A_{24}}{A_{42}}\right] = 0. 
    \end{align*}
    
    Since $A^{2}_{33} = A_{42}A_{24}$ and $(\alpha_{jk} + \alpha_{ji}\alpha_{ik}) = c \cdot \alpha_{ik}$, 
    \begin{align*}
        &\alpha_{ji}\left[
        \left(
        \alpha_{ji} - \frac{A_{33}}{A_{42}}\right)^{2}
         - \frac{A^{2}_{33}}{A^{2}_{42}} + \frac{A_{24}}{A_{42}}
         \right]\\
         &\quad = 
         \alpha_{ji}\left[\alpha_{ji} - \frac{\sum_{{k} \in \mathcal{I}_{A}}^{}\alpha_{ik}^{3}(\alpha_{jk} + \alpha_{ji}\alpha_{ik})^{3} \mathrm{cum}^{(6)}(u_{k})}{\sum_{{k} \in \mathcal{I}_{A}}^{}\alpha_{ik}^{4}(\alpha_{jk} + \alpha_{ji}\alpha_{ik})^{2} \mathrm{cum}^{(6)}(u_{k})}\right]^{2}\\
         &\quad =
         \alpha_{ji}\left[\alpha_{ji} - \frac{c^{3}\sum_{{k} \in \mathcal{I}_{A}}^{}\alpha_{ik}^{6} \mathrm{cum}^{(6)}(u_{k})}{c^2\sum_{{k} \in \mathcal{I}_{A}}^{}\alpha_{ik}^{6} \mathrm{cum}^{(6)}(u_{k})}\right]^{2}  =
         \alpha_{ji}(\alpha_{ji} - c)^{2} = 0.
    \end{align*}
    Thus, $\alpha_{ji} = 0 \text{ or } c$. 
    $\alpha_{ji}=c$ implies that $\alpha_{jk}=0$, which contradicts the faithfulness assumption. 
    Therefore, we conclude that $\alpha_{ji} = 0$, which implies that there is no directed path from $V_{i}$ to $V_{j}$.
\end{proof}
\begin{lemma}
    \label{lem: X has one parent}   
    Assume that $X_{i}$ and $X_{j}$ belong to distinct clusters, that they are the children with the highest causal order of $L_{i}$ and $L_{j}$, respectively, and that $L_{j} \notin \mathrm{Anc}(L_{i})$. Under Assumptions \ref{A1} and \ref{A3}, if $X_{i}$ and $X_{j}$ have only one latent confounder $L_{c}$ in the canonical model over them, one of the following conditions generically holds:
    \begin{enumerate}
    \setlength{\itemsep}{0pt}
    \setlength{\parsep}{0pt}
    \setlength{\parskip}{0pt}
        \item $\mathrm{Conf}(L_{i}, L_{j}) = \emptyset$.
            Then, $L_{i}$ and $L_{c}$ are identical, 
            and
            \begin{align*}
                c^{(k)}_{(X_i \to X_j)}(L_{c}) &= c^{(k)}_{(X_i \to X_j)}(L_{i}) = \mathrm{cum}^{(k)}(L_{i}),\\ 
            c^{(k)}_{(X_j \to X_i)}(L_{c}) &= c^{(k)}_{(X_j \to X_i)}(L_{i}) = \mathrm{cum}^{(k)}(\alpha^{ll}_{ji}\cdot L_{i}).
            \end{align*}
        \item 
        $\mathrm{Conf}(L_{i}, L_{j}) = \{L_c\}$. 
            Then, 
            \begin{align*}
                c^{(k)}_{(X_i \to X_j)}(L_{c}) = \mathrm{cum}^{(k)}(\alpha^{ll}_{ic} \cdot L_{c}), \quad  c^{(k)}_{(X_j \to X_i)}(L_{c}) = \mathrm{cum}^{(k)}(\alpha^{ll}_{jc} \cdot L_{c}).
            \end{align*}
    \end{enumerate}
\end{lemma}
\begin{proof}
    According to Lemma \ref{lem: L does not have multiple parent}, $\vert \mathrm{Conf}(L_{i}, L_{j})\vert = 0 \text{ or } 1$.
    Since $X_{i}$ and $X_{j}$ are confounded by $L_{c}$, $X_{i} \notindep X_{j}$, which implies $L_{i} \notindep L_{j}$ by Lemma \ref{lem: X independent L independent}. 
    
    First, consider the case where $\mathrm{Conf}(L_{i}, L_{j}) = \emptyset$. According to Lemma \ref{lem: X has multiple parent}, when $L_{i} \notindep L_{j}$, the only possible latent confounder of $X_{i}$ and $X_{j}$ is $L_{i}$.
    Furthermore, there is at least one causal path from $L_{i}$ to $L_{j}$.
    
    Define $\mathcal{I}_A$ and $\mathcal{I}_B$ by 
    \begin{align*}
        \mathcal{I}_{A} = \{k \mid L_{k} \in \mathrm{Anc}(L_{i})\}, \quad
        \mathcal{I}_{B} = \{k \mid L_{k} \in \mathrm{Anc}(L_{j}) \setminus (\mathrm{Anc}(L_{i}) \cup \{L_{i}\})\}.
    \end{align*}
    Then, $X_i$ and $X_j$ are written as 
    \begin{align}
        \label{eq: Xi}
        X_{i} &= \left(
        \sum_{k \in \mathcal{I}_{A}}\alpha_{ik}^{ll}\epsilon_{k} + \epsilon_{i}
        \right) + e_{i}, \\
        X_{j} &= \left(
        \sum_{k \in \mathcal{I}_{A}}\alpha_{jk}^{ll}\epsilon_{k} + \sum_{k \in \mathcal{I}_{B}}\alpha_{jk}^{ll}\epsilon_{k}
        + \alpha_{ji}^{ll}\epsilon_{i} 
        + \epsilon_{j}
        \right) + e_{j}. \notag
    \end{align}
    From Lemma \ref{lem: noBCA}, we have 
    \[
    \sum_{k \in \mathcal{I}_{A}}\alpha_{jk}^{ll}\epsilon_{k} = \alpha_{ji}^{ll} \cdot \sum_{k \in \mathcal{I}_{A}}\alpha_{ik}^{ll}\epsilon_{k}.
    \]
    Letting $v_{j} = \sum_{k \in \mathcal{I}_{B}}\alpha_{jk}^{ll}\epsilon_{k} + \epsilon_{j} + e_{j}$, $X_{j}$ is rewritten as 
    \begin{align}
        \label{eq: Xj}
        X_{j} = \alpha_{ji}^{ll}(\sum_{k \in \mathcal{I}_{A}}\alpha_{ik}^{ll}\epsilon_{k} + \epsilon_{i}) + v_{j}. 
    \end{align}
    Note that $L_{i}$, $e_{i}$, and $v_{j}$ are mutually independent. 
    From Proposition \ref{thm: estimate b} with $\ell=1$, the roots of the polynomial on $\alpha$
    \begin{align*}
        \left\vert
        \begin{array}{ccc}
           1  & {\alpha} & {\alpha}^{2} \\
           c^{(3)}_{i, i, i}  & c^{(3)}_{i, i, j} & c^{(3)}_{i, j, j} \\
           c^{(4)}_{i, i, i, i} & c^{(4)}_{i, i, i, j} & c^{(4)}_{i, i, j, j}
        \end{array}\right\vert  = 0
    \end{align*}
    are $\alpha^{oo}_{ji}$ and $\alpha^{ol}_{ji}$. From (\ref{eq: Xi}) and (\ref{eq: Xj}), we have 
    \begin{align*}
        &\left\vert
        \begin{array}{ccc}
           1  & {\alpha} & {\alpha}^{2} \\
           c^{(3)}_{i, i, i}  & c^{(3)}_{i, i, j} & c^{(3)}_{i, j, j} \\
           c^{(4)}_{i, i, i, i} & c^{(4)}_{i, i, i, j} & c^{(4)}_{i, i, j, j}
        \end{array}\right\vert\notag\\ 
        &\quad = \left(
        (\alpha_{ji}^{ll})^3 \mathrm{cum}^{(3)}(L_i) \mathrm{cum}^{(4)}(L_i)
        - (\alpha_{ji}^{ll})^3 \mathrm{cum}^{(3)}(L_i) \mathrm{cum}^{(4)}(L_i)
        \right)\notag\\
        &\qquad 
        - \alpha \left(
        (\alpha_{ji}^{ll})^2 \cdot 
        (\mathrm{cum}^{(3)}(e_i) \mathrm{cum}^{(4)}(L_i) - \mathrm{cum}^{(3)}(L_i) \mathrm{cum}^{(4)}(e_i))
        \right)\notag\\
        &\qquad
        + \alpha^2 \left(
        (\alpha_{ji}^{ll}) \cdot 
        (\mathrm{cum}^{(3)}(e_i) \mathrm{cum}^{(4)}(L_i) - \mathrm{cum}^{(3)}(L_i) \mathrm{cum}^{(4)}(e_i))
        \right)=0, \notag
    \end{align*}
    which is generically equivalent to 
    \begin{equation}
        \label{eq: appendix: polynomial 1}
        -\alpha\cdot (\alpha_{ji}^{ll})^{2} + \alpha^{2} \cdot \alpha_{ji}^{ll}=0. 
    \end{equation}
    The roots of (\ref{eq: appendix: polynomial 1}) are $\alpha = 0, \alpha_{ji}^{ll}$.
    Since $X_{i}$ and $X_{j}$ belong to different clusters, $\alpha^{oo}_{ji}=0$ and hence $\alpha_{ji}^{ol} = \lambda_{jj}\alpha_{ji}^{ll} = \alpha_{ji}^{ll}$.
        
    From Proposition \ref{lem: estimate e cumulant}, 
    \begin{align*}
        \left[
        \begin{array}{cc}
            1 & 1 \\
            0 & \alpha_{ji}^{ll} 
        \end{array}
        \right]
        \left[
        \begin{array}{c}
            c^{(k)}_{(X_i \to X_j)}(e_{i}) \\
            c^{(k)}_{(X_i \to X_j)}(L_{c})  \\
        \end{array}
        \right]
        = 
        \left[
        \begin{array}{c}
            c^{(k)}_{i, \dots, i, i} \\
            c^{(k)}_{i, \dots, i, j}  
        \end{array}
        \right]
        = 
        \left[
        \begin{array}{c}
            \mathrm{cum}^{(k)}(e_{i}) + \mathrm{cum}^{(k)}(L_{i}) \\
            \mathrm{cum}^{(k)}(\alpha_{ji}^{ll} \cdot L_{i})  
        \end{array}
        \right]. 
    \end{align*}
    Then, we have $c_{(X_{i} \to X_{j})}^{(k)}(e_{i}) = \mathrm{cum}^{(k)}(e_{i})$ and $c_{(X_i \to X_j)}^{(k)}(L_{c}) = \mathrm{cum}^{(k)}(L_{i})$.
    In the same way, we can obtain $c_{(X_j \to X_i)}^{(k)}(L_{c}) = \mathrm{cum}^{(k)}(\alpha_{ji}^{ll} \cdot L_{i})$. 
    
    Next, we consider the case where $\mathrm{Conf}(L_{i}, L_{j}) = \{L_c\}$. Then, only $L_c$ has outgoing directed paths to $X_i$ and $X_j$ that share no latent variable other than $L_c$. 
    
    Define $\mathcal{I}_A$ and $\mathcal{I}_B$ by
    \begin{align*}
        \mathcal{I}_{A} &= \{{k} \mid L_{k} \in \mathrm{Anc}(L_{i}) \setminus (\mathrm{Anc}(L_{c})\cup \{L_{c}\})\}, \\
        \mathcal{I}_{B} &= \{{k} \mid L_{k} \in \mathrm{Anc}(L_{j}) \setminus (\mathrm{Anc}(L_{c})\cup \mathrm{Anc}(L_{i})\cup \{L_{c}, L_{i}\})\}. 
    \end{align*}
    Following Salehkaleybar et al.~\cite{salehkaleybar2020learning}, $X_i$ and $X_j$ are expressed as
    \begin{align*}
        X_{i} &= (\sum_{k \in \mathcal{I}_{A}}\alpha_{ik}^{ll}\epsilon_{k} + \alpha^{ll}_{ic}L_{c} + \epsilon_{i}) +e_{i}, \\
        X_{j} &= (\sum_{k \in \mathcal{I}_{A}}\alpha_{jk}^{ll}\epsilon_{k} + \sum_{k \in \mathcal{I}_{B}}\alpha_{jk}^{ll}\epsilon_{k} + \alpha^{ll}_{jc}L_{c} + \alpha_{ji}^{ll}\epsilon_{i} + \epsilon_{j}) + e_{j}.
    \end{align*}
    Since 
    \[
    \sum_{k \in \mathcal{I}_{A}}\alpha_{jk}^{ll}\epsilon_{k} = \alpha_{ji}^{ll} \cdot \sum_{k \in \mathcal{I}_{A}}\alpha_{ik}^{ll}\epsilon_{k}, 
    \]
    by Lemma \ref{lem: noBCA}, $X_{j}$ is rewritten as 
    \begin{align*}
        X_{j} 
        &= \alpha^{ll}_{ji}(\sum_{k \in \mathcal{I}_{A}}\alpha_{ik}^{ll}\epsilon_{k} +  \epsilon_{i}) + \sum_{k \in \mathcal{I}_{B}}\alpha_{jk}^{ll}\epsilon_{k} + \alpha^{ll}_{jc}L_{c} + \epsilon_{j} + e_{j}.
    \end{align*}
    The cumulants $c^{(6)}_{i, i, i, j, j, j}$, $c^{(6)}_{i, i, i, i, j, j}$, and $c^{(6)}_{i, i, j, j, j, j}$ are written as follows: 
    \begin{align*}
        c^{(6)}_{i, i, i, j, j, j} 
        &= (\alpha^{ll}_{ji})^{3}\sum^{}_{k \in \mathcal{I}_{A}} (\alpha_{ik}^{ll})^{6}\mathrm{cum}^{(6)}(\epsilon_{k}) + (\alpha^{ll}_{ji})^{3}\mathrm{cum}^{(6)}(\epsilon_{i}) \\ &\quad+ (\alpha_{ic}^{ll})^{3}(\alpha_{jc}^{ll})^{3}\mathrm{cum}^{(6)}(L_{c}), \\
        c^{(6)}_{i, i, i, i, j, j} 
        &= (\alpha^{ll}_{ji})^{2}\sum^{}_{k \in \mathcal{I}_{A}} (\alpha_{ik}^{ll})^{6}\mathrm{cum}^{(6)}(\epsilon_{k}) + (\alpha^{ll}_{ji})^{2}\mathrm{cum}^{(6)}(\epsilon_{i}) \\ &\quad+ (\alpha_{ic}^{ll})^{4}(\alpha_{jc}^{ll})^{2}\mathrm{cum}^{(6)}(L_{c}), \\
        c^{(6)}_{i, i, j, j, j, j} 
        &= (\alpha^{ll}_{ji})^{4}\sum^{}_{k \in \mathcal{I}_{A}} (\alpha_{ik}^{ll})^{6}\mathrm{cum}^{(6)}(\epsilon_{k}) + (\alpha_{ji}^{ll})^{4}\mathrm{cum}^{(6)}(\epsilon_{i}) \\ &\quad+ (\alpha_{ic}^{ll})^{2}(\alpha_{jc}^{ll})^{4}\mathrm{cum}^{(6)}(L_{c}). 
    \end{align*}
    Since $X_{i}$ and $X_{j}$ have only one confounder, $(c^{(6)}_{i, i, i, j, j, j})^{2} = c^{(6)}_{i, i, i, i, j, j}c^{(6)}_{i, i, j, j, j, j}$ holds from Theorem \ref{thm: one latent coef estimate}, which implies
    \begin{align*}
        & 2(\alpha^{ll}_{ji})^{3}(\alpha_{ic}^{ll})^{3}(\alpha_{jc}^{ll})^{3}=(\alpha_{ji}^{ll})^{2}(\alpha_{ic}^{ll})^{2}(\alpha_{jc}^{ll})^{4} + (\alpha^{ll}_{ji})^{4}(\alpha_{ic}^{ll})^{4}(\alpha_{jc}^{ll})^{2}\\
        &\quad\Leftrightarrow 
        (\alpha^{ll}_{ji})^{2}
        \left(
        \alpha^{ll}_{ji} - \frac{\alpha_{jc}^{ll}}{\alpha_{ic}^{ll}}
        \right)^{2}=0. 
    \end{align*}
    When $\alpha^{ll}_{ji} = \alpha_{jc}^{ll}/\alpha_{ic}^{ll}$, all directed paths from $L_{c}$ to $L_{j}$ pass through $L_{i}$ by Lemma \ref{lem: noBCA}, and then $L_{c}$ is not a confounder between $L_{i}$ and $L_{j}$, which leads to a contradiction. 
    Therefore, $\alpha^{ll}_{ji}=0$. 
    Letting $v_{i} = \sum_{k \in \mathcal{I}_{A}}\alpha_{ik}^{ll}\epsilon_{k} +  \epsilon_{i} + e_{i}$ and $v_{j} = \sum_{k \in \mathcal{I}_{B}}\alpha_{jk}^{ll}\epsilon_{k} + \epsilon_{j} + e_{j}$, 
    $X_i$ and $X_j$ are rewritten as
    \begin{align*}
        X_{i} = \alpha^{ll}_{ic}L_{c} + v_{i}, \quad X_{j} = \alpha^{ll}_{jc}L_{c} + v_{j}.
    \end{align*}
    Therefore, we find that the unique confounder of $X_{i}$ and $X_{j}$ is $L_{c}$. 
    
    In the same way as (\ref{eq: appendix: polynomial 1}), we have generically
    \begin{equation}
    \label{eq: appendix: polynomial 2}
        \begin{aligned}
        \left\vert
        \begin{array}{ccc}
           1  & {\alpha} & {\alpha}^{2} \\
           c^{(3)}_{i, i, i}  & c^{(3)}_{i, i, j} & c^{(3)}_{i, j, j} \\
           c^{(4)}_{i, i, i, i} & c^{(4)}_{i, i, i, j} & c^{(4)}_{i, i, j, j}
        \end{array}\right\vert 
        = 0 \;\Leftrightarrow \;
        -{\alpha}\cdot \alpha^{ll}_{jc} + {\alpha}^{2} \cdot \alpha^{ll}_{ic}=0
        \end{aligned}
    \end{equation}
    Then, ${\alpha} = 0, \alpha^{ll}_{jc}/\alpha^{ll}_{ic}$. 
    Due to the assumption \ref{A3}, 
    $\alpha^{oo}_{ji}=0$. Therefore, $\alpha^{ol}_{jc}=\alpha^{ll}_{jc}/\alpha^{ll}_{ic}$ from Proposition 
    \ref{thm: estimate b}.
    According to Proposition \ref{lem: estimate e cumulant} 
    \begin{align*}
        \left[
        \begin{array}{cc}
            1 & 1 \\
            0 & \frac{\alpha^{ll}_{jc}}{\alpha^{ll}_{ic}}
        \end{array}
        \right]
        \left[
        \begin{array}{c}
            c_{(X_i \to X_j)}^{(k)}(e_{i}) \\
            c_{(X_i \to X_j)}^{(k)}(L_{c})  \\
        \end{array}
        \right]
        &=
        \left[
        \begin{array}{c}
            c^{(k)}_{i, \dots, i, i} \\
            c^{(k)}_{i, \dots, i, j}  
        \end{array}
        \right]\\
        &=
        \left[
        \begin{array}{c}
            \mathrm{cum}^{(k)}(v_{i}) + \mathrm{cum}^{(k)}(\alpha_{ic}^{ll} L_{c}) \\
            \frac{\alpha_{jc}^{ll}}{\alpha_{ic}^{ll}} \cdot \mathrm{cum}^{(k)}(\alpha_{ic}^{ll}L_{c})  
        \end{array}
        \right]
    \end{align*}
    Solving this equation yields $c_{(X_i \to X_j)}^{(k)}(e_{i}) = \mathrm{cum}^{(k)}(v_{i})$ and $c^{(k)}_{(X_i \to X_j)}(L_{c}) = \mathrm{cum}^{(k)}(\alpha^{ll}_{ic}\cdot L_{c})$.
    In the same way, we can obtain $c_{(X_j \to X_i)}^{(k)}(L_{c}) = \mathrm{cum}^{(k)}(\alpha_{jc}^{ll} L_{c})$. 
\end{proof}
\begin{lemma}
\label{lem: new statistic}
    Let $X_{i}$ and $X_{j}$ be two dependent observed variables. 
    Assume that $\mathrm{Conf}(X_i,X_j)=\{L_{c}\}$
    and that there is no directed path between $X_{i}$ and $X_{j}$ in the canonical model over them.
    Then, \(\tilde{e}_{(X_i, X_j)} \indep L_{c}\). 
\end{lemma}
\begin{proof}
    Let $v_{i}$ and $v_{j}$ be the disturbances of $X_{i}$ and $X_{j}$, respectively, in the canonical model over $X_{i}$ and $X_{j}$. 
    $X_{i}$ and $X_{j}$ are expressed as 
    \begin{align*}
        X_{i} = \alpha_{ic}^{ol}L_{c} + v_{i}, \quad
        X_{j} = \alpha_{jc}^{ol}L_{c} + v_{j}.
    \end{align*}
    
    Then, $\tilde{e}_{(X_{i}, X_{j})}$ is given by 
    \begin{align*}
        \tilde{e}_{\left(X_{i}, X_{j}\right)} = \alpha_{ic}^{ol}L_{c} + v_{i} - \frac{(\alpha^{ol}_{ic})^{2}(\alpha^{ol}_{jc})^{2}\mathrm{cum}^{(4)}(L_{c})}{(\alpha^{ol}_{ic})(\alpha^{ol}_{jc})^{3}\mathrm{cum}^{(4)}(L_{c})}(\alpha_{jc}^{ol}L_{c} + v_{j})=
        v_{i} - \frac{\alpha^{ol}_{ic}}{\alpha^{ol}_{jc}}v_{j},
    \end{align*}
    which shows that $\tilde{e}_{(X_{i}, X_{j})} \indep L_{c}$. 
\end{proof}

\section{Proofs of Theorems in Section \ref{sec: partial cluster}}
\subsection{The proof of Theorem \ref{thm: clusters in proposed method}}

\begin{proof}
We prove this theorem by contrapositive. Let $L_i$ and $L_j$ be the respective latent parents of $X_i$ and $X_j$, assuming that $L_{j} \notin \mathrm{Anc}(L_{i})$.

We divide the proof into four cases. 
\begin{enumerate}
    \setlength{\itemsep}{0pt}
    \setlength{\parsep}{0pt}
    \setlength{\parskip}{0pt}
    \item $X_{i}$ and $X_{j}$ are pure: 
    \begin{enumerate}
    \setlength{\itemsep}{0pt}
            \setlength{\parsep}{0pt}
            \setlength{\parskip}{0pt}
        \item[1-1.] The number of observed children of $L_{j}$ is greater than one: \\
            There exists another observed child $X_k$ of $L_j$ such that $\mathrm{Pa}(X_{k}) = \{L_{j}\}$.
            Then $X_i$, $X_j$ and $X_k$ are expressed as 
            \begin{align*}
                X_{i} = L_{i} + e_{i}, \quad
                X_{j} = L_{j} + e_{j}, \quad
                X_{k} = \lambda_{kj}L_{j} + e_{k},
            \end{align*}
            and ${e}_{(X_{i}, X_{j} \mid X_{k})}$ is given by 
            \begin{align*}
                {e}_{(X_{i}, X_{j} \mid X_{k})} 
                &=(L_{i} + e_{i}) - \frac{\mathrm{Cov}(L_{i}, \lambda_{kj}L_{j})}{\mathrm{Cov}(L_{j}, \lambda_{kj}L_{j})}(L_{j} + e_{j}).
            \end{align*}
            Since $\lambda_{kj} \ne 0$ and ${\mathrm{Cov}(L_{i}, L_{j})} \ne 0$ generically, both $e_{(X_{i}, X_{j} \mid X_{k})}$ and $X_{k}$ contain terms of $\epsilon_{j}$, implying that $X_{k} \notindep {e}_{(X_{i}, X_{j} \mid X_{k})}$ from Theorem \ref{thm: DS}. 
        \item[1-2.] The number of observed children of $L_{j}$ is one: \\
            According to Assumption \ref{A2}, $L_j$ must have a latent child, denoted by $L_k$, and let $X_k$ be the child of $L_k$ with the highest causal order. 
            Then
            \begin{align*}
                X_{i} &= L_{i} + e_{i}, \quad
                X_{j} = L_{j} + e_{j}, \\
                X_{k} &= L_{k} + e_{k} = a_{kj}^{}L_{j} + \sum_{h: L_{h} \in \mathrm{Pa}(L_{k}) \setminus \{L_{j}\}} a_{kh}^{}L_{h} + \epsilon_{k} + e_{k}.
            \end{align*}
            and ${e}_{(X_{i}, X_{j} \mid X_{k})}$ is given by 
            \begin{align*}
                {e}_{(X_{i}, X_{j} \mid X_{k})} &= X_{i} - \frac{\mathrm{Cov}(X_{i}, X_{k})}{\mathrm{Cov}(X_{j}, X_{k})}(L_{j} + e_{j}).
            \end{align*}
            Since $a^{}_{kj} \ne 0$ and $\frac{\mathrm{Cov}(X_{i}, X_{j})}{\mathrm{Cov}(X_{j}, X_{k})} \ne 0$ generically, ${e}_{(X_{i}, X_{j} \mid X_{k})} \notindep X_{k}$ from Theorem \ref{thm: DS}.
    \end{enumerate}
    \item At least one of $X_i, X_j$ is impure: 
        Assume that $X_i$ is impure and that a directed edge exists between $X_i$ and $X_k$. 
        The proof proceeds analogously when $X_j$ is impure. 
        \begin{enumerate}
            \setlength{\itemsep}{0pt}
            \setlength{\parsep}{0pt}
            \setlength{\parskip}{0pt}
            \item[2-1.] $X_{i} \in \mathrm{Pa}(X_{k})$:\\
                $X_i, X_j$ and $X_k$ are expressed as
                \begin{align*}
                        X_{i} = L_{i} + e_{i}, \quad
                        X_{j} = L_{j} + e_{j}, \quad
                        X_{k} = (\lambda_{ki} + b_{ki})L_{i} + b_{ki}e_{i} + e_{k}, 
                    \end{align*}
                    respectively, and $e_{(X_{i}, X_{j} \mid X_{k})}$ is given by
                    \begin{align*}
                        e_{(X_{i}, X_{j} \mid X_{k})} 
                        &= (L_{i} + e_{i}) - \frac{(\lambda_{ki} + b_{ki})\mathrm{Var}(L_{i}) + b_{ki}\mathrm{Var}(e_{i})}{(\lambda_{ki}+b_{ki})\mathrm{Cov}(L_{i}, L_{j})}(L_{j}+e_{j}).
                    \end{align*}
                    Since both $e_{(X_{i}, X_{j} \mid X_{k})}$ and $X_{k}$ contain $e_{i}$, $e_{(X_i,X_j\mid X_k)} \notindep X_k$ from Theorem \ref{thm: DS}.
                    \item[2-2.] $X_{k} \in \mathrm{Pa}(X_{i})$:\\
                    $X_i, X_j$ and $X_k$ are expressed as
                    \begin{align*}
                        X_{i} = (\lambda_{ii} + b_{ik})L_{i} + b_{ik}e_{k} + e_{i}, \quad
                        X_{j} = L_{j} + e_{j}, \quad
                        X_{k} = L_{i} + e_{k}.
                    \end{align*}
                    respectively, and $e_{(X_{i}, X_{j} \mid X_{k})}$ is given by
                    \begin{align*}
                        e_{(X_{i}, X_{j} \mid X_{k})} 
                        &=(\lambda_{ii} + b_{ik})L_{i} + b_{ik}e_{k} + e_{i} \\ &- \frac{(b_{ik}+\lambda_{ii})\mathrm{Var}(L_{i}) + b_{ik}\mathrm{Var}(e_{k})}{\mathrm{Cov}(L_{i}, L_{j})}(L_{j}+e_{j}). 
                    \end{align*}
                    Since $b_{ik} \ne 0$, and both $e_{(X_{i}, X_{j} \mid X_{k})}$ and $X_{k}$ contain terms about $e_{k}$, $e_{(X_{i}, X_{j} \mid X_{k})}\notindep X_{k}$ is shown from Theorem \ref{thm: DS}.
                \end{enumerate}
        \end{enumerate}
\end{proof}
\section{Proofs of Theorems and Lemmas in Section \ref{sec: causal order}}
\subsection{The proof of Theorem \ref{thm: proposed cluster root}}
\begin{proof}
    \noindent\textbf{Sufficiency:}
    If $L_i$ is a latent source in $\mathcal{G}$, then no confounder exists between $L_i$ and any other latent variable $L_j$. 
    By Lemma~\ref{lem: X has multiple parent}, when $X_i$ and $X_j$ belong to distinct clusters, we have $\mathrm{Conf}(X_i,X_j)=\{L_i\}$, since $L_i \notindep L_j$. 
    If $X_i$ and $X_j$ belong to the same cluster confounded by $L_i$, then again $\mathrm{Conf}(X_i,X_j)=\{L_i\}$.
    
    \noindent\textbf{Necessity:} 
    Note that $\bm{X}_{\mathrm{oc}} \subset \bm{X} \setminus \{X_{i}\}$.
    We will prove the necessity by showing that if $L_i$ is not a latent source, there exists some $X_j \in \bm{X}_{\mathrm{oc}} \setminus \{X_i\}$ such that $X_i \notindep X_j$ and $\mathrm{Conf}(X_i, X_j) \ne \{L_i\}$ in the canonical model over $\{X_i, X_j\}$.
    
    Let $L_s$ be a latent source and let $X_s$ be the child of $L_s$ with the highest causal order. 
    By Lemma~\ref{lem: X has multiple parent} and the fact that $L_i \notindep L_s$, we have $\mathrm{Conf}(X_i,X_s)=\{L_s\} \ne \{L_{i}\}$. 
\end{proof}
\subsection{The proof of Corollary \ref{col: proposed cluster root}}
\begin{proof}
    According to Theorem~\ref{thm: proposed cluster root}, sufficiency is immediate.
    We therefore prove only necessity by showing that if $L_i$ is not a latent source, then neither case~1 nor case~2 holds. 
    
    If $L_i$ is not a latent source, then $\mathcal{X}_i \ne \emptyset$ or $\lvert \bm{X}_{\mathrm{oc}} \setminus \{X_{i}\}\rvert \ge 2$, and therefore case~1 is not satisfied.
    We now consider case~2 and show that either condition (a) or (b) is not satisfied.
     First, note that condition~(a) does not hold whenever there exists $X_j \in \bm{X}_{\mathrm{oc}}\setminus\{X_i\}$ with $\lvert \mathrm{Conf}(X_i,X_j)\rvert \ne 1$.
    Hence, assume that condition~(a) holds. 
    
    Let $L_s$ be the latent source in $\mathcal{G}$ and $X_s$ be its observed child with the highest causal order among $\hat{C}_s$. Let $X_j \in \bm{X}_{\mathrm{oc}}\setminus\{X_i\}$ have a latent parent $L_j$, and let $L_c$ be the unique confounder between $X_i$ and $X_j$, and $X_c$ be its observed child with the highest causal order. We have $X_c,X_s \in \bm{X}_{\mathrm{oc}}\setminus\{X_i\}$, and $\lvert \mathrm{Conf}(X_i,X_s) \rvert = \lvert \mathrm{Conf}(X_i,X_j)\rvert=1$. Next, we divide the following discussion into two cases depending on whether $L_i$ has a latent child:
    
    \begin{enumerate}
        \setlength{\itemsep}{0pt}
        \setlength{\parsep}{0pt}
        \setlength{\parskip}{0pt}
        \item If $L_i \in \mathrm{Pa}(L_j)$, then by Lemma~\ref{lem: X has one parent},
        \begin{align*}
            c^{(k)}_{(X_{i} \to X_{j})}(L^{(i, j)}) = \mathrm{cum}^{(k)}(L_{i}), \quad c^{(k)}_{(X_{i} \to X_{s})}(L^{(i, s)}) = \mathrm{cum}^{(k)}(\alpha_{is}^{ll}\cdot L_{s}).
        \end{align*}
        Thus, $c^{(k)}_{(X_{i} \to X_{j})}(L^{(i, j)}) \ne c^{(k)}_{(X_{i} \to X_{s})}(L^{(i, s)})$ generically, and condition (b) is not satisfied.
        \item If $L_i$ has no latent children, then $\mathcal{X}_i \ne \emptyset$, and
        \begin{align*}
            c^{(k)}_{(X_{i}\to X_{i'})}(L^{(i, i')}) = \mathrm{cum}^{(k)}(L_{i}).
        \end{align*}
        Also $c^{(k)}_{(X_i\to X_{i'})}(L^{(i, i')}) \ne c^{(k)}_{(X_i\to X_s)}(L^{(i, s)})$ generically, so condition~(b) is not satisfied. 
    \end{enumerate}
\end{proof}
\subsection{The proof of Theorem \ref{thm: independent e tilde}}
\begin{proof}
    We define sets
    \begin{align*}
        \mathcal{I}_{A} &= \{h \mid L_{h} \in \mathrm{Anc}(L_{i}) \cap \mathrm{Anc}(L_{j}) \setminus \{L_{1}\}\}, \\
        \mathcal{I}_{B} &= \{h \mid L_{h} \in \mathrm{Anc}(L_{i}) \setminus (\mathcal{I}_{A}\cup \{L_{1}\})\}, \\
        \mathcal{I}_{C} &= \{h \mid L_{h} \in \mathrm{Anc}(L_{j}) \setminus (\mathcal{I}_{A}\cup \{L_{1}\})\}.
    \end{align*}
    Then,
    \begin{align*}
        X_{1} &= \epsilon_{1} + e_{1},\\
        L_{i} &= \alpha_{i1}^{ll}\epsilon_{1}+\sum_{h\in \mathcal{I}_{A}} \alpha_{ih}^{ll}\epsilon_{h} + \sum_{h\in \mathcal{I}_{B}} \alpha_{ih}^{ll}\epsilon_{h} + \epsilon_{i}, \quad 
        X_{i} = L_{i} + e_{i}, \\
        L_{j} &= \alpha_{j1}^{ll}\epsilon_{1}+\sum_{h\in \mathcal{I}_{A}} \alpha_{jh}^{ll}\epsilon_{h} + \sum_{h\in \mathcal{I}_{C}} \alpha_{jh}^{ll}\epsilon_{h} + \epsilon_{j},  \quad X_{j} = L_{j} + e_{j},
    \end{align*}
    We can easily show that 
    \[
    \frac{\mathrm{cum}(X_i,X_i,X_1,X_1)}{\mathrm{cum}(X_i,X_1,X_1,X_1)} = \alpha^{ll}_{i1}, 
    \]
    and hence, we have
    \begin{align*}
        \tilde{e}_{(X_{i}, X_{1})} &= \sum_{h\in \mathcal{I}_{A}} \alpha_{ih}^{ll}\epsilon_{h} + \sum_{h\in \mathcal{I}_{B}} \alpha_{ih}^{ll}\epsilon_{h} + \epsilon_{i} + e_{i} - \alpha_{i1}^{ll}e_{1}.
    \end{align*}
    
    \noindent\textbf{Sufficiency:}
    Assume $L_{i} \indep L_{j}$ in the submodel induced by $\mathcal{G}^{-}(\{L_{1}\})$, which implies that $\mathcal{I}_{A} = \emptyset$.
    Therefore, $\tilde{e}_{(X_{i}, X_{1})}$ and $X_{j}$ can be written as
    \begin{align*}
         X_{j} &= \alpha_{j1}^{ll}\epsilon_{1} + \sum_{h\in \mathcal{I}_{C}} \alpha_{jh}^{ll}\epsilon_{h} + \epsilon_{j} + e_{j}, \\
        \tilde{e}_{(X_{i}, X_{1})} &= \sum_{h\in \mathcal{I}_{B}} \alpha_{ih}^{ll}\epsilon_{h} + \epsilon_{i} + e_{i} - \alpha_{i1}^{ll}e_{1}.
    \end{align*}
    Thus, we conclude that $X_{j} \indep \tilde{e}_{(X_{i}, X_{1})}$.
    Similarly, we can also show that $X_{i} \indep \tilde{e}_{(X_{j}, X_{1})}$.
    
    \noindent\textbf{Necessity:}
    Assume $L_{i} \notindep L_{j}$ in the submodel induced by $\mathcal{G}^{-}(\{L_{1}\})$, which implies that $\mathcal{I}_{A} \ne \emptyset$.
    Since neither $\alpha_{ih}^{ll}$ nor $\alpha_{jh}^{ll}$ for $h \in \mathcal{I}_{A}$ equals zero, $X_j \notindep \tilde{e}_{(X_i,X_1)}$ by the contrapositive of Theorem~\ref{thm: DS}.
    Similarly, we can also show that $X_i \notindep \tilde{e}_{(X_j,X_1)}$.
\end{proof}
\subsection{The proof of Theorem \ref{thm: proposed cluster root next}}
According to Lemma~\ref{lem: new statistic}, $\tilde{e}_{(X_{i}, X_{1})}$ can be regarded as a statistic obtained by removing the influence of $L_{1}$ from $X_{i}$. Based on this observation, we now provide the proof of Theorem~\ref{thm: proposed cluster root next}. 
\begin{proof}
    Let $L_{1}$ and $L_{i}$ be two latent variables, and define the set
    \begin{align*}
        \mathcal{I}_{A} = \{{h} \mid L_{h} \in \mathrm{Anc}(L_{i}) \setminus \{L_{1}\}\}.
    \end{align*}
    Then, $X_1$, $X_i$, and $\tilde{e}_{(X_{i}, X_{1})}$ are represented as
    \begin{align*}
        X_{1} &= \epsilon_{1} + e_{1}, \quad
        X_{i} = \alpha^{ll}_{i1}\epsilon_{1} + \sum_{{h} \in \mathcal{I}_{A}} \alpha^{ll}_{ih} \epsilon_{h} + \epsilon_{i} + e_{i}, \\
        \tilde{e}_{(X_{i}, X_{1})} &= \sum_{{h} \in \mathcal{I}_{A}} \alpha^{ll}_{ih} \epsilon_{h} + \epsilon_{i} + e_{i} - \alpha^{ll}_{i1}e_{1}, 
    \end{align*}
    respectively. 
    
    \noindent\textbf{Sufficiency:}
    If $L_{i}$ is the latent source in $\mathcal{G}^{-}(\{L_{1}\})$, $\mathcal{I}_{A} = \emptyset$. Hence, we have
    \begin{align*}
        X_{i} &= \alpha^{ll}_{i1}\epsilon_{1}+\epsilon_{i} + e_{i},\quad
        \tilde{e}_{(X_{i}, X_{1})} = \epsilon_{i} + e_{i} - \alpha^{ll}_{i1}e_{1}. 
    \end{align*}
    Assume that $X_{j} \in \bm{X}_{\mathrm{oc}} \setminus \{X_{1}, X_{i}\}$. 
    Define $\mathcal{I}_B$ by 
    \begin{align*}
        \mathcal{I}_{B} = \{{h} \mid L_{h}\in \mathrm{Anc}(L_{j}) \setminus \{L_{1}, L_{i}\} \}.
    \end{align*}
    Then, $\tilde{e}_{(X_{i}, X_{1})}$ and $X_j$ are written as 
    \begin{align*}
        \tilde{e}_{(X_{i}, X_{1})} 
        = {\epsilon}_{i} + (e_{i} - \alpha^{ll}_{i1}e_{1}),\quad
        X_{j} 
        = \alpha^{ll}_{ji}{\epsilon}_{i} + \alpha^{ll}_{j1}\epsilon_{1} + \sum_{{h} \in \mathcal{I}_{B}}\alpha^{ll}_{jh}\epsilon_{h} + \epsilon_{j} + e_{j},
    \end{align*}
    which shows that $\mathrm{Conf}(\tilde{e}_{(X_{i}, X_{1})}, X_{j}) = \{\epsilon_{i}\}$.
    
    Next, assume that $\mathcal{X}_{i} \ne \emptyset$ and $X_{i'} \in \mathcal{X}_{i}$.
    We divide the discussion into the following two cases.
    \begin{enumerate}
        \setlength{\itemsep}{0pt}
        \setlength{\parsep}{0pt}
        \setlength{\parskip}{0pt}
        \item $X_{i} \notin \mathrm{Anc}(X_{i'})$.
            Define the set 
            \begin{align*}
                \mathcal{I}_{C} = \{k \mid X_{k} \in \mathrm{Anc}(X_{i'}) \cap \hat{C}_{i}\}.
            \end{align*}
            We note that $i \notin \mathcal{I}_{A}$, and rewrite $X_{i'}$ as
            \begin{align*}
                \tilde{e}_{(X_{i}, X_{1})}
                &= {\epsilon}_{i} + (e_{i} - \alpha^{ll}_{i1}e_{1}), \\
                X_{i'} &= \alpha_{i'i}^{ol}L_{i} + \sum_{k \in \mathcal{I}_{C}}\alpha_{i'k}^{oo}e_{k} + e_{i'} = \alpha_{i'i}^{ol}(\alpha_{i1}^{ll}\epsilon_{1}+ \epsilon_{i}) + \sum_{k \in \mathcal{I}_{C}}\alpha_{i'k}^{oo}e_{k} + e_{i'},
            \end{align*}
            hence, $\mathrm{Conf}(\tilde{e}_{(X_{i}, X_{1})}, X_{i'}) = \{\epsilon_{i}\}$.
        \item $X_{i} \in \mathrm{Anc}(X_{i'})$.
            Define the set 
            \begin{align*}
                \mathcal{I}_{C} = \{k \mid X_{k} \in (\mathrm{Anc}(X_{i'}) \cap \hat{C}_{i})\setminus \{X_{i}\}\}.
            \end{align*}
            $\tilde{e}_{(X_{i}, X_{1})}$ and $X_{i'}$ are written as
            \begin{align*}
            \tilde{e}_{(X_{i}, X_{1})}
            &= {\epsilon}_{i} + (e_{i} - \alpha^{ll}_{i1}e_{1}), \\
            X_{i'} &= \alpha_{i'i}^{ol}(\alpha_{i1}^{ll}\epsilon_{1}+ \epsilon_{i}) + \alpha_{i'i}^{ol}e_{i} +  \sum_{k \in \mathcal{I}_{C}}\alpha_{i'k}^{oo}e_{k} + e_{i'}.
            \end{align*}
            Both $e_i$ and $\epsilon_i$ appear in $\tilde{e}_{(X_i,X_1)}$ and $X_{i'}$. 
            Since only $\tilde{e}_{(X_i, X_1)}$ contains $e_1$ while only $X_{i'}$ contains $e_{i'}$, there is no ancestral relation between them in their canonical model, according to Lemma~5 of Salehkaleybar et al.~\cite{salehkaleybar2020learning}. 
            Hence, $\mathrm{Conf}(\tilde{e}_{(X_i,X_1)},X_{i'})=\{\epsilon_i,e_i\}$, and we have $$\mathrm{Conf}(\tilde{e}_{(X_i,X_1)},X_{i'}) \cap \mathrm{Conf}(\tilde{e}_{(X_i,X_1)},X_{j}) = \{\epsilon_{i}\}.$$
    \end{enumerate}
    \noindent\textbf{Necessity:}
    By contrapositive, we aim to show that if $L_{i}$ is not a latent source, then either condition 1 or 2 does not hold. Assume that $L_{s}$ is the latent source in $\mathcal{G}^{-}(\{L_{1}\})$, and that $X_{s}$ is its observed child with the highest causal order. Then, we have 
    \begin{align*}
        \tilde{e}_{(X_{i}, X_{1})} &= \alpha_{is}^{ll}\epsilon_{s} + \sum_{{h} \in \mathcal{I}_{A}\setminus\{s\}} \alpha^{ll}_{ih} \epsilon_{h} + \epsilon_{i} + e_{i} - \alpha^{ll}_{i1}e_{1} \\
        X_{s} &= a_{s1}\epsilon_{1} + \epsilon_{s} + e_{s},
    \end{align*}
    implying that $\mathrm{Conf}(\tilde{e}_{(X_{i}, X_{1})}, X_{s}) = \{\epsilon_{s}\} \ne \{\epsilon_{i}\}$.
    Thus, condition 1 is not satisfied.
    
\end{proof}
\subsection{The Proof of Lemma~\ref{lem: new statistic equation}}
\begin{proof}
    Since it is trivial that
    \[
        \tilde{e}_{(X_{i}, \tilde{\bm{e}}_{s})} = \epsilon_{i} + \sum_{k=s}^{i-1} \alpha^{ll}_{ih}\epsilon_{h} + e_{i},
    \]
    when $i > s$ and $s = 1$, we only discuss the remaining two cases.

    We first prove the case where $s = i$ by induction on $i$.
    When $i=1$, 
    \begin{align*}
        \tilde{e}_{(X_{1}, \tilde{\bm{e}}_{1})} = X_1 = \epsilon_{1} + e_{1}, 
    \end{align*}
    where $U_{[1]} = \epsilon_1$. 

    Assume that the inductive assumption holds up to $i$. Then, 
    \[
    \tilde{e}_{(X_{h}, \tilde{\bm{e}}_{h})} = \epsilon_{h} + U_{[h]}, \quad 1 \le h \le i.
    \]
    Since $X_{i+1}$ is expressed as 
    \[
    X_{i+1} = \epsilon_{i+1} + e_{i+1} + \sum_{h=1}^{i}\alpha^{ll}_{i+1,h}\epsilon_{h}, 
    \]
    we have 
    \[
    \rho_{(X_{i+1}, \tilde{e}_{(X_{h}, \tilde{\bm{e}}_{h})})}=\alpha^{ll}_{i+1,h}
    \] 
    for $h=1,\ldots,i$, according to Definition \ref{def: tilde e general}. Hence, we have
    \begin{align*}
        \tilde{e}_{(X_{i+1}, \tilde{\bm{e}}_{i+1})} &= X_{i+1} - \sum_{h=1}^{i}\rho_{(X_{i+1}, \tilde{e}_{(X_{h}, \tilde{\bm{e}}_{h})})}\tilde{e}_{(X_{h}, \tilde{\bm{e}}_{h})}\\
        &= \epsilon_{i+1} + e_{i+1} + \sum_{h=1}^{i} \alpha^{ll}_{i+1, h} \epsilon_{h} - \sum_{h=1}^{i}\alpha^{ll}_{i+1, h} (\epsilon_{h} + U_{[h]}) \\
        &= \epsilon_{i+1} + e_{i+1} - \sum_{h=1}^{i}\alpha^{ll}_{i+1, h} U_{[h]} = \epsilon_{i+1} + U_{[i+1]},
    \end{align*}
    where $U_{[i+1]} = e_{i+1} - \sum_{h=1}^{i}\alpha^{ll}_{i+1, h} U_{[h]}$. 
    Thus, the claim holds for all $i$ by induction. 

    Next, we discuss the case where $i > s$ and $s > 1$.
    According to Definition \ref{def: tilde e general},
    \[
        \tilde{e}_{(X_{i}, \tilde{e}_{s})} = X_{i} - \sum_{h=1}^{s-1} \rho_{(X_{i}, \tilde{e}_{(X_{h}, \tilde{\bm{e}}_{h})})}\tilde{e}_{(X_{h}, \tilde{\bm{e}}_{h})}, \quad \rho_{(X_{i}, \tilde{e}_{(X_{h}, \tilde{\bm{e}}_{h})})} = \alpha^{ll}_{ih}.
    \]
    Using the conclusion of the case where $i = s$, we obtain
    \begin{align*}
        \tilde{e}_{(X_{i}, \tilde{e}_{s})} &= \epsilon_{i} + \sum_{h=1}^{i-1}\alpha^{ll}_{ih}\epsilon_{h}+ e_i - \sum_{h=1}^{s-1}\alpha^{ll}_{ih}\left(\epsilon_{h} + U_{[h]}\right) \\
        &= \epsilon_{i} + \sum_{h=s}^{i-1}\alpha^{ll}_{ih}\epsilon_{h} + U_{[s-1]} + e_{i}.
    \end{align*}
\end{proof}

\subsection{The proof of Theorem \ref{thm: independent e tilde generalized}}
\begin{proof}
    The proof of this theorem follows similarly to that of Theorem \ref{thm: independent e tilde}.
\end{proof}
\subsection{The proof of Theorem \ref{thm: proposed cluster root generalized}}
\begin{proof}
    The proof of this theorem follows similarly to that of Theorem \ref{thm: proposed cluster root next}.
\end{proof}
\subsection{The proof of Corollary \ref{col: proposed cluster root generalized}}
According to Theorem~\ref{thm: proposed cluster root generalized}, sufficiency is immediate. We therefore prove only necessity by showing that if $L_i$ is not a latent source in 
$\mathcal{G}^{-}(\{L_1,\ldots,L_{s-1}\})$, then neither case~1 nor case~2 holds. 

If $L_i$ is not a latent source, $\mathcal{X}_i \ne \emptyset$ or $\lvert \bm{X}_{\mathrm{oc}} \setminus \{X_{1}, \dots, X_{s-1}, X_{i}\}\rvert \ge 2$, and therefore case~1 is not satisfied. We will show that one of the conditions (a), (b), and (c) is not satisfied. First, note that condition~(a) does not hold whenever there exists $X_j \in \bm{X}_{\mathrm{oc}}\setminus\{X_{1}, \dots, X_{s-1}, X_{i}\}$ with $\lvert \mathrm{Conf}(\tilde{e}_{(X_i, \tilde{\bm{e}}_{s})},X_j)\rvert \ne 1$. 
Hence, assume that condition~(a) holds.

Assume that $L_s$ is the latent source of $\mathcal{G}^{-}(\{L_1,\ldots,L_{s-1}\})$, 
and that $X_s$ is its observed child with the highest causal order. Let $L_j$ be the latent parent of $X_j \in \bm{X}_{\mathrm{oc}}\setminus\{X_{1}, \dots, X_{s-1}, X_{i}\}$, respectively. 
Since the condition (a) holds, $\lvert \mathrm{Conf}(\tilde{e}_{(X_i, \tilde{\bm{e}}_{s})},X_s)\rvert=\lvert \mathrm{Conf}(\tilde{e}_{(X_i, \tilde{\bm{e}}_{s})},X_j)\rvert=1$, where $\tilde{\bm{e}}_{s} = (\tilde{e}_{(X_1, \tilde{\bm{e}}_{1})}, \dots, \tilde{e}_{(X_{s-1}, \tilde{\bm{e}}_{s-1})})$.  

Then, $X_{s}$ and $\tilde{e}_{(X_{i}, \tilde{\bm{e}}_{s})}$ are written as
\begin{align*}
    X_{s} &= \sum_{h = 1}^{s-1}\alpha^{ll}_{sh}\epsilon_{h} + \epsilon_{s} + e_{s},
    \\
    \tilde{e}_{(X_{i}, \tilde{\bm{e}}_{s})} &= \epsilon_{i} + \sum_{h = s}^{i-1}\alpha_{ih}^{ll}\epsilon_{h}+ U_{[s-1]} + e_{i},
\end{align*}
according to Lemma \ref{lem: new statistic equation}.
Hence, we have $\mathrm{Conf}(\tilde{e}_{(X_i, \tilde{\bm{e}}_{s})},X_s) = \{\epsilon_{s}\}$. 

Assume that $L_i$ has a latent child $L_j$ and that none of the descendants of $L_{i}$ are parents of $L_{j}$. $X_j$ is expressed by 
\[
X_j = \sum_{h=1}^{j-1} \alpha^{ll}_{jh} \epsilon_h +\epsilon_{j}+ e_j. 
\]
Both $\tilde{e}_{(X_{i}, \tilde{\bm{e}}_{s})}$ and $X_j$ involve linear combinations of $\epsilon_{s}, \ldots, \epsilon_{i}$. 
Since $\{\epsilon_{s}, \dots, \epsilon_{i}\}$ are mutually independent and $\lvert \mathrm{Conf}(\tilde{e}_{(X_i, \tilde{\bm{e}}_{s})},X_j)\rvert =1$, $\alpha^{ll}_{jh}=\alpha^{ll}_{ji}\alpha^{ll}_{ih}$ according to Hoyer et al. \cite{hoyer2008estimation}, and then $X_{j}$ can be rewritten as
\[
X_j = 
\left\{
\begin{array}{ll}
    \displaystyle{
    \sum_{i=1}^{s-1} \alpha^{ll}_{jh} \epsilon_h + 
    \alpha^{ll}_{ji} \left(
    \epsilon_{i} + \sum_{h = s}^{i-1}\alpha_{ih}^{ll}\epsilon_{h}
    \right)  + \epsilon_{j} + 
    e_j}, & j=i+1,  \\
    \displaystyle{
    \sum_{i=1}^{s-1} \alpha^{ll}_{jh} \epsilon_h +
    \alpha^{ll}_{ji} 
    \left(
    \epsilon_{i} + \sum_{h = s}^{i-1}\alpha_{ih}^{ll}\epsilon_{h}
    \right) + \sum_{h=i+1}^{j-1} \alpha^{ll}_{jh} \epsilon_h +\epsilon_{j} + e_j,} & j>i+1.
\end{array}
\right.
\]
Therefore,
\begin{align*}
    c^{(k)}_{(\tilde{e}_{(X_i, \tilde{\bm{e}}_{s})} \to X_{j})}(L^{(i, j)}) &= \mathrm{cum}^{(k)} \left(\epsilon_{i} + \sum_{h=s}^{i-1}\alpha^{ll}_{ih}\epsilon_{h}\right), \\ 
    c^{(k)}_{(\tilde{e}_{(X_i, \tilde{\bm{e}}_{s})} \to X_{s})}(L^{(i, s)}) &= \mathrm{cum}^{(k)}(\alpha^{ll}_{is}\epsilon_{s}),
\end{align*}
according to Lemma \ref{lem: X has one parent}.
Therefore, we conclude that $c^{(k)}_{(\tilde{e}_{(X_i, \tilde{\bm{e}}_{s})} \to X_{j})}(L^{(i, j)}) \ne c^{(k)}_{(\tilde{e}_{(X_i, \tilde{\bm{e}}_{s})} \to X_{s})}(L^{(i, s)})$ generically, and condition (b) is not satisfied.
Next, we assume that $L_{i}$ does not have latent children, so that $\mathcal{X}_{i} \ne \emptyset$. 
Assume that $X_{i'} \in \mathcal{X}_{i}$ and it can be
expressed as
\begin{align*}
    X_{i'} = \sum_{h=1}^{i-1} \alpha^{ll}_{ih}\epsilon_{h} + \epsilon_{i} + b_{i'i}e_{i} + e_{i'}.
\end{align*}
Then,
\begin{align*}
    \mathrm{Conf}(\tilde{e}_{(X_i, \tilde{\bm{e}}_{s})}, X_{i'}) = 
    \begin{cases}
       \left\{\epsilon_{i} + \sum_{h=s}^{i-1}\alpha^{ll}_{ih}\epsilon_{h}\right\}, & b_{i'i} = 0, \\[6pt]
       \left\{\epsilon_{i} + \sum_{h=s}^{i-1}\alpha^{ll}_{ih}\epsilon_{h}, \; e_{i}\right\}, & b_{i'i} \ne 0.
    \end{cases}
\end{align*}
In either case, there is one latent $L^{(i, i')}$ satisfies that
\begin{align*}
    c^{(k)}_{(\tilde{e}_{(X_i, \tilde{\bm{e}}_{s})} \to X_{i'})}(L^{(i,i')}) 
    = \mathrm{cum}^{(k)}\!\left(\epsilon_{i} + \sum_{h=s}^{i-1}\alpha^{ll}_{ih}\epsilon_{h}\right),
\end{align*}
which generically implies
\[
    c^{(k)}_{(\tilde{e}_{(X_i, \tilde{\bm{e}}_{s})} \to X_{i'})}(L^{(i,i')}) 
    \ne c^{(k)}_{(\tilde{e}_{(X_i, \tilde{\bm{e}}_{s})} \to X_{s})}(L^{(i,s)}).
\]
Thus, the condition (c) is not satisfied.
\section{Proofs of Theorems in Section \ref{sec: redundant edges}}
\subsection{Theorem \ref{thm: proposed cluster root independence 2}}
\begin{proof}
    From Lemma \ref{lem: new statistic equation}, 
    \begin{align}
        \label{eq: tilde e thm 3.19}
        \tilde{e}_{(X_{i-k}, \tilde{\bm{e}}_{i-k})} 
        &= \epsilon_{i-k} + U_{[i-k]}. 
    \end{align}
    By definition, $\tilde{r}_{i,k-1}$ is written as
    \begin{align}
        \label{eq: tilde_r}
        \tilde{r}_{i,k-1} &= X_i - \sum_{h=i-(k-1)}^{i-1} a_{ih}X_h \notag\\
        &= \sum_{h=1}^{i-1} a_{ih} L_h + \epsilon_i + e_i - \sum_{h=i-(k-1)}^{i-1} a_{ih}(L_h + e_h) \notag\\
        &= \sum_{h=1}^{i-(k+1)} a_{ih} L_h + a_{i,i-k}\epsilon_{i-k} + \epsilon_i + e_i - \sum_{h=i-(k-1)}^{i-1} a_{ih}e_h \notag\\
        &= V_{[i-(k+1)]} + a_{i,i-k}\epsilon_{i-k} + \epsilon_i + U_{[i] \setminus [i-k]}, 
    \end{align}
    where $V_{[i-(k+1)]}$ is a linear combination of $\{\epsilon_1,\ldots,\epsilon_{i-(k+1)}\}$ and 
    $U_{[i] \setminus [i-k]}$ is a linear combination of $\{e_{i-(k-1)},\ldots,e_i\}$. 
    From (\ref{eq: tilde e thm 3.19}) and (\ref{eq: tilde_r}), we can show that
    \[
    \tilde{r}_{i,k-1} \indep \tilde{e}_{(X_{i-k}, \tilde{\bm{e}}_{i-k})}
    \;\Leftrightarrow\; a_{i,i-k}=0,
    \]
    and otherwise
    \begin{align*}
    \rho_{(X_{i},\tilde{e}_{(X_{i-1}, \tilde{\bm{e}}_{i-1})})}
    = \frac{a_{i, i-1}^{2}\,\mathrm{cum}^{(4)}(\epsilon_{i-1})}{a_{i, i-1}\,\mathrm{cum}^{(4)}(\epsilon_{i-1})}
    = a_{i, i-1}
    \end{align*}
    generically holds. 
\end{proof}

\section{Reducing an LvLiNGAM}
\label{sec: reduce models}
In this paper, we have discussed the identifiability of LvLiNGAM under the assumption that each observed variable has exactly one latent parent. However, even when some observed variables do not have latent parents, by iteratively marginalizing out sink nodes and conditioning on source nodes to remove such variables one by one, the model can be progressively reduced to one in which each observed variable has a single latent parent. This can be achieved by first estimating the causal structure involving the observed variables without latent parents. ParceLiNGAM \cite{tashiro2014parcelingam} or RCD \cite{Maeda2020, maeda2022rcd} can identify the ancestral relationship between two observed variables if at least one of them does not have a latent parent, and remove the influence of the observed variable without a latent parent. 

Models 1–3 in Figure~\ref{fig: reduced models} contain observed variables that do not have a latent parent.
According to Definition \ref{def: cluster}, $X_{1}$ and $X_{2}$ in Models 1 and 3 belong to distinct clusters whose latent parents are $L_{1}$ and $L_{2}$, respectively, whereas in Model 2, $X_{1}$ and $X_{2}$ share the same latent parent $L_{1}$. Model 3 contains a directed path between clusters, whereas Models 1 and 2 do not. We consider the model reduction procedure for Models 1–3 individually. 

\begin{example}[Model 1]
By using ParceLiNGAM or RCD, we can identify $X_{4} \to X_{1}$, $X_{4} \to X_{2}$, $X_{1} \to X_{5}$, and $X_{3} \to X_{5}$. 
Since $X_{5}$ is a sink node, the induced subgraph obtained by removing $X_{5}$ represents the marginal model over the remaining variables. Since $X_{4}$ is a source node, if we replace $X_{1}$ and $X_{2}$ with the residuals $r_{1}^{(4)}$ and $r_{2}^{(4)}$ obtained by regressing them on $X_{4}$, then the induced subgraph obtained by removing $X_{4}$ represents the conditional distribution given $X_{4}$. As a result, Model 1 is reduced to the model shown in Figure~\ref{fig: reduced models}~(d). This model satisfies Assumptions A1–A3.
\end{example}

\begin{example}[Model 2]
$X_{1}$ and $X_{2}$ are confounded by $L_{1}$, and they are mediated through $X_{3}$. By using ParceLiNGAM or RCD, the ancestral relationship among $X_1, X_2, X_3$ can be identified. Let $r^{(1)}_{3}$ be the residual obtained by regressing $X_3$ on $X_1$. Let $\tilde{r}^{(3)}_{2}$ be the residual obtained by regressing $X_2$ on $r^{(1)}_3$. According to \cite{tashiro2014parcelingam} and \cite{Maeda2020, maeda2022rcd}, the model for $L_1$, $X_1$, and $r^{(3)}_{2}$ corresponds to the one shown in Figure~\ref{fig: reduced models}~(e). This model satisfies Assumptions A1–A3.
\end{example}

\begin{example}[Model 3]
In Model 3, $X_{1}\in \mathrm{Anc}(X_{3})$, and they are mediated by $X_{5}$. By using ParceLiNGAM or RCD, the ancestral relationship among $X_1, X_3, X_5$ can be identified. Let $r^{(1)}_{5}$ be the residual obtained by regressing $X_5$ on $X_1$. Let $\tilde{r}^{(5)}_{3}$ be the residual obtained by regressing $X_3$ on $r^{(1)}_5$. According to \cite{tashiro2014parcelingam} and \cite{Maeda2020, maeda2022rcd}, by reasoning in the same way as for Models 1 and 2, Model 3 is reduced to the model shown in Figure~\ref{fig: reduced models}~(f). This model doesn't satisfy Assumptions A1–A3. 
\end{example}

\begin{figure}[t]
    \centering
    \begin{subfigure}[t]{0.3\textwidth}
        \centering
        \begin{tikzpicture}[scale=0.7]
            \node[draw, rectangle] (L_1) at (0, 1.5) {$L_1$};
            \node[draw, circle] (x_1) at (-0.5, 0) {$X_1$};
            \node[draw, rectangle] (L_2) at (2, 1.5) {$L_2$};
            \node[draw, circle] (x_2) at (2.5, 3) {$X_2$};
            \node[draw, circle] (x_3) at (2.5, 0) {$X_3$};
            \node[draw, circle] (x_4) at (-0.5, 3) {$X_4$};
            \node[draw, circle] (x_5) at (1, -1.5) {$X_5$};
            
            \draw[thick, ->] (L_1) -- (x_1);
            \draw[thick, ->] (L_1) -- (L_2);
            \draw[thick, ->] (L_2) -- (x_2);
            \draw[thick, ->] (L_2) -- (x_3);
            \draw[thick, ->] (x_4) to [in=120, out=-120](x_1);
            \draw[thick, ->] (x_4) -- (x_2);
            \draw[thick, ->] (x_1) -- (x_5);
            \draw[thick, ->] (x_3) -- (x_5);
        \end{tikzpicture}
        \caption{Model 1}
    \end{subfigure}
    \hfill
    \begin{subfigure}[t]{0.3\textwidth}
        \centering
        \begin{tikzpicture}[scale=0.7]
            \node[draw, rectangle] (L_1) at (0, 3) {$L_1$};
            \node[draw, circle] (x_1) at (-2, 0) {$X_1$};
            \node[draw, circle] (x_2) at (2, 0) {$X_2$};
            \node[draw, circle] (x_3) at (0, 0) {$X_3$};
            
            \draw[thick, ->] (L_1) -- (x_1);
            \draw[thick, ->] (L_1) -- (x_2);
            \draw[thick, ->] (x_1) -- (x_3);
            \draw[thick, ->] (x_3) -- (x_2);
        \end{tikzpicture}
        \caption{Model 2}
    \end{subfigure}
    \hfill
    \begin{subfigure}[t]{0.3\textwidth}
        \centering
        \begin{tikzpicture}[scale=0.7]
            \node[draw, rectangle] (L_1) at (0, 1.5) {$L_1$};
            \node[draw, circle] (x_1) at (-0.5, 0) {$X_1$};
            \node[draw, rectangle] (L_2) at (2, 1.5) {$L_2$};
            \node[draw, circle] (x_2) at (2.5, 3) {$X_2$};
            \node[draw, circle] (x_3) at (2.5, 0) {$X_3$};
            \node[draw, circle] (x_4) at (-0.5, 3) {$X_4$};
            \node[draw, circle] (x_5) at (1, -1.5) {$X_5$};
            
            \draw[thick, ->] (L_1) -- (x_1);
            \draw[thick, ->] (L_1) -- (L_2);
            \draw[thick, ->] (L_2) -- (x_2);
            \draw[thick, ->] (L_2) -- (x_3);
            \draw[thick, ->] (x_4) to [in=120, out=-120](x_1);
            \draw[thick, ->] (x_4) -- (x_2);
            \draw[thick, ->] (x_1) -- (x_5);
            \draw[thick, ->] (x_5) -- (x_3);
        \end{tikzpicture}
        \caption{Model 3}
    \end{subfigure}

    \begin{subfigure}[t]{0.3\textwidth}
        \centering
        \begin{tikzpicture}[scale=0.7]
            \node[draw, rectangle] (L_1) at (0, 2) {$L_1$};
            \node[draw, circle] (x_1) at (0, 0) {$r^{(4)}_1$};
            \node[draw, rectangle] (L_2) at (2, 2) {$L_2$};
            \node[draw, circle] (x_3) at (3, 0) {$X_3$};
            \node[draw, circle] (x_2) at (4, 2) {$r^{(4)}_2$};
            
            \draw[thick, ->] (L_1) -- (x_1);
            \draw[thick, ->] (L_1) -- (L_2);
            \draw[thick, ->] (L_2) -- (x_2);
            \draw[thick, ->] (L_2) -- (x_3);
        \end{tikzpicture}
        \caption{Reduced model of Model 1}
    \end{subfigure}
    \hfill
    \begin{subfigure}[t]{0.3\textwidth}
        \centering
        \begin{tikzpicture}[scale=0.7]
            \node[draw, rectangle] (L_1) at (0, 2) {$L_1$};
            \node[draw, circle] (x_1) at (-1.25, 0) {$X_1$};
            \node[draw, circle] (x_2) at (1.25, 0) {$\tilde{r}^{(3)}_2$};
            
            \draw[thick, ->] (L_1) -- (x_1);
            \draw[thick, ->] (L_1) -- (x_2);
            \draw[thick, ->] (x_1) -- (x_2);
        \end{tikzpicture}
        \caption{Reduced model of Model 2}
    \end{subfigure}
    \hfill
    \begin{subfigure}[t]{0.3\textwidth}
        \centering
        \begin{tikzpicture}[scale=0.7]
            \node[draw, rectangle] (L_1) at (0, 2) {$L_1$};
            \node[draw, circle] (x_1) at (0, 0) {$r^{(4)}_1$};
            \node[draw, rectangle] (L_2) at (2, 2) {$L_2$};
            \node[draw, circle] (x_2) at (4, 2) {$r^{(4)}_2$};
            \node[draw, circle] (x_3) at (3, 0) {$\tilde{r}^{(5)}_3$};
            
            \draw[thick, ->] (L_1) -- (x_1);
            \draw[thick, ->] (L_1) -- (L_2);
            \draw[thick, ->] (L_2) -- (x_2);
            \draw[thick, ->] (L_2) -- (x_3);
            \draw[thick, ->] (x_1) -- (x_3);
        \end{tikzpicture}
        \caption{Reduced model of Model 3}
    \end{subfigure}

    \caption{Three models that can be reduced.}
    \label{fig: reduced models}
\end{figure}
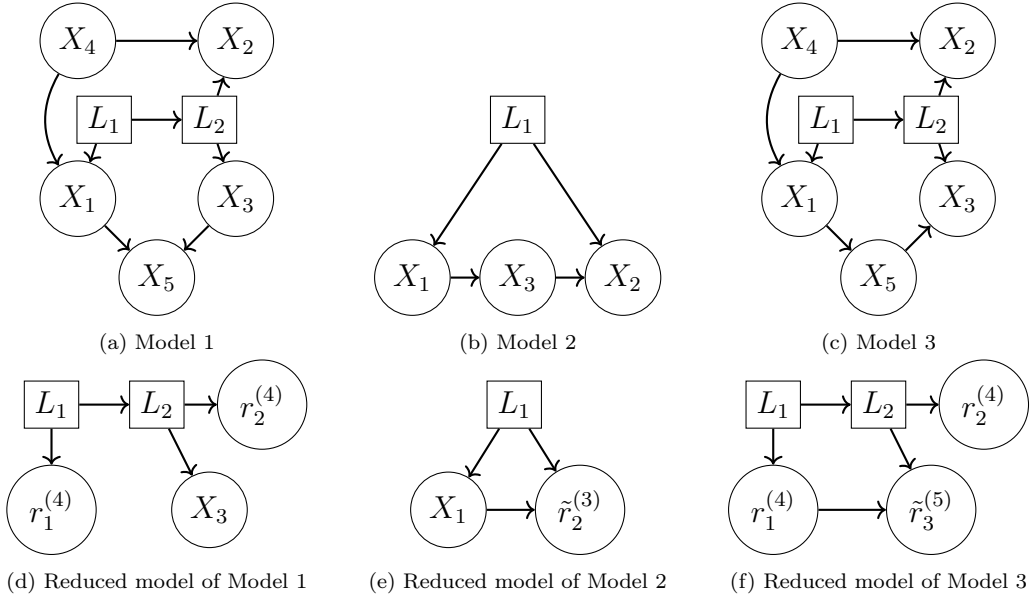

Using \cite{tashiro2014parcelingam} and \cite{Maeda2020, maeda2022rcd}, ancestral relations between pairs of observed variables that include at least one variable without a latent parent can be identified. The graph obtained by the model reduction procedure is constructed by iteratively applying the following steps:
\begin{enumerate}
\setlength{\itemsep}{0pt}
\setlength{\parsep}{0pt}
\setlength{\parskip}{0pt}
    \item[(i)] iteratively remove observed variables without latent parents that appear as source or sink nodes, updating the induced subgraph at each step so that any new source or sink nodes are subsequently removed;
    \item[(ii)] when an observed variable without a latent parent serves as a mediator, remove the variable and connect its parent and child with a directed edge.
\end{enumerate}
If no directed path exists between any two observed variables with distinct latent parents, the model obtained through the model reduction procedure satisfies Assumptions A1–A3. Conversely, if there exist two observed variables with distinct latent parents that are connected by a directed path, the model obtained through the model reduction procedure does not satisfy Assumption A3. In summary, Assumption A1 can be generalized to 
\begin{enumerate}
\setlength{\itemsep}{0pt}
\setlength{\parsep}{0pt}
\setlength{\parskip}{0pt}
\item[\textbf{A1}$\bm{'}$\textbf{.}] Each observed variable has at most one latent parent.
\end{enumerate}

\begin{proposition}
    Given observed data generated from an LvLiNGAM $\mathcal{M}_{\mathcal{G}}$ that satisfies the assumptions A1$^\prime$ and A2–A5, the latent causal structure among the latent variables, the directed edges from the latent variables to the observed variables, and the ancestral relationships among the observed variables can be identified by using the proposed method in combination with ParceLiNGAM and RCD. 
\end{proposition}

\end{document}